\documentclass[english]{article}

\usepackage{geometry}
\usepackage[T1]{fontenc}
\usepackage{amsmath,amssymb,amsthm,graphicx,bm,dsfont}
\usepackage{epsfig, fontawesome}
\usepackage[authoryear]{natbib} 
\usepackage{psfrag}
\usepackage{enumerate} \usepackage{setspace}
\usepackage{float} 
\usepackage{color}
\usepackage{array}
\usepackage{microtype, wasysym, comment,mathtools}
\usepackage{subfigure}
\usepackage{booktabs} 
\usepackage{enumitem,multirow}
\usepackage[colorlinks=true,allcolors=blue]{hyperref}%
\usepackage{graphicx} 
\usepackage{epstopdf}
\usepackage[ruled,vlined]{algorithm2e}
\usepackage{algorithmic}

\definecolor{yuting}{RGB}{255,69,0}
\definecolor{gen}{RGB}{199,21,133} 
\definecolor{yxc}{RGB}{21,199,133} 
\definecolor{yc}{RGB}{21,0,255}

\DeclareMathOperator{\ind}{\mathds{1}}

\newcommand{\mymid}{\,|\,}

\newcommand{\overalpha}{\overline{\alpha}}

\newcommand{\Pdata}{p_{\mathsf{data}}}
\newcommand{\diff}{\mathrm{d}}
\newcommand{\score}{\mathsf{score}}
\newcommand{\Jacobi}{\mathsf{Jacobi}}

\newtheorem{assumption}{\textbf{Assumption}}
\newtheorem{definition}{\textbf{Definition}}
\newtheorem{claim}{\textbf{Claim}}
\newtheorem{remark}{\textbf{Remark}}

\theoremstyle{plain}

\newtheorem{theo}{Theorem}[section]

\newtheorem{lem}{Lemma}[section]
\newtheorem{prop}{Proposition}[section]
\newtheorem{cor}{Corollary}[section]

\theoremstyle{definition} 

\newtheorem{nota}{Notation}[section]
\newtheorem{de}{Definition}[section]
\newtheorem{exa}{Example}[section]
\newtheorem{as}{Assumption}[section]
\newtheorem{alg}{Algorithm}[section]

\newcommand{\btheo}{\begin{theo}}
\newcommand{\bde}{\begin{de}}
\newcommand{\ble}{\begin{lem}}
\newcommand{\bpr}{\begin{prop}}
\newcommand{\bno}{\begin{nota}}
\newcommand{\bex}{\begin{exa}}
\newcommand{\bcor}{\begin{cor}}
\newcommand{\spro}{\begin{proof}}
\newcommand{\bas}{\begin{as}}
\newcommand{\balg}{\begin{alg}}

\newcommand{\etheo}{\end{theo}}
\newcommand{\ede}{\end{de}}
\newcommand{\ele}{\end{lem}}
\newcommand{\epr}{\end{prop}}
\newcommand{\eno}{\end{nota}}
\newcommand{\eex}{\end{exa}}
\newcommand{\ecor}{\end{cor}}
\newcommand{\fpro}{\end{proof}}
\newcommand{\eas}{\end{as}}
\newcommand{\ealg}{\end{alg}}

\theoremstyle{plain}

\newtheorem{theos}{Theorem}
\newtheorem{props}{Proposition}
\newtheorem{lems}{Lemma}
\newtheorem{cors}{Corollary}

\theoremstyle{definition}
\newtheorem{exas}{Example}
\newtheorem{algs}{Algorithm}
\newtheorem{asss}{Assumption}
\newtheorem{defns}{Definition}

\newcommand{\btheos}{\begin{theos}}
\newcommand{\etheos}{\end{theos}}
\newcommand{\bprops}{\begin{props}}
\newcommand{\eprops}{\end{props}}
\newcommand{\bdes}{\begin{defns}}
\newcommand{\edes}{\end{defns}}
\newcommand{\blems}{\begin{lems}}
\newcommand{\elems}{\end{lems}}
\newcommand{\bcors}{\begin{cors}}
\newcommand{\ecors}{\end{cors}}
\newcommand{\bexs}{\begin{exas}}
\newcommand{\eexs}{\end{exas}}
\newcommand{\balgs}{\begin{algs}}
\newcommand{\ealgs}{\end{algs}}
\newcommand{\bass}{\begin{asss}}
\newcommand{\eass}{\end{asss}}

\newcommand{\ltwo}[1]{\|#1\|_2}

\newcommand{\real}{\ensuremath{\mathbb{R}}}

\newcommand{\defn}{\coloneqq}

\long\def\comment#1{}

\newcommand{\HACKPROOF}{\begin{proof}}
\newcommand{\HACKENDPROOF}{\end{proof}}

\newlength{\widebarargwidth}
\newlength{\widebarargheight}
\newlength{\widebarargdepth}

\makeatletter
\long\def\@makecaption#1#2{
        \vskip 0.8ex
        \setbox\@tempboxa\hbox{\small {\bf #1:} #2}
        \parindent 1.5em  
        \dimen0=\hsize
        \advance\dimen0 by -3em
        \ifdim \wd\@tempboxa >\dimen0
                \hbox to \hsize{
                        \parindent 0em
                        \hfil 
                        \parbox{\dimen0}{\def\baselinestretch{0.96}\small
                                {\bf #1.} #2
                                } 
                        \hfil}
        \else \hbox to \hsize{\hfil \box\@tempboxa \hfil}
        \fi
        }
\makeatother

\allowdisplaybreaks

\begin{document}

\title{A Sharp Convergence Theory for \\ The Probability Flow ODEs of Diffusion Models\footnotetext{This manuscript presents improved theory for probability flow ODEs compared to its earlier version \citet{li2023towards}.}}

\author{Gen Li\thanks{Department of Statistics, The Chinese University
of Hong Kong, Hong Kong.} \and  Yuting Wei\thanks{Department of Statistics and Data Science, Wharton School, University
of Pennsylvania, Philadelphia, PA 19104, USA.} 
\and Yuejie Chi\thanks{Department of Electrical and Computer Engineering, Carnegie Mellon University, Pittsburgh, PA 15213, USA.}
\and Yuxin Chen\footnotemark[2] \thanks{Department of Electrical and Systems Engineering, University
of Pennsylvania, Philadelphia, PA 19104, USA.} 
 }

\date{\today}

\maketitle

\medskip

\begin{abstract}

Diffusion models, which convert noise into new data instances by learning to reverse a diffusion process, have become a cornerstone in contemporary generative modeling.   
	In this work, we develop non-asymptotic convergence theory for a popular diffusion-based sampler (i.e., the probability flow ODE sampler) in discrete time,  assuming access to $\ell_2$-accurate estimates of the (Stein) score functions. 
For distributions in $\mathbb{R}^d$, 
we prove that $d/\varepsilon$ iterations --- modulo some logarithmic and lower-order terms --- are sufficient 
to approximate the target distribution to within $\varepsilon$ total-variation distance. 
This is the first result establishing nearly linear dimension-dependency (in $d$) for the probability flow ODE sampler.  
	Imposing only minimal assumptions on the target data distribution (e.g., no smoothness assumption is imposed), our results also characterize how $\ell_2$ score estimation errors affect the quality of the data generation processes.  
	In contrast to prior works, our theory is developed based on an elementary yet versatile non-asymptotic approach without the need of resorting to SDE and ODE toolboxes. 
 
\end{abstract}







\setcounter{tocdepth}{2}
\tableofcontents

\section{Introduction}

Diffusion models have emerged as a cornerstone in contemporary generative modeling, a task that learns to generate new data instances (e.g., images, text, audio) that look similar in distribution to the training data \citep{ho2020denoising,sohl2015deep,song2019generative,dhariwal2021diffusion,jolicoeur2021adversarial,chen2021wavegrad,kong2021diffwave,austin2021structured}.  
Originally proposed by \citet{sohl2015deep} and later popularized by \citet{song2019generative,ho2020denoising}, 
the mainstream diffusion generative models --- e.g., denoising diffusion implicit models (DDIMs) \citep{song2020denoising} 
and denoising diffusion probabilistic models (DDPMs) \citep{ho2020denoising}   
 --- have underpinned major successes in content generators like DALL$\cdot$E \citep{ramesh2022hierarchical}, Stable Diffusion \citep{rombach2022high} 
and Imagen \citep{saharia2022photorealistic}, 
claiming state-of-the-art performance in the now broad field of generative artificial intelligence (AI). 
See  \citet{yang2022diffusion,croitoru2023diffusion,chen2024overview} for overviews of recent development.

In a nutshell, a diffusion generative model is based upon two stochastic processes in $\real^d$: 
\begin{itemize}
	\item[1)] a forward process 
		\begin{equation}
			X_0 \rightarrow X_1 \rightarrow \cdots \rightarrow X_T
			\label{eq:forward-process-informal}
		\end{equation}
		that starts from a sample drawn from the target data distribution (e.g., of natural images) 
		and gradually diffuses it into a noise-like distribution (e.g., standard Gaussians); 
	\item[2)] a reverse process 
		\begin{equation}
			Y_T \rightarrow Y_{T-1} \rightarrow \cdots \rightarrow Y_0
			\label{eq:reverse-process-informal}
		\end{equation}
		that starts from pure noise (e.g., standard Gaussians) and successively converts it into new samples sharing similar distributions as the target data distribution. 
\end{itemize}
\noindent 
Transforming data into noise in the forward process is straightforward, often hand-crafted by increasingly injecting more noise into the data at hand.   
What is challenging is the construction of the reverse process: 
how to generate the desired information out of pure noise? 
To do so, a diffusion model learns to build a reverse process \eqref{eq:reverse-process-informal}
 that imitates the dynamics of the forward process \eqref{eq:forward-process-informal} in a time-reverse fashion; 
 more precisely, the design goal is to ascertain  distributional proximity\footnote{Two random vectors $X$ and $Y$ are said to obey $X \overset{\mathrm{d}}{=} Y$ (resp.~$X \overset{\mathrm{d}}{\approx} Y$) if they are equivalent (resp.~close) in distribution.} 
\begin{equation}
	Y_t \,\overset{\mathrm{d}}{\approx}\, X_t, \qquad t = T,\cdots,1	
\end{equation}
through proper learning based on how the training data propagate in the forward process. 
Encouragingly, there often exist feasible strategies to achieve this goal  
as long as faithful estimates about the (Stein) score functions --- the gradients of the log marginal density of the forward process --- are available \citep{anderson1982reverse,haussmann1986time}. 
Viewed in this light, a diverse array of diffusion models are frequently referred to as {\em score-based generative modeling (SGM)}. 
The popularity of SGM was initially motivated by, and has since further inspired,  numerous recent studies on the problem of learning score functions, 
a subroutine that also goes by the name of score matching (e.g., \citet{hyvarinen2005estimation,hyvarinen2007some,vincent2011connection,song2020sliced,koehler2022statistical}).

Nonetheless, despite the mind-blowing empirical advances, a mathematical theory for diffusion generative models is still in its infancy. 
Given the complexity of developing a full-fledged end-to-end theory, a divide-and-conquer approach has been advertised, 
decoupling the score learning phase (i.e., how to estimate score functions reliably from training data) 
and the generative sampling phase (i.e., how to generate new data instances given the score estimates). 
In particular, 
the past few years have witnessed growing interest and remarkable progress from the theoretical community
towards understanding the generative sampling phase 
\citep{block2020generative,de2021diffusion,liu2022let,de2022convergence,lee2023convergence,pidstrigach2022score,chen2022sampling,chen2022improved,chen2023restoration,tang2024contractive,tang2024score,pedrotti2023improved,liang2024non,li2024adapting}. 
For instance, 
polynomial-time convergence guarantees have been established for stochastic samplers (e.g., \citet{chen2022sampling,chen2022improved,benton2023linear,li2023towardsICLR,tang2024contractive,li2024accelerating,mbacke2023note,liang2024non,li2024adapting}) and deterministic samplers (e.g., \citet{chen2023restoration,benton2023error,li2023towardsICLR,gao2024convergence2,li2024accelerating,huang2024convergence}),  both of which accommodated a fairly general family of data distributions.


\paragraph{This paper.} 
The present paper contributes to this growing list of theoretical endeavors by 
developing non-asymptotic convergence theory for a popular deterministic sampler \citep{song2020score} --- originally proposed based on a sort of ordinary differential equations (ODEs) for the reverse process called probability flow ODEs or diffusion ODEs, closely related to the DDIM sampler \citep{song2020denoising}. 
For concreteness, we prove that the iteration complexity
is no larger than the order of 
\begin{align}
	(\text{iteration complexity}):\quad d/\varepsilon
\end{align}
(up to some logarithmic factor and lower-order term), with $d$ the data dimension and $\varepsilon$ the target accuracy level in total-variation (TV) distance. 
We impose only minimal assumptions on the target distribution (e.g., no smoothness condition is needed), 
and quantify the impact of $\ell_2$ score estimation errors upon convergence.   
In comparisons to past works, our main contributions are as follows.

\begin{itemize}

	\item {\em Linear $d$-dependency.} 
		Our iteration complexity scales nearly linearly in the dimension $d$, 
		which improves upon all prior theoretical guarantees for deterministic samplers \citep{li2023towards,chen2023restoration,huang2024convergence}; in fact, the state-of-the-art $d$-dependency before our work scales with $d^2$
		\citep{li2023towards,huang2024convergence}. 
		Note that $d$-linear convergence theory was established for the stochastic sampler DDPM \citep{benton2023linear};  the theoretical framework for DDPM is not applicable for analyzing probability flow ODEs, but the use of a stochastic localization result in \citet{benton2023linear} motivates our approach in sharpening the $d$ dependency. 
		 Additionally,  our result does not exhibit exponential dependency on the smoothness or regularity conditions as in \citet{chen2023restoration,benton2023error} (e.g., the regularity parameter used in \citet{benton2023error} might even scale with the dimension $d$).

	\item {\em Linear dependency on $1/\varepsilon$.}  
		We derive an iteration complexity upper bound that is proportional to $1/\varepsilon$. 
		Note that this was already accomplished in an earlier version of this work \citep{li2023towards}, 
		strengthening prior convergence guarantees considerably \citep{chen2023restoration}.  
		This scaling $1/\varepsilon$ was also proven by a recent work 
		\citet{huang2024convergence} via a completely different ODE-based approach.


	\item {\em $\ell_2$ score estimation errors for the deterministic sampler.} 
		Our theory reveals that the TV distance between $X_1$ and $Y_1$ is proportional to the $\ell_2$ score estimation error as well as the associated mean Jacobian errors, an appealing property already established in an earlier version of this work \citep{li2023towards}. 
		In comparison, prior theoretical results   
		either study stochastic variations of this deterministic sampler \citep{chen2023probability} (so that the samplers are no longer the original deterministic sampler) or fall short of accommodating discretization errors \citep{benton2023error}, 
		with the only exception being the recent work \citet{huang2024convergence} that also accounts for score errors for deterministic samplers.

	\item {\em An elementary analysis framework.}  
		From the technical point of view, the analysis framework laid out in this paper is fully non-asymptotic in nature. 
		In contrast to prior theoretical analyses that take a detour to study the continuum limits and then control the discretization error, 
		our approach tackles the discrete-time processes directly using elementary analysis strategies.  
		No knowledge on SDEs or ODEs is needed for establishing our theory, 
		resulting in a versatile framework and sometimes lowering the technical barrier towards understanding diffusion models (for those with no background in SDEs/ODEs). 
		

\end{itemize}

\noindent It is worth emphasizing that our analysis for the probability flow ODE differs drastically from the analysis for DDPM \citep{chen2022sampling,chen2022improved,benton2023linear}.  
More concretely, the state-of-the-art analysis for DDPM \citep{benton2023linear} is built upon 
the Girsanov theorem, a hammer that provides a powerful way to control the Kullback-Leibler (KL) divergence between the forward process and the sampling process. 
This approach, however, is known to be inapplicable to ODE-based deterministic samplers, 
given that the aforementioned KL divergence might even approach infinity. 
Working backward, 
our proof attempts to track the proximity of $p_{X_t}$ and $p_{Y_t}$ 
by iteratively computing how $p_{X_t}/p_{Y_t}$ evolves from $p_{X_{t+1}}/p_{Y_{t+1}}$.

%
%
%
%
%

\paragraph{Notation.}
Before proceeding, we introduce a couple of notation to be used throughout. 
For any two functions $f(d,T)$ and $g(d,T)$, 
we adopt the notation $f(d,T)\lesssim g(d,T)$ or $f(d,T)=O( g(d,T) )$ (resp.~$f(d,T)\gtrsim g(d,T)$)
to mean that there exists some universal constant $C_1>0$ such that $f(d,T)\leq C_1 g(d,T)$ (resp.~$f(d,T)\geq C_1 g(d,T)$) for all $d$ and $T$; moreover, the notation  $f(d,T)\asymp g(d,T)$ indicates that $f(d,T)\lesssim g(d,T)$ and $f(d,T)\gtrsim g(d,T)$ hold at once. 
The notation $\widetilde{O}(\cdot)$ is defined similar to $O(\cdot)$ except that it hides the logarithmic dependency. 
Additionally, the notation  $f(d,T)=o\big( g(d,T) \big)$ means that $f(d,T)/g(d,T) \rightarrow 0$ as $d,T$ tend to infinity. 
We shall often use capital letters to denote random variables/vectors/processes, and lowercase letters for deterministic variables. 
For any two probability measures $P$ and $Q$, the TV distance between them is defined to be $\mathsf{TV}(P,Q)\coloneqq \frac{1}{2}\int |\mathrm{d}P - \mathrm{d}Q|$. 
Throughout the paper, $p_X(\cdot)$ (resp.~$p_{X\mymid Y}(\cdot \mymid \cdot)$) denotes the probability density function of $X$ (resp.~$X$ given $Y$). 
For any matrix $A$, we denote by $\|A\|$ (resp.~$\|A\|_{\mathrm{F}}$) the spectral norm (resp.~Frobenius norm) of $A$. 
Also, for any vector-valued function $f$, we let $J_f$ or $\frac{\partial f}{\partial x}$ represent the Jacobian matrix of $f$.





\section{Preliminaries} 
\label{sec:preliminaries}

In this section, we introduce the basics of diffusion generative models. 
The ultimate goal of a generative model can be concisely stated: 
given  data samples drawn from an unknown distribution of interest $\Pdata$ in $\real^d$, 
we wish to generate new samples whose distributions closely resemble $\Pdata$.



\subsection{Diffusion generative models}

Towards achieving the above goal, 
a diffusion generative model typically encompasses two Markov processes: a forward process and a reverse process, as described below.

\paragraph{The forward process.}
In the forward chain, one progressively injects noise into the data samples to diffuse and obscure the data. 
The distributions of the injected noise are often hand-picked, 
with the standard Gaussian distribution receiving widespread adoption. 
More specifically, the forward Markov process produces a sequence of $d$-dimensional random vectors $X_1\rightarrow X_2\rightarrow \cdots\rightarrow X_T $ as follows:   
\begin{subequations}
\label{eq:forward-process}
\begin{align}
	X_0 &\sim \Pdata,\\
	X_t &= \sqrt{1-\beta_t}X_{t-1} + \sqrt{\beta_t}\,W_{t}, \qquad 1\leq t\leq T,
\end{align}
\end{subequations}
where $\{W_t\}_{1\leq t\leq T}$ 
indicates a sequence of independent noise vectors drawn from $W_{t} \overset{\mathrm{i.i.d.}}{\sim} \mathcal{N}(0, I_d)$. 
The hyper-parameters $\{\beta_t \in (0,1)\}$ represent prescribed learning rate schedules 
that control the variance of the noise injected in each step. 
If we define 
\begin{align}
	\alpha_t\coloneqq 1 - \beta_t, 
	\qquad \overline{\alpha}_t \coloneqq \prod_{k = 1}^t \alpha_k ,\qquad 1\leq t\leq T,
\end{align}
then it can be straightforwardly verified that for every $1\leq t\leq T$, 
\begin{align}
\label{eqn:Xt-X0}
	X_t = \sqrt{\overalpha_t} X_{0} + \sqrt{1-\overalpha_t} \,\overline{W}_{t}
	\qquad \text{for some } \overline{W}_{t}\sim \mathcal{N}(0,I_d) .
\end{align}
%
Clearly, if the covariance of $X_0$ is also equal to $I_d$, then the covariance of $X_t$ is preserved throughout the forward process; 
for this reason, this forward process \eqref{eq:forward-process} is sometimes referred to as variance-preserving \citep{song2020score}. 
Throughout this paper, we employ the notation
\begin{equation}
	q_t \coloneqq \mathsf{distribution}\big( X_t \big)
	\label{eq:defn-qt}
\end{equation}
to denote the distribution of $X_t$. 
As long as  $\overalpha_{T}$ is vanishingly small, 
one has the following property for a fairly general family of data distributions: 
\begin{equation}
	q_T \approx \mathcal{N}(0,I_d). 
\end{equation}
%






\paragraph{The reverse process.} 
The reverse chain $Y_T\rightarrow Y_{T-1}\rightarrow \ldots\rightarrow Y_1 $ is designed to (approximately) revert the forward process, 
allowing one to transform pure noise into new samples with matching distributions as the original data. 
To be more precise, by initializing it as
\begin{subequations}
	\label{eq:goal-reverse-process}
\begin{equation}
	Y_T \sim \mathcal{N}(0,I_d), 
\end{equation}
we seek to design a reverse-time process with nearly identical marginals as the forward process,  namely, 
\begin{equation}
	(\text{goal})\qquad\qquad
	Y_t \,\overset{\mathrm{d}}{\approx}\, X_t, \qquad t = T, T-1,\cdots, 1.  
\end{equation}
\end{subequations}
Throughout the paper, we shall often employ the following notation to indicate the distribution of $Y_t$: 
\begin{equation}
	p_t \coloneqq \mathsf{distribution}\big( Y_t \big) .
	\label{eq:defn-pt}
\end{equation}

\subsection{The probability flow ODE}
\label{sec:det-stochastic-samplers}

Evidently, the most crucial step of the diffusion model lies in effective design of the reverse process. 
The data-generation process of a deterministic sampler typically proceeds as follows: 
starting from $Y_T \sim \mathcal{N}(0,I_d)$, one selects a set of functions $\{\Phi_t(\cdot)\}_{1\leq t\leq T}$ and computes:
	\begin{subequations}
\label{eqn:ode-sampling}
	\begin{equation}
		Y_T\sim \mathcal{N}(0,I_d),\qquad 
		Y_{t-1}=\Phi_{t}\big(Y_{t}\big)\quad~ \text{for }  t= T,\cdots,1.
		\label{eqn:ode-sampling-Y}
	\end{equation}
	Clearly, the sampling process is fully deterministic except for the initialization $Y_T$. 
	Suppose now that we are armed with the estimates $\{s_t(\cdot)\}_{1\leq t\leq T}$ for the log density functions $\{s_t^{\star}(\cdot)\coloneqq \nabla\log q_{t}(\cdot)\}_{1\leq t\leq T}$ --- often referred to as the (Stein) score functions. 
Then a discrete-time version of the probability flow ODE approach (cf.~\eqref{eq:prob-flow-ODE}) adopts the following mapping: 
%
%
\begin{align}
\label{eqn:phi-func}
	\Phi_t(x) \coloneqq 
	\frac{1}{\sqrt{\alpha_t}} \bigg( x + \frac{1-\alpha_{t}}{2}s_{t}(x) \bigg). 
\end{align}
\end{subequations}
This approach, based on the probability flow ODE \eqref{eq:prob-flow-ODE}, often achieves faster sampling compared to the stochastic counterpart like DDPM \citep{song2020score}.  



%


	%
%

In order to elucidate the plausibility of a deterministic approach, 
we find it helpful to look at the continuum limit through the lens of SDEs and ODEs. 
It is worth emphasizing, however, that the development of our main theory does not rely on knowledge of SDEs and ODEs.   
\begin{itemize}
	\item {\em The forward process.} 
		A continuous-time analog of the forward diffusion process can be modeled as 
		\begin{equation}
			\mathrm{d}X_{t}=f(X_{t},t)\mathrm{d}t+g(t)\mathrm{d}W_{t} \quad (0\leq t\leq T),
			\qquad X_{0}\sim\Pdata
			\label{eq:forward-SDE-general}
		\end{equation}	
		for some functions $f(\cdot,\cdot)$ and $g(\cdot)$ (denoting respectively the drift and diffusion coefficient), 
		where $W_t$ denotes a $d$-dimensional standard Brownian motion. 
		As a special example, the continuum limit of \eqref{eq:forward-process} takes the following form\footnote{To see its connection with \eqref{eq:forward-process}, it suffices to derive from \eqref{eq:forward-process} that 
$X_{t}-X_{t-\mathrm{d}t}=\sqrt{1-\beta_{t}}X_{t-\mathrm{d}t}-X_{t-\mathrm{d}t}+\sqrt{\beta_{t}}W_{t}\approx-\frac{1}{2}\beta_{t}X_{t-\mathrm{d}t}+\sqrt{\beta_{t}}W_{t}$.} \citep{song2020score}
		\begin{equation}
			\mathrm{d}X_{t}= - \frac{1}{2} \beta(t) X_t \mathrm{d}t+ \sqrt{\beta(t)}\,\mathrm{d}W_{t} \quad (0\leq t\leq T),
			\qquad X_{0}\sim\Pdata
			\label{eq:forward-SDE}
		\end{equation}	
		for some function $\beta(t)$.  
		As before, we denote by $q_t$ the distribution of $X_t$ in \eqref{eq:forward-SDE-general}.

	\item {\em The reverse process.} 
		As it turns out, 
		there exist reverse processes capable of reconstructing the marginal distribution of the forward process. 
		In particular, the {\em probability flow ODE} is a reverse process taking the following form \citep{song2020score}
				\begin{align}
					\mathrm{d}Y_{t}^{\mathsf{ode}} 
					& =\Big(-f\big(Y_{t}^{\mathsf{ode}},T-t\big)+\frac{1}{2}g(T-t)^{2}\nabla\log q_{T-t}\big(Y_{t}^{\mathsf{ode}}\big)\Big)
					\mathrm{d}t\quad(0\leq t\leq T),\qquad Y_{0}^{\mathsf{ode}}\sim q_{T}, 
					\label{eq:prob-flow-ODE}
				\end{align}
		where we use $\nabla \log q_{t}(X)$ to abbreviate $\nabla_X \log q_{t}(X)$ for notational simplicity.
		This ODE exhibits matching distributions with the forward process in that  
				\begin{align*}
					Y_{T-t}^{\mathsf{ode}}  \, \overset{\mathrm{d}}{=} \, X_t, \qquad 0\leq t\leq T. 
				\end{align*}
				As can be easily shown, 
				the continuous-time limit of \eqref{eqn:ode-sampling} falls under this category. 
				Note that this family of deterministic samplers is closely related to the DDIM sampler \citep{karraselucidating,song2020score}.

		%


%
		%

\end{itemize}

\noindent 
Interestingly, in addition to the functions $f$ and $g$ that define the forward process, construction of \eqref{eq:prob-flow-ODE} 
relies only upon knowledge of the (Stein) score function $\nabla\log q_{t}(\cdot)$ of the intermediate steps of the forward diffusion process, 
an intriguing fact that also holds when designing stochastic samplers like DDPM. 
Consequently, a key enabler of diffusion models lies in reliable learning of the score function, 
and hence the name {\em score-based generative modeling}.

\section{Convergence theory for the probability flow ODE sampler}
\label{sec:main-results}

In this section, we analyze the probability flow ODE sampler in discrete time.  
While the proofs for our main theory are all postponed to the appendix, 
it is worth emphasizing upfront that our analysis framework directly tackles the discrete-time processes without 
the need of resorting to any toolbox of SDEs and ODEs 
tailored to the continuous-time limits. 
This elementary approach might potentially be versatile for analyzing a broad class of variations of these samplers. 

\subsection{Assumptions and learning rates}

Before proceeding, we impose some assumptions on the score estimates and the target data distributions, and specify the hyper-parameters $\{\alpha_t\}$
 that shall be adopted throughout all cases.

\paragraph{Score estimates.}
Given that the score functions are an essential component in score-based generative modeling, 
we assume access to faithful estimates of the score functions $\nabla \log q_t(\cdot)$ across all intermediate steps $t$, 
thus disentangling the score learning phase and the data generation phase. 
Towards this end, let us first formally introduce the true score function as follows. 
\begin{definition}[Score function]
\label{defition:score}
The score function, denoted by $s_t^{\star}: \real^d\rightarrow \real^d$ ($1\leq t\leq T$), is defined as 
\begin{align}
\label{eqn:training-score}
	s_{t}^{\star}(X) \coloneqq 
	\nabla \log q_{t}(X) , 
	\qquad 1\leq t\leq T. 
\end{align}
\end{definition}
As has been pointed out by previous works concerning score matching (e.g., \citet{hyvarinen2005estimation,vincent2011connection,chen2022sampling}), 
the score function $s_{t}^{\star}$ admits an alternative form as follows (owing to properties of Gaussian distributions): 
\begin{equation}
s_{t}^{\star}\coloneqq\arg\min_{s:\real^{d}\rightarrow\real^{d}}\mathop{\mathbb{E}}_{W\sim\mathcal{N}(0,I_{d}),X_{0}\sim\Pdata}\Bigg[\bigg\| s\big(\sqrt{\overline{\alpha}_{t}}X_{0}+\sqrt{1-\overline{\alpha}_{t}}W\big)+\frac{1}{\sqrt{1-\overline{\alpha}_{t}}}W\bigg\|_{2}^{2}\Bigg],
	\label{eqn:training-score-equiv}
\end{equation}
which takes the form of the minimum mean square error estimator for $-\frac{1}{\sqrt{1-\overline{\alpha}_{t}}}W$ 
given $\sqrt{\overline{\alpha}_{t}}X_{0}+\sqrt{1-\overline{\alpha}_{t}}W$
and is often more amenable to training.

With Definition~\ref{defition:score} in place, 
we can readily introduce the following assumptions that capture the quality of the score estimate $\{s_t\}_{1\leq t\leq T}$ we have available. 
\begin{assumption}
\label{assumption:score-estimate}
Suppose that the score function estimate $\{s_t\}_{1\leq t\leq T}$ obeys
\begin{align}
\label{eqn:score-estimate}
	\frac{1}{T}\sum_{t = 1}^T \mathop{\mathbb{E}}_{X\sim q_t}\Big[ \big\| s_t(X) - s_t^{\star}(X) \big\|_2^2\Big] \le \varepsilon_{\mathsf{score}}^2.
\end{align}
\end{assumption}
\begin{assumption}
\label{assumption:score-estimate-Jacobi}
	For each $1\leq t\leq T$, assume that $s_t(\cdot)$ is continuously differentiable, and 
	denote by $J_{s_t^{\star}} = \frac{\partial s_t^{\star}}{\partial x}$ and $J_{s_t} = \frac{\partial s_t}{\partial x}$ 
	the Jacobian matrices of $s_t^{\star}(\cdot)$ and $s_t(\cdot)$, respectively. 
	Assume that the score function estimate $\{s_t\}_{1\leq t\leq T}$ obeys
\begin{align}
\label{eqn:score-estimate-Jacobi}
\qquad\qquad
	\frac{1}{T}\sum_{t = 1}^T \mathop{\mathbb{E}}_{X\sim q_t}\Big[\big\| J_{s_t}(X) - J_{s_t^{\star}} (X)  \big\|\Big] \le \varepsilon_{\mathsf{Jacobi}}.
\end{align}
\end{assumption}
In a nutshell, 
Assumption~\ref{assumption:score-estimate} reflects the $\ell_2$ score estimation error, 
whereas Assumption~\ref{assumption:score-estimate-Jacobi} is concerned with the estimation error in terms of the corresponding Jacobian matrix (so as to ensure certain continuity of the score estimator).  
Both assumptions consider the {\em average} estimation errors over all $T$ steps. 
As we shall see momentarily, our theory for the deterministic sampler relies on both Assumptions~\ref{assumption:score-estimate} and \ref{assumption:score-estimate-Jacobi}, 
while the theory for the stochastic sampler requires only Assumption~\ref{assumption:score-estimate}. 
We shall discuss in Section~\ref{sec:ODE-basic} the insufficiency of Assumption~\ref{assumption:score-estimate} alone for the probability flow ODE sampler.

\paragraph{Target data distributions.}  
Our goal is to uncover the effectiveness of diffusion models in generating a broad family of data distributions. 
Throughout this paper, the only assumptions we need to impose on the target data distribution $\Pdata$ are the following: 
\begin{itemize}
	\item $X_0$ is an absolutely continuous random vector,
		and
\begin{equation}
	\mathbb{P}\big(\|X_0\|_2 \leq R = T^{c_R}   \big)=1, \qquad X_0\sim \Pdata 
	\label{eq:assumption-data-bounded}
\end{equation}
for some arbitrarily large constant $c_R>0$. 
\end{itemize}
This assumption allows the radius of the support of $\Pdata$ to be exceedingly large (given that the exponent $c_R$ can be arbitrarily large).


\paragraph{Learning rate schedule. }
Let us also take a moment to specify the learning rates to be used for our theory and analyses. 
For some large enough numerical constants $c_0,c_1 > 0$, we set
\begin{subequations}
\label{eqn:alpha-t}
\begin{align}
	\beta_{1} & =1-\alpha_{1} = \frac{1}{T^{c_0}};\\
\beta_{t} & =1-\alpha_{t} = \frac{c_{1}\log T}{T}\min\bigg\{\beta_{1}\Big(1+\frac{c_{1}\log T}{T}\Big)^{t},\,1\bigg\}.
\end{align}
\end{subequations}

In words, our choice of $\{\beta_t\}$ undergoes two phases: at the beginning (when $t$ is small), $\beta_t$ exhibits exponential increase; once it reaches the level of $\frac{c_1\log T}{T}$, it stays flat for the remaining steps. 
This two-phase choice shares similarity with the choice adopted in prior diffusion model theory like \citet{benton2023linear}.

\subsection{Main results} 
\label{sec:ODE-basic}




We are now ready to present our non-asymptotic convergence guarantee --- measured by the total variation distance between the forward and the reverse processes --- for the discrete-time version \eqref{eqn:ode-sampling} of the probability flow ODE. 
The proof of our theory is postponed to Section~\ref{sec:pf-theorem-ode}.


\begin{theos}
\label{thm:main-ODE}
Suppose that \eqref{eq:assumption-data-bounded} holds true. 
Assume that the score estimates $s_t(\cdot)$ $(1 \le t \le T)$ satisfy Assumptions~\ref{assumption:score-estimate} and \ref{assumption:score-estimate-Jacobi}.
Then the sampling process \eqref{eqn:ode-sampling} with the learning rate schedule \eqref{eqn:alpha-t} satisfies
\begin{align}
	\mathsf{TV}\big(q_1, p_1\big) \leq C_1 \frac{d\log^4 T}{T} 
	+C_1\sqrt{d\log^{4}T}\,\varepsilon_{\score}+ C_1d(\log^2 T)\varepsilon_{\Jacobi}
\label{eq:ratio-ODE}
\end{align}
%
%
for some universal constants $C_1>0$, provided that $T \ge C_2d^2\log^5 T$ for some large enough constant $C_2>0$. Here, we recall that $p_1$ (resp.~$q_1$) represents the distribution of $Y_1$ (resp.~$X_1$). 
\end{theos}
%

%
Let us remark on the main implications of Theorem~\ref{thm:main-ODE}, as well as several points worth discussing. 
Before proceeding, we shall note that our theory is concerned with convergence to $q_1$. 
Given that $X_1 \sim q_1$ and $X_0 \sim q_0$ are very close due to the choice of $\alpha_1$, 
focusing on the convergence w.r.t.~$q_1$ instead of $q_0$ remains practically relevant.

\paragraph{Iteration complexity.} 
Consider first the scenario that has access to perfect score estimates (i.e., $\varepsilon_{\score}=0$). 
In order to achieve $\varepsilon$-accuracy (in the sense that $\mathsf{TV}(q_1, p_1)\leq \varepsilon$), 
the number of steps $T$ only needs to exceed
%
\begin{equation}
	\widetilde{O}\bigg( \frac{d}{\varepsilon} \bigg)
	\label{eq:iteration-complexity-ODE}
\end{equation}
for small enough accuracy level $\varepsilon$. 
As far as we know, this is the first result that unveils linear dimension dependency for the probability flow ODE sampler. 
Note that our theory is established without assuming any sort of smoothness or log-concavity on the target data distribution.

\paragraph{Stability.} Turning to the more general case with imperfect score estimates (i.e., $\varepsilon_{\score}>0$), 
		the deterministic sampler \eqref{eqn:ode-sampling} yields a distribution 
		whose distance to the target distribution (measured again by the TV distance) scales proportionally with $\varepsilon_{\score}$ and $\varepsilon_{\Jacobi}$. 
		It is noteworthy that in addition to the $\ell_2$ score estimation errors, 
		we are in need of an assumption on the stability of the associated Jacobian matrices, 
		which plays a pivotal in ensuring that the reverse-time deterministic process does not deviate considerably from the desired process.

\paragraph{Insufficiency of the score estimation error assumption alone.}
The careful reader might wonder why we are in need of additional assumptions beyond the $\ell_2$ score error stated in 
Assumption~\ref{assumption:score-estimate}. 
To answer this question, we find it helpful to look at a simple example below. 
\begin{itemize}
	\item {\bf Example.}   
		Consider the case where $X_{0}\sim\mathcal{N}(0,1)$, and hence $X_{1}\sim\mathcal{N}(0,1)$. Suppose that the reverse process for time $t=2$ can lead to the desired distribution if exact score function is employed, namely, 
\[
Y_{1}^{\star}\coloneqq\frac{1}{\sqrt{\alpha_{2}}}\left(Y_{2}-\frac{1-\alpha_{2}}{2}s_{2}^{\star}(Y_{2})\right)\sim\mathcal{N}(0,1). 
\]
Now, suppose that the score estimate $s_{2}(\cdot)$ we have available obeys
\[
	s_{2}(y_{2})=s_{2}^{\star}(y_{2})+\frac{2\sqrt{\alpha_{2}}}{1-\alpha_{2}}\left\{ y_{1}^{\star}-L\left\lfloor \frac{y_{1}^{\star}}{L}\right\rfloor \right\} 
		\qquad \text{with } y_{1}^{\star}\coloneqq\frac{1}{\sqrt{\alpha_{2}}}\left(y_{2}-\frac{1-\alpha_{2}}{2}s_{2}^{\star}(y_{2})\right)
\]
for some $L>0$, where $\lfloor z \rfloor$ is the greatest  integer not exceeding $z$. It follows that
\[
	Y_{1}=Y_{1}^{\star}+\frac{1-\alpha_{2}}{2\sqrt{\alpha_{2}}}\big[ s_{2}^{\star}(Y_{2})-s_{2}(Y_{2})\big]  =L\left\lfloor \frac{Y_{1}^{\star}}{L}\right\rfloor .
\]
Clearly, the score estimation error $\mathbb{E}_{X_2\sim \mathcal{N}(0,1)}\big[|s_{2}(X_{2})-s_{2}^{\star}(X_{2})|^2\big]$
can be made arbitrarily small by taking $L$ to be sufficiently small. 
However, the discrete nature of $Y_{1}$ forces the TV distance to be
\[
	\mathsf{TV}(Y_{1},X_{1}) = 1. 
\]
%
\end{itemize}
The above example demonstrates that, for the deterministic sampler, the TV distance between $Y_1$ and $X_1$ might not improve as the score error decreases. 
This is in stark contrast to the stochastic sampler like DDPM. 
If we wish to eliminate the need of imposing Assumption~\ref{assumption:score-estimate-Jacobi}, 
one potential way is to resort to other metrics (e.g., the Wasserstein distance) instead of the TV distance between $Y_1$ and $X_1$.

\paragraph{Support size of $\Pdata$.} 
It is noteworthy that our theory holds true even when the support size of the target distribution is polynomially large (see \eqref{eq:assumption-data-bounded}). This implies that careful normalization of the target data is often unnecessary.  
%
%
Furthermore, we note that the assumption~\eqref{eq:assumption-data-bounded} can also be relaxed. 
Supposing that $\mathbb{P}\big(\|X_0\|_2 \leq  B \mid X_0\sim \Pdata \big)=1$ for some quantity $B>0$ (which is allowed to grow faster than a polynomial in $T$), we can readily extend our analysis to obtain 
{
\begin{align*}
	\mathsf{TV}\big(q_1, p_1\big) \leq C_1 \left( \frac{d}{T} 
	+\sqrt{d} \,\varepsilon_{\score}+ d\varepsilon_{\Jacobi} \right) \mathrm{polylog}(T,B).
\end{align*}
}
Importantly, the convergence rate depends only logarithmically in $B$.



\paragraph{Comparisons to previous works.}
Next, let us compare our results with past works. 
\begin{itemize}
	\item 
The first analysis for the discretized probability flow ODE approach in
prior literature was derived by a recent work \citet{chen2023restoration}, which established  non-asymptotic convergence guarantees that exhibit polynomial dependency in both $d$ and $1/\varepsilon$ (see, e.g.,  \citet[Theorem~4.1]{chen2023restoration}). However, it fell short of providing concrete polynomial dependency in $d$ and $1/\varepsilon$,  suffered from exponential dependency in the Lipschitz constant of the score function, and relied on exact score estimates.  
		In contrast, our result in Theorem~\ref{thm:main-ODE} uncovers a concrete $\widetilde{O}(d/\varepsilon)$ scaling (ignoring lower-order and logarithmic terms) without imposing any smoothness assumption on the target data distribution, 
and makes explicit the effect of $\ell_2$ score estimation errors, both of which were previously unavailable for such discrete-time deterministic samplers. 
	\item 
\citet{benton2023error} studied the convergence of the probability flow ODE approach  
without accounting for the discretization error. The result therein also exhibited exponential dependency on a certain Lipschitz constant w.r.t.~the forward flow and a regularity parameter (denoted by $\lambda$ therein, which might scale with the dimension $d$).

	\item 
	 \citet{chen2023probability} studied two variants of the probability flow ODE.
	By inserting an additional stochastic corrector step --- based on overdamped (resp.~underdamped) Langevin diffusion --- in each iteration of the probability flow ODE (so strictly speaking, these variations are no longer deterministic samplers), 
\citet{chen2023probability} showed that $\widetilde{O}(L^3d/\varepsilon^2)$ (resp.~$\widetilde{O}(L^2\sqrt{d}/\varepsilon)$) steps are sufficient, where $L$ denotes the Lipschitz constant of the score function. 
		In comparison, our result demonstrates for the first time that the plain probability flow ODE already achieves the $\widetilde{O}(d/\varepsilon)$ scaling without requiring either corrector steps or smoothness assumptions. 

	\item The very recent work \citet{huang2024convergence} developed a novel suite of theory for the probability flow ODE, 
		accounting for $p$-th ($p\geq 1$) order Runge-Kutta integrators (so as to demonstrate the degree of acceleration based on higher-order ODEs). When $p=1$, the algorithm resembles what we analyze herein; let us make comparisons for this case in the following. 
		The iteration complexity derived by \citet{huang2024convergence} scales as $\widetilde{O}(d^2/\varepsilon)$, 
		whereas we obtain a sharper bound $\widetilde{O}(d/\varepsilon)$.  
		In addition, the iteration complexity in \citet{huang2024convergence} scales quadratically in the support size of the target distribution, while our theory allows the support size to be polynomially large without affecting the iteration complexity. 
		Moreover, the TV distance bound in \citet{huang2024convergence} scales proportionally to $d^{3/4}\varepsilon_{\mathsf{score}}$ (in addition to other multiplicative factors like the support size and Lipschitz constants), which is weaker than our result  $\sqrt{d}\varepsilon_{\mathsf{score}}$ in terms of the $d$-dependency.

\end{itemize}
\noindent 
Another recent work \citet{gao2024convergence2} established the first non-asymptotic theory for the probability flow ODE in 2-Wasserstein distance. The results therein require the target data distribution to satisfy strong log-concavity though.

Finally, let us briefly compare our result with the theory for the popular stochastic sampler: DDPM. The state-of-the-art convergence theory \citet{benton2023linear} reveals that the iteration complexity for DDPM scales as $\widetilde{O}(d/\varepsilon^2)$, 
which exhibits worse $\varepsilon$-dependency compared to our theory for the probability flow ODE. 

\section{Other related works}
\label{sec:related-works}

\medskip

\paragraph{Convergence theory for diffusion models.}
Early theoretical efforts in understanding the convergence of score-based stochastic samplers suffered from being either not quantitative \citep{de2021diffusion,liu2022let,pidstrigach2022score}, or the curse of dimensionality (e.g., exponential dependencies in the convergence guarantees) \citep{block2020generative,de2022convergence}.
The recent work \citet{lee2022convergence} provided the first polynomial convergence guarantee in the presence of $\ell_2$-accurate score estimates, for any smooth distribution satisfying the log-Sobelev inequality.
\citet{chen2022sampling,lee2023convergence,chen2022improved} subsequently lifted such a stringent data distribution assumption. 
More concretely,  \citet{chen2022sampling} accommodated a broad family of data distributions under the premise that the score functions over the entire trajectory of the forward process are Lipschitz; \citet{lee2023convergence} only required certain smoothness assumptions but came with worse dependence on the problem parameters; and more recent results in \citet{chen2022improved,benton2023linear} applied to literally any data distribution with bounded second-order moment. In addition, \citet{wibisono2022convergence} also established a convergence theory for score-based generative models, assuming that the error of the score estimator has a bounded moment generating function and that the data distribution satisfies the log-Sobelev inequality. 
The recent work \citet{li2024adapting} further showed that DDPM can automatically adapt to intrinsic low dimensionality of the target distribution and converge faster.  
Turning attention to samplers based on the probability flow ODE, \citet{chen2023restoration} derived the first non-asymptotic bounds for this type of samplers.  
Improved convergence guarantees have recently been provided by a concurrent work \citet{chen2023probability}, with the assistance of additional corrector steps inerspersed in each iteration of the probability flow ODE. 
It is worth noting that the corrector steps proposed therein are based on Langevin-type diffusion and inject additive noise, and hence the resulting sampling processes are not deterministic. 
Additionally,  theoretical justifications for DDPM in the context of image in-painting have been developed by \citet{rout2023theoretical}.
Moreover, convergence results based on the Wasserstein distance have recently  been derived as well (e.g., \citet{tang2024contractive,benton2023error}), although these results typically exhibit exponential dependency on the Lipschitz constants of the score functions. 
While the vast majority of past theory has been devoted to accommodating general distributions in $\mathbb{R}^d$, 
acceleration is shown to be possible if we restrict attention to discrete-valued distributions  \citep{chen2024convergence}. 
Another strand of recent works (e.g., \citet{chen2024accelerating,gupta2024faster})
explored how to exploit parallel sampling to achieve considerable speed-up.  
Theoretical guarantees have also recently been extended to cover other popular methods like consistency models \citep{song2023consistency,li2024towards,dou2024provable} and diffusion guidance \citep{ho2022classifier,wu2024theoretical,fu2024unveil}.




\paragraph{Score matching.} \citet{hyvarinen2005estimation} showed that the score function can be estimated via integration by parts, 
a result that was further extended in \citet{hyvarinen2007some}. \citet{song2020sliced} proposed sliced score matching to tame the computational complexity in high dimension. The consistency of the score matching estimator was studied in \citet{hyvarinen2005estimation}, with asymptotic normality established in \citet{forbes2015linear}. Optimizing the score matching loss has been shown to be intimately connected to minimizing upper bounds on the Kullback-Leibler divergence \citep{song2021maximum} and Wasserstein distance \citep{kwon2022score} between the generated distribution and the target data distribution. The recent work \citet{koehler2022statistical} studied the statistical efficiency of score matching by connecting it with the isoperimetric properties of the target data distribution. 
Furthermore, \citet{feng2024optimal} showed that statistical procedures based on score matching can achieve minimal asymptotic covariance for convex $M$-estimation.



\paragraph{Other theory for diffusion models.} 
The development of diffusion model theory is certainly beyond the above two strand of works. 
For instance,  
 \citet{oko2023diffusion} studied the approximation and generalization capabilities of diffusion modeling for distribution estimation; 
\citet{kadkhodaie2023generalization,zhang2023emergence,biroli2024dynamical} investigated the phase transition between the memorization regime and the generalization regime in diffusion models; 
\citet{chen2023score,qu2024diffusion} studied how diffusion models can adapt to  
low-dimensional structure. 
%
Moreover, \citet{ghimire2023geometry} adopted a geometric perspective and showed that the forward and backward processes of diffusion models are essentially Wasserstein gradient flows operating in the space of probability measures.
Recently, the idea of stochastic localization, which is closely related to diffusion models, is adopted to sample from posterior distributions \citep{montanari2023posterior,el2022sampling}, which has been implemented using the approximate message passing algorithm (\cite{donoho2009message,li2022non}); 
some results discovered in the stochastic localization literature (e.g., \citet{eldan2020taming}) have also paved the way to sharpening of dimension dependency \citep{benton2023linear}. 
In addition to the DDPM and DDIM type samplers discussed herein, 
convergence of other flow-based generative modeling has also been established in recent works (e.g., \citet{gao2024convergence,cheng2024convergence,xu2024normalizing}).
 \citet{xu2024provably} developed provably robust methods for  posterior sampling with diffusion priors for general nonlinear inverse problems, whereas \citet{montanari2024provably} exploited the idea of measure decomposition to improve posterior sampling for linear inverse problems. 
There have also been a couple of recent works that delve into various properties of diffusion models for Gaussian mixture models 
\citep{wu2024theoretical,chen2024learning,cui2023analysis,li2024critical}.




\section{Analysis}
\label{sec:analysis}

In this section, we describe our non-asymptotic proof strategies for establishing Theorem~\ref{thm:main-ODE}. 


\subsection{Preliminary facts}
\label{sec:preliminary-facts}

Before proceeding, 
we gather a couple of facts that will be useful for  the proof, with most proofs postponed to Appendix~\ref{sec:proof-preliminary}.

\paragraph{Properties related to the score function.}
First of all, in view of the alternative expression \eqref{eqn:training-score-equiv} for the score function and the property of the minimum mean square error (MMSE) estimator (e.g., \citet[Section~3.3.1]{hajek2015random}), we know that the true score function $s_t^{\star}$ is given by the conditional expectation
\begin{align}
s_{t}^{\star}(x) & =\mathbb{E}\left[-\frac{1}{\sqrt{1-\overline{\alpha_{t}}}}W\,\bigg|\,\sqrt{\overline{\alpha_{t}}}X_{0}+\sqrt{1-\overline{\alpha}_{t}}W=x\right]=\frac{1}{1-\overline{\alpha}_{t}}\mathbb{E}\left[\sqrt{\overline{\alpha_{t}}}X_{0}-x\,\big|\,\sqrt{\overline{\alpha_{t}}}X_{0}+\sqrt{1-\overline{\alpha}_{t}}W=x\right] \notag\\
	& = - \frac{1}{1-\overline{\alpha}_{t}} \underset{\eqqcolon \, g_t(x) }{\underbrace{ {\displaystyle \int}_{x_{0}}\big(x-\sqrt{\overline{\alpha_{t}}}x_{0}\big)p_{X_{0}\mid X_{t}}(x_{0}\,|\, x)\mathrm{d}x_{0} }}.
	\label{eq:st-MMSE-expression}
\end{align}
%
%
%
Let us also introduce the Jacobian matrix associated with $g_t(\cdot)$ as follows:
\begin{subequations}
	\label{eq:Jt-properties-summary}
\begin{equation}
J_{t}(x) \coloneqq 
	\frac{\partial g_t(x)}{\partial x}, \label{eq:Jacobian-Thm4}
\end{equation}
which can be equivalently rewritten as
\begin{align}
J_{t}(x) 
 & =I_{d}-\frac{1}{1-\overline{\alpha}_{t}}\mathsf{Cov}\Big(X_{t}-\sqrt{\overline{\alpha}_{t}}X_{0}\mid X_{t}=x\Big).
	\label{eq:Jt-x-expression-ij-23}
\end{align}
\end{subequations}

\paragraph{Properties about the learning rates.}
Next, we isolate a few useful properties about the learning rates as specified by $\{\alpha_t\}$ in \eqref{eqn:alpha-t}:  
\begin{subequations}
\label{eqn:properties-alpha-proof}
\begin{align}
	\alpha_t &\geq1-\frac{c_{1}\log T}{T}  \ge \frac{1}{2},\qquad\qquad\quad~~ 1\leq t\leq T \label{eqn:properties-alpha-proof-00}\\
	\frac{1}{2}\frac{1-\alpha_{t}}{1-\overline{\alpha}_{t}} \leq \frac{1}{2}\frac{1-\alpha_{t}}{\alpha_t-\overline{\alpha}_{t}}
	&\leq \frac{1-\alpha_{t}}{1-\overline{\alpha}_{t-1}}  \le \frac{4c_1\log T}{T},\qquad\quad~~ 2\leq t\leq T  \label{eqn:properties-alpha-proof-1}\\
	1&\leq\frac{1-\overline{\alpha}_{t}}{1-\overline{\alpha}_{t-1}} \leq1+\frac{4c_{1}\log T}{T} ,\qquad2\leq t\leq T  \label{eqn:properties-alpha-proof-3} \\
	\overline{\alpha}_{T} & \le \frac{1}{T^{c_2}}, \label{eqn:properties-alpha-proof-alphaT} \\
	\frac{\overline{\alpha}_{t+1}}{1-\overline{\alpha}_{t+1}} & \leq\frac{\overline{\alpha}_{t}}{1-\overline{\alpha}_{t}}\leq\frac{4\overline{\alpha}_{t+1}}{1-\overline{\alpha}_{t+1}}, \qquad\qquad~~ 1\leq t< T
	\label{eqn:properties-alpha-proof-alpha-ratio}
\end{align}
provided that $T$ is large enough. 
Here, $c_1$ is defined in \eqref{eqn:alpha-t}, and $c_2\geq 1000$ is some large numerical constant. 
In addition, if $\frac{d(1-\alpha_{t})}{\alpha_{t}-\overline{\alpha}_{t}}\ll 1$, then one has
\begin{align}
\Big(\frac{1-\overline{\alpha}_{t}}{\alpha_{t}-\overline{\alpha}_{t}}\Big)^{d/2} 
	& =1+\frac{d(1-\alpha_{t})}{2(\alpha_{t}-\overline{\alpha}_{t})}+\frac{d(d-2)(1-\alpha_{t})^{2}}{8(\alpha_{t}-\overline{\alpha}_{t})^{2}}+O\bigg(d^{3}\Big(\frac{1-\alpha_{t}}{\alpha_{t}-\overline{\alpha}_{t}}\Big)^{3}\bigg),\label{eq:expansion-ratio-1-alpha} \\
\Big(\frac{1-\overline{\alpha}_{t}}{\alpha_{t}-\overline{\alpha}_{t}}\Big)^{d/2}
	&=\exp\bigg(\frac{1-\alpha_{t}}{\alpha_{t}-\overline{\alpha}_{t}}\cdot\frac{d}{2}\bigg)\cdot\left(1+O\bigg(d\Big(\frac{1-\alpha_{t}}{\alpha_{t}-\overline{\alpha}_{t}}\Big)^{2}\bigg)\right).	
	\label{eq:expansion-ratio-3-alpha}
\end{align}
\end{subequations}
The proof of these properties is postponed to Appendix~\ref{sec:proof-properties-alpha}.

\paragraph{Properties of the forward process.}
Recall that the forward process satisfies $X_t \overset{\mathrm{d}}{=} \sqrt{\overline{\alpha}_t} X_0 + \sqrt{1-\overline{\alpha}_t} W$ with $W\sim \mathcal{N}(0,I_d)$.  
We have the following tail bound concerning the random vector $X_0$ conditional on $X_t$, 
whose proof can be found in Appendix~\ref{sec:proof-lem:x0}. Here and throughout,  we take
\begin{align}
\label{eqn:choice-y}
	\theta_t(x) \coloneqq  \max\bigg\{ -\frac{\log {p_{X_t}(x)}}{d\log T} , c_6 \bigg\}
\end{align}
for any $x \in \real^d$, where $c_6>0$ is  some large enough constant obeying $c_6 \geq 2c_R+c_0$. 
\begin{lems} \label{lem:x0}
Suppose that \eqref{eq:assumption-data-bounded} holds true. 
%
Then for any quantity $c_5 \ge 2$, conditioned on $X_t=y$ one has
\begin{align}
	\big\|\sqrt{\overline{\alpha}_{t}}X_0 - y \big\|_2 \leq 5c_5\sqrt{\theta_t(y) d(1-\overline{\alpha}_{t})\log T} 
	\label{eq:P-xt-X0-124}
\end{align} 
with probability at least $1 - \exp\big(-c_5^2\theta_t(y) d\log T \big)$.
In addition, it holds that
\begin{subequations}
\begin{align}
	\mathbb{E}\left[\big\| \sqrt{\overline{\alpha}_{t}}X_{0} - y \big\|_{2}\,\big|\,X_{t}=y\right] &\leq 12\sqrt{\theta_t(y) d(1-\overline{\alpha}_{t})\log T},\label{eq:E-xt-X0} \\
	\mathbb{E}\left[\big\| \sqrt{\overline{\alpha}_{t}}X_{0} - y \big\|^2_{2}\,\big|\,X_{t}=y\right] &\leq 120\theta_t(y) d(1-\overline{\alpha}_{t})\log T,\label{eq:E2-xt-X0} \\
	\mathbb{E}\left[\big\| \sqrt{\overline{\alpha}_{t}}X_{0} - y \big\|^3_{2}\,\big|\,X_{t}=y\right] &\leq 1040\big(\theta_t(y) d(1-\overline{\alpha}_{t})\log T\big)^{3/2},\label{eq:E3-xt-X0}\\
	\mathbb{E}\left[\big\| \sqrt{\overline{\alpha}_{t}}X_{0} - y \big\|^4_{2}\,\big|\,X_{t}=y\right] &\leq 10080\big(\theta_t(y) d(1-\overline{\alpha}_{t})\log T\big)^{2}.\label{eq:E4-xt-X0}
\end{align}
\end{subequations}
\end{lems}
\noindent
In order to interpret Lemma~\ref{lem:x0}, let us look at the case with $\theta_t(y)=c_6$, 
corresponding to the scenario where $p_{X_t}(y)\geq \exp(-c_6d\log T)$ (so that $p_{X_t}(y)$ is not exceedingly small). 
In this case, Lemma~\ref{lem:x0} implies that conditional on $X_t=y$ taking on a ``typical'' value,  
the vector $\sqrt{\overline{\alpha}_{t}}X_{0} - X_t  = \sqrt{1-\overline{\alpha}_t} \,\overline{W}_t$ (see \eqref{eqn:Xt-X0}) might still follow a sub-Gaussian tail, 
whose expected norm remains on the same order of that of an unconditional Gaussian vector $\mathcal{N}(0, (1-\overline{\alpha}_t)I_d)$.

\paragraph{Properties about the conditional covariance matrices.} 
We shall also single out two basic properties about certain conditional covariances as follows. 
To be precise, generate 
\begin{align}
	X_0 \sim \Pdata
	\qquad \text{and} \qquad
	Z\sim \mathcal{N}(0,I_d)
\end{align}
independently. Define, for any $\overline{\alpha}\in (0,1)$ and any $x\in \mathbb{R}^d$,   
the following conditional covariance matrix
\begin{align}
\Sigma_{\overline{\alpha}}(x) 
	\coloneqq \mathsf{Cov}\Big(Z\,|\,\sqrt{\overline{\alpha}}X_{0}+\sqrt{1-\overline{\alpha}}Z=x\Big).\label{eq:defn-Sigma-s}
\end{align}
The lemma below reveals two properties about $\Sigma_{\overline{\alpha}}(\cdot)$ that play a crucial role in our analysis; the proof is postponed to Appendix~\ref{sec:proof-lem-cond-covariance}. 
\begin{lems}
	\label{lem:cond-covariance}
	The conditional covariance matrix defined in \eqref{eq:defn-Sigma-s} satisfies the following properties.
	\begin{itemize}
		\item[(a)] For any $\overline{\alpha},\overline{\alpha}'\in (0,1)$ obeying 
			$\frac{|\overline{\alpha}' - \overline{\alpha}|}{\overline{\alpha}(1-\overline{\alpha})} \lesssim \frac{1}{d\log T}$ and $1-\overline{\alpha}\geq T^{-c_0}$ (with $c_0$ the constant defined in \eqref{eqn:alpha-t}), it holds that
	\begin{align*}
	\mathbb{E}\left[\Big(\Sigma_{\overline{\alpha}'}\big(\sqrt{\overline{\alpha}'}X_{0}+\sqrt{1-\overline{\alpha}'}Z\big)\Big)^{2}\right] & \preceq C_{3}^{2}\mathbb{E}\left[\Big(\Sigma_{\overline{\alpha}}\big(\sqrt{\overline{\alpha}}X_{0}+\sqrt{1-\overline{\alpha}}Z\big)\Big)^{2}\right]+C_{8}\exp\big(-C_{9}d\log T\big)I_{d}
	\end{align*}
	for some universal constants $C_3, C_8,C_9>0$.

		\item[(b)] For the learning rates \eqref{eqn:alpha-t}, one has
%
			\begin{align*}
\sum_{t=2}^{T}\frac{1-\alpha_{t}}{1-\overline{\alpha}_{t}}\mathsf{Tr}\bigg(\mathbb{E}\Big[\Big(\Sigma_{\overline{\alpha}_{t}}\big(\sqrt{\overline{\alpha}_{t}}X_{0}+\sqrt{1-\overline{\alpha}_{t}}Z\big)\Big)^{2}\Big]\bigg) & 
				\lesssim d\log T. 
\end{align*}
	\end{itemize}
\end{lems}

\begin{remark}
	This lemma, which plays a pivotal role in achieving linear $d$-dependency, is inspired by the analysis of \citet{benton2023linear} for DDPM,  
	exploiting an intriguing property (see \eqref{eq:time-differential-As-1}) originally discovered in the stochastic localization literature \citep{eldan2020taming}. 
	Note, however, that this property can also be established using elementary analysis without resorting to any sort of SDE toolboxes \citep{el2022information}.  
\end{remark}

\paragraph{Distance between $p_T$ and $q_T$.} 
We now record a simple result that demonstrates the proximity of $p_T$ and $q_T$, whose proof is provided in Appendix~\ref{sec:proof-lem-KL-T}. 
\begin{lems}
	\label{lem:KL-T}
	For any large enough $T$, it holds that
	\begin{align}
		\big( \mathsf{TV}(p_{X_{T}}\parallel p_{Y_{T}}) \big)^2 \leq \frac{1}{2}\mathsf{KL}(p_{X_{T}}\parallel p_{Y_{T}}) \lesssim \frac{1}{T^{200}}. 
	\end{align}
\end{lems}

\paragraph{Additional notation about score errors.}  
For any vector $x\in \mathbb{R}^d$ and any $1 < t\leq T$, let us define
%
\begin{align}
	\varepsilon_{\score, t}(x) \coloneqq \big\|s_t(x) - s_t^{\star}(x) \big\|_2
\qquad\text{and}\qquad
	\varepsilon_{\Jacobi, t}(x) \coloneqq \big\| J_{s_t}(x)  -  J_{s_t^{\star}} (x) \big\|,
	\label{eq:pointwise-epsilon-score-J}
\end{align} 
with $J_{s_t}$ and $J_{s_t^{\star}}$ the Jacobian matrices of $s_t(\cdot)$ and $s_t^{\star}(\cdot)$, respectively.  
Under Assumption~\ref{assumption:score-estimate}, 
we have
\begin{subequations}
	\label{eq:score-assumptions-equiv}
\begin{align}
\frac{1}{T}\sum_{t=1}^{T}\mathbb{E}_{X\sim q_{t}}\big[\varepsilon_{\score,t}(X)\big] & \leq\bigg(\frac{1}{T}\sum_{t=1}^{T}\mathbb{E}_{X\sim q_{t}}\left[\varepsilon_{\score,t}(X)^{2}\right]\bigg)^{1/2}\leq\varepsilon_{\score}.
\end{align}
Also, Assumption~\ref{assumption:score-estimate-Jacobi} says that
\begin{align}
	\frac{1}{T}\sum_{t=1}^{T}\mathbb{E}_{X\sim q_{t}}\big[\varepsilon_{\Jacobi,t}(X)\big] & \leq\varepsilon_{\Jacobi}.
\end{align}
\end{subequations}

\subsection{Main steps for the proof of Theorem~\ref{thm:main-ODE}}
\label{sec:pf-theorem-ode}



We now present the proof for our main result (i.e., Theorem~\ref{thm:main-ODE}) 
for the discrete-time sampler \eqref{eqn:ode-sampling} based on the probability flow ODE. 
Given that the TV distance is always bounded above by 1, it suffices to assume 
%
\begin{subequations}
	\label{eq:assumption-T-score-Jacob}
\begin{align}
	\varepsilon_{\mathsf{score}} & \leq\frac{1}{C_{1}\sqrt{d}\log^{2}T} \label{eq:assumption-T-score-Jacob-score}\\
	\varepsilon_{\mathsf{Jacobi}} & \leq\frac{1}{C_{1}d\log^{2}T} \label{eq:assumption-T-score-Jacob-Jacobi}
\end{align}
\end{subequations}
throughout the proof; otherwise the claimed result \eqref{eq:ratio-ODE} becomes trivial. 

\paragraph{Preparation.}
Before proceeding, we find it convenient to introduce a function 
\begin{subequations}
	\label{defn:phit-x}
\begin{align}
	\phi_t^{\star}(x) &= x + \frac{1-\alpha_{t}}{2}s_t^{\star}(x) 
	= x - \frac{1-\alpha_{t}}{2(1-\overline{\alpha}_{t})} \displaystyle \int_{x_{0}}\big(x-\sqrt{\overline{\alpha}_{t}}x_{0}\big)p_{X_{0}\mid X_{t}}(x_{0}\,|\, x)\mathrm{d}x_{0}, \\
	\phi_t(x) &= x + \frac{1-\alpha_{t}}{2}s_t(x),
\end{align}
\end{subequations}
where the first line follows from \eqref{eq:st-MMSE-expression}. 
The update rule \eqref{eqn:ode-sampling} can then be expressed as follows: 
\begin{equation}
	Y_{t-1} = \Phi_t(Y_t) = \frac{1}{\sqrt{\alpha_t}} \phi_t(Y_t). 
	\label{eq:Yt-phi-ODE}
\end{equation}
%
%

Moreover, for any point $y_T\in \mathbb{R}^d$ (resp.~$y_T'\in \mathbb{R}^d$), let us define the corresponding deterministic sequence
		\begin{equation}
			y_{t-1} = \frac{1}{\sqrt{\alpha_t}} \phi_t(y_t),
			\qquad 
			y_{t-1}' = \frac{1}{\sqrt{\alpha_t}} \phi_t(y_t'),
			\qquad t=T, T-1,\cdots 
			\label{eq:defn-yt-sequence-proof}
		\end{equation}
In other words, $\{y_{T-1},\ldots,y_1\}$ (resp.~$\{y_{T-1}',\ldots,y_1'\}$) is the (reverse-time) sequence generated by the probability flow ODE (cf.~\eqref{eq:Yt-phi-ODE}) when initialized to $Y_T=y_T$ (resp.~$Y_T=y_T'$). 
We also define the following quantities for any point $y_T\in \mathbb{R}^d$ and its associated sequence   $\{y_{T-1},\ldots,y_1\}$: 
		%
\begin{subequations}
	\label{eq:defn-xik-Stk-proof}
\begin{align}
	\xi_t(y_t) &\coloneqq \frac{\log T}{T}\big(d\varepsilon_{\Jacobi, t}(y_t) + \sqrt{d\log T}\varepsilon_{\score, t}(y_t)\big); \\ 
	S_{t}(y_T) &\coloneqq \sum_{1 < k \le t} \xi_k(y_k), \quad \text{for }t\geq 2,
	\qquad \text{ and } \qquad
	S_{1}(y_T) = 0. 
\end{align}
\end{subequations}
%
In words, for any given starting point $y_T$, 
$\xi_t(y_t)$ captures the (properly weighted) score error incurred in the $t$-th iteration, 
whereas $S_{t}(y_T)$ quantifies the aggregate weighted score error up to the $t$-th iteration.

With the above notation in place, we can readily proceed to our proof, which consists of several steps.

\paragraph{Step 1: bounding the density ratios of interest.} 
To begin with, we note that for any vectors $y_{t-1}$ and $y_t$, elementary properties about transformation of probability distributions give
\begin{align}
\frac{p_{Y_{t-1}}(y_{t-1})}{p_{X_{t-1}}(y_{t-1})} & =\frac{p_{\sqrt{\alpha_{t}}Y_{t-1}}(\sqrt{\alpha_{t}}y_{t-1})}{p_{\sqrt{\alpha_{t}}X_{t-1}}(\sqrt{\alpha_{t}}y_{t-1})}\notag\\
 & =\frac{p_{\sqrt{\alpha_{t}}Y_{t-1}}(\sqrt{\alpha_{t}}y_{t-1})}{p_{Y_{t}}(y_{t})}\cdot\bigg(\frac{p_{\sqrt{\alpha_{t}}X_{t-1}}(\sqrt{\alpha_{t}}y_{t-1})}{p_{X_{t}}(y_{t})}\bigg)^{-1}\cdot\frac{p_{Y_{t}}(y_{t})}{p_{X_{t}}(y_{t})},\label{eq:recursion}
\end{align}
thus converting the density ratio of interest into the product of three other density ratios. 
Noteworthily, this observation \eqref{eq:recursion} connects the target density ratio $\frac{p_{Y_{t-1}}}{p_{X_{t-1}}}$ at the $(t-1)$-th step with its counterpart $\frac{p_{Y_{t}}}{p_{X_{t}}}$ at the $t$-th step, 
motivating us to look at the density changes within adjacent steps in both the forward and the reverse processes (i.e., $p_{X_{t-1}}$ vs.~$p_{X_{t}}$ and $p_{Y_{t-1}}$ vs.~$p_{Y_{t}}$). 
In light of this expression, we develop a key lemma related to some of these density ratios.   
%
\begin{lems} 
\label{lem:main-ODE}
Recall the definition of $\theta_t(x)$ in \eqref{eqn:choice-y}. Consider any $x \in \real^d$ obeying
$
	\frac{40c_{1}\varepsilon_{\score,t}(x)\log^{\frac{3}{2}}T}{T}\leq\sqrt{\theta_t(x)d}.
$
%
Then one has
\begin{align} 
	\frac{p_{\sqrt{\alpha_{t}}X_{t-1}}\big(\phi_{t}(x)\big)}{p_{X_{t}}(x)}&\leq2\exp\bigg(\Big(5\varepsilon_{\score,t}(x)\sqrt{\theta_t(x)d\log T}+60\theta_t(x)d\log T\Big)\frac{1-\alpha_{t}}{\alpha_{t}-\overline{\alpha}_{t}}\bigg).
	\label{eq:xt_up}
\end{align}
If, in addition, we have
$
C_{10}\frac{\theta_t(x)d\log^{2}T+\varepsilon_{\score,t}(x)\sqrt{\theta_t(x)d\log^{3}T}}{T}\leq1
$
for some large enough constant $C_{10}>0$, then it holds that 
\begin{subequations}
\label{eq:ODE}
\begin{align} 
&\frac{p_{\sqrt{\alpha_{t}}X_{t-1}}(\phi_{t}(x))}{p_{X_{t}}(x)} \notag \\
&=1+\frac{d(1-\alpha_{t})}{2(\alpha_{t}-\overline{\alpha}_{t})}
+\frac{(1-\alpha_{t})\Big(\big\|\int\big(x-\sqrt{\overline{\alpha}_{t}}x_{0}\big)p_{X_{0}\mymid X_{t}}(x_{0}\mymid x)\mathrm{d}x_{0}\big\|_{2}^{2}-\int\big\| x-\sqrt{\overline{\alpha}_{t}}x_{0}\big\|_{2}^{2}p_{X_{0}\mymid X_{t}}(x_{0}\mymid x)\mathrm{d}x_{0}\Big)}{2(\alpha_{t}-\overline{\alpha}_{t})(1-\overline{\alpha}_{t})}
 \notag\\
 & \quad
+O\bigg(\theta_t(x)^2d^{2}\Big(\frac{1-\alpha_{t}}{\alpha_{t}-\overline{\alpha}_{t}}\Big)^{2}\log^{2}T + \varepsilon_{\score, t}(x)\sqrt{\theta_t(x) d\log T}\Big(\frac{1-\alpha_{t}}{\alpha_{t}-\overline{\alpha}_{t}}\Big)\bigg).
	\label{eq:xt}
\end{align}

Moreover, for any random vector $Y$, one has 
\begin{align} 
 & \frac{p_{\phi_{t}(Y)}(\phi_{t}(x))}{p_{Y}(x)} \notag \\
 &=1 + \frac{d(1-\alpha_{t})}{2(\alpha_{t}-\overline{\alpha}_{t})}+\frac{(1-\alpha_{t})\Big(\big\|\int\big(x-\sqrt{\overline{\alpha}_{t}}x_{0}\big)p_{X_{0}\mymid X_{t}}(x_{0}\mymid x)\mathrm{d}x_{0}\big\|_{2}^{2}-\int\big\| x-\sqrt{\overline{\alpha}_{t}}x_{0}\big\|_{2}^{2}p_{X_{0}\mymid X_{t}}(x_{0}\mymid x)\mathrm{d}x_{0}\Big)}{2(\alpha_{t}-\overline{\alpha}_{t})(1-\overline{\alpha}_{t})} \notag \\
&\quad+ O\bigg(\theta_t(x)^2d^{2}\Big(\frac{1-\alpha_{t}}{\alpha_{t}-\overline{\alpha}_{t}}\Big)^{2}\log^{2}T + \frac{d \log T \varepsilon_{\mathsf{Jacobi},t}(x)}{T}\bigg),
\label{eq:yt}
\end{align}
\end{subequations}
provided that 
$
	C_{11}\frac{d^{2}\log^{2}T+d\varepsilon_{\Jacobi,t}(x)\log T}{T}  \leq 1
$
 for some large enough constant $C_{11}>0$. 
\end{lems}
\begin{proof} The proof of this lemma is postponed to Appendix~\ref{sec:proof-lem:main-ODE}. \end{proof}

\begin{remark}
Combining Lemma~\ref{lem:main-ODE} with Lemma~\ref{lem:x0} and \eqref{eqn:properties-alpha-proof}   
gives: if $C_{10}\frac{\theta_t(x)d\log^{2}T+\varepsilon_{\score,t}(x)\sqrt{\theta_t(x)d\log^{3}T}}{T}\leq1$ 
and if $\theta_t(x)\lesssim 1$, then \eqref{eq:xt} tells us that
\begin{align}
	\log\frac{p_{\sqrt{\alpha_{t}}X_{t-1}}(\phi_{t}(x))}{p_{X_{t}}(x)}\leq\frac{4c_{1}d\log T}{T}
	+C_{10}\left\{ \frac{d^{2}\log^{4}T}{T^{2}}+\frac{\varepsilon_{\score,t}(x)\sqrt{d\log^{3}T}}{T}\right\} 
	\label{eq:crude-ratio-qt-1-qt}
\end{align}
under our sample size assumption~\eqref{eq:assumption-T-score-Jacob}, 
where $C_{10}>0$ is some large enough constant. Here, we have made use of the fact that the penultimate term in \eqref{eq:xt} is non-positive due to Jensen's inequality. 
\end{remark}

Informally, the result in \eqref{eq:ODE} already tells us that
\[
	\frac{p_{\phi_{t}(Y_{t})}(\phi_{t}(x))}{p_{Y_{t}}(x)} / 
	\frac{p_{\sqrt{\alpha_{t}}X_{t-1}} (\phi_{t}(x))}{p_{X_{t}}(x)} \approx 1
\]
for many points $x$ if we ignore the residual terms, 
which combined with \eqref{eq:recursion} shows that 
\[
	\frac{p_{Y_{t-1}}(y_{t-1})}{p_{X_{t-1}}(y_{t-1})} 
	\approx \frac{p_{Y_{t}}(y_{t})}{p_{X_{t}}(y_{t})}
 \]
for many points $y_t$. 
However, it is worth pointing out that: while Lemma~\ref{lem:main-ODE} already provides useful estimates for the density ratios of interest, these results alone are not sufficient to yield the desired $d$-dependency. 
For instance, the residual term in \eqref{eq:ODE} scales quadratically in $d$, 
thereby precluding one from obtaining linear $d$-dependency.

To further make improvements, 
we develop a more refined bound below when $\theta_t(x)\lesssim 1$, whose proof can be found in Appendix~\ref{sec:lem-refine}. 
\begin{lems} \label{lem:refine}
Recall the definition of $\theta_t(\cdot)$ in \eqref{eqn:choice-y}. 
There exists some function $\zeta_t(\cdot)$ such that: 
for any $x$ obeying $\theta_t(x)\lesssim 1$, $C_{10}\frac{\theta_t(x)d\log^{2}T+\varepsilon_{\score,t}(x)\sqrt{\theta_t(x)d\log^{3}T}}{T}\leq1$ and $C_{11}\frac{d\varepsilon_{\Jacobi,t}(x)\log T}{T}  \leq 1$ (with the constants $C_{10},C_{11}$ defined in Lemma~\ref{lem:main-ODE}), one has
%
%
%
\begin{align}
 & \frac{p_{\phi_{t}(Y_{t})}(\phi_{t}(x))}{p_{Y_{t}}(x)} / 
	\frac{p_{\sqrt{\alpha_{t}}X_{t-1}} (\phi_{t}(x))}{p_{X_{t}}(x)} \notag\\
 & \qquad=1+\zeta_{t}(x)+O\bigg(\bigg\|\frac{\partial\phi_{t}^{\star}(x)}{\partial x}-I\bigg\|_{\mathrm{F}}^{2}+\frac{\varepsilon_{\mathsf{score},t}(x)\sqrt{d\log^{3}T}}{T}+\frac{d \log T \varepsilon_{\mathsf{Jacobi},t}(x)}{T}+\frac{d\log^{3}T}{T^{2}}\bigg)
	\label{eq:ratio-Y-X-complex-refined}
\end{align}
with $\zeta_t(x) \le 0$. In addition, this function $\zeta_t(\cdot)$ satisfies
\begin{align}
\mathop{\mathbb{E}}_{X\sim q_{t}}\big[\big|\zeta_{t}(X)\big|\big]
	\lesssim
	\mathop{\mathbb{E}}_{X\sim q_{t}}\left[\Big\|\frac{\partial\phi_{t}^{\star}}{\partial x}(X)-I\Big\|_{\mathrm{F}}^{2}\right]+
	\frac{d\log^{3}T}{T^{2}},
\end{align}
provided that $T \gtrsim d^2\log^5T$. 
\end{lems}
In words, Lemma~\ref{lem:refine}~makes apparent that a key quantity to control when bounding the density ratios of interest is 
\begin{equation}
	\Big\|\frac{\partial\phi_{t}^{\star}}{\partial x}(X)-I\Big\|_{\mathrm{F}}^{2}
	\label{eq:key-quantity-fro-norm}
\end{equation}
While we are unable to obtain the desired control of \eqref{eq:key-quantity-fro-norm} in a pointwise manner, 
the expected sum of this quantity \eqref{eq:key-quantity-fro-norm} over all $t$ can be bounded in a fairly tight manner (we shall demonstrate this momentarily in \eqref{eq:sum-fro-phit-norm-UB}), 
which forms a crucial step towards sharpening the dimension dependency.

\paragraph{Step 2: decomposing the TV distance based on ``typical'' points.}  

To bound the TV distance of interest, it is helpful to isolate the following sets
%
\begin{align}
	\mathcal{E} &\coloneqq \Big\{y : q_{1}(y) > \max\big\{ p_{1}(y),\, \exp\big(- c_{6} d\log T \big) \big\} \Big\},
\end{align}
%
where $c_{6}>0$ is some large enough universal constant introduced in Lemma~\ref{lem:main-ODE}. 
In words, this set $\mathcal{E}$ contains all $y$ that can be viewed as ``typical'' values under the distribution~$q_1$ (meaning that
$q_1(y)$ is not exceedingly small), while at the same time obeying $q_1(y)>p_1(y)$.

In view of the basic properties about the TV distance, we can derive
\begin{align}
\mathsf{TV}\big(q_{1},p_{1}\big) &= \int_{y : q_{1}(y) > p_{1}(y)}\big(q_{1}(y) - p_{1}(y)\big)\mathrm{d} y \notag\\
	&= \int_{y \in \mathcal{E}}\big(q_{1}(y) - p_{1}(y)\big)\mathrm{d} y +  \int_{y:p_{1}(y)<q_{1}(y)\le\exp(-c_{6}d\log T)}\big(q_{1}(y)-p_{1}(y)\big)\mathrm{d}y . 
	\label{eqn:ode-tv-123}
\end{align}
In order to bound the second term on the right-hand side of \eqref{eqn:ode-tv-123},  
we make note of a basic fact: 
since $X_{t}\overset{\mathrm{(d)}}{=}\sqrt{\overline{\alpha}_{t}}X_0+\sqrt{1-\overline{\alpha}_{t}}W$
with $W\sim\mathcal{N}(0,I_{d})$ and $\mathbb{P}(\|X_{0}\|_{2}\leq T^{c_{R}})=1$,
it holds that
\begin{equation}
\mathbb{P}\left\{ \|X_{t}\|_{2}\geq T^{c_{R}+2}\right\} \leq\mathbb{P}\left\{ \|W\|_{2}\geq T^{2}\right\} < \exp\left(-c_{6}d\log T\right)
	\label{eq:Xt-2range-ODE}
\end{equation}
under our assumption \eqref{eq:assumption-T-score-Jacob} on $T$, thereby indicating that 
\begin{equation}
\int_{y:\|y\|_{2}\geq T^{c_{R}+2}}q_t(y)\mathrm{d}y < \exp\left(-c_{6}d\log T\right).
	\label{eq:y_norm-qy-UB}
\end{equation}
This basic fact in turn reveals that
\begin{align*}
\int_{y:p_{1}(y)<q_{1}(y)\le\exp(-c_{12}d\log T)}\big(q_{1}(y)-p_{1}(y)\big)\mathrm{d}y & \le
\int_{y:q_{1}(y)\le\exp(-c_{6}d\log T)}q_{1}(y)\mathrm{d}y \\
	& \leq\exp(-c_{6}d\log T)\int_{y:\|y\|_{2}\leq T^{c_{R}+2}}\mathrm{d}y+\exp\left(-c_{6}d\log T\right)\\
 & \leq\exp(-c_{6}d\log T)\big(2T^{c_{R}+2}\big)^{d}+\exp\left(-c_{6}d\log T\right)\\
 & \le\exp\big(-0.5c_{6}d\log T\big), 
\end{align*}
provided that $c_{6}\geq 4(c_R+2)$. Substitution into \eqref{eqn:ode-tv-123} then yields
\begin{align}
\mathsf{TV}\big(q_{1},p_{1}\big) 
	&\le \mathbb{E}_{Y_{1}\sim p_{1}}\bigg[\Big(\frac{q_{1}(Y_{1})}{p_{1}(Y_{1})}-1\Big)\ind\left\{ Y_{1}\in\mathcal{E}\right\} \bigg] + \exp\big(-c_{6}d\log T\big),
	\label{eqn:ode-tv-10}
\end{align}
with the proviso that $c_{6}\geq 4(c_R+2)$.

To proceed, let us isolate the following set   
%
\begin{align}
	\mathcal{I}_{1}\coloneqq\Big\{ y_T \mid S_{T}\big(y_{T}\big)\leq c_{14}\Big\}
	\label{eq:defn-I1-proof-ode}
\end{align}
for some small enough constant $c_{14}>0$. 
In words, $\mathcal{I}_{1}$ is composed of a set of points whose aggregated score error along the backward trajectory is well-controlled; 
in fact, these are points that exhibit ``typical'' behavior under the assumptions \eqref{eq:assumption-T-score-Jacob-score} and \eqref{eq:assumption-T-score-Jacob-Jacobi}.
As a result, we can decompose the first term of \eqref{eqn:ode-tv-10} into the influence of ``typical'' points and that of the remaining points as follows: 
\begin{align}
 & \mathop{\mathbb{E}}_{Y_{1}\sim p_{1}}\bigg[\Big(\frac{q_{1}(Y_{1})}{p_{1}(Y_{1})}-1\Big)\ind\left\{ Y_{1}\in\mathcal{E}\right\} \bigg]=\mathop{\mathbb{E}}_{Y_{T}\sim p_{T}}\bigg[\Big(\frac{q_{1}(Y_{1})}{p_{1}(Y_{1})}-1\Big)\ind\left\{ Y_{1}\in\mathcal{E}\right\} \bigg]\notag\\
 & \quad= \mathop{\mathbb{E}}_{Y_{T}\sim p_{T}}\bigg[\Big(\frac{q_{1}(Y_{1})}{p_{1}(Y_{1})}-1\Big)\ind\left\{ Y_{1}\in\mathcal{E}, Y_{T}\in \mathcal{I}_1\right\} \bigg]
	+ \mathop{\mathbb{E}}_{Y_{T}\sim p_{T}}\bigg[\frac{q_{1}(Y_{1})}{p_{1}(Y_{1})}\ind\left\{ Y_{1}\in\mathcal{E},Y_{T}\notin \mathcal{I}_1\right\} \bigg],
 	\label{eq:decompose-I1-I1c}
\end{align}
where the first identity holds since $Y_{1}$ is determined purely by $Y_{T}$ via deterministic update rules. 
%
The decomposition \eqref{eq:decompose-I1-I1c} leaves us with two terms to control, which we accomplish in the next two steps.

\paragraph{Step 3: controlling the first term on the right-hand side of \eqref{eq:decompose-I1-I1c}.}
This step analyzes the first term on the right-hand side of \eqref{eq:decompose-I1-I1c}. 
We would like to make the analysis in this step slightly more general than needed, given that it will be useful for the subsequent analysis as well. 

To begin with, let us introduce the following quantity: 
\begin{equation}
	\tau(y_T)
	\coloneqq
	\max\Big\{2\le t\le T+1:S_{t-1}\big(y_{T}\big)\leq c_{14}\Big\},\label{eq:defn-tao-i}
\end{equation}
meaning that the score errors exhibit ``typical'' behavior up to the $\big(\tau(y_T)-1\big)$-th iteration.   
As can be clearly seen from the definition \eqref{eq:defn-I1-proof-ode} of $\mathcal{I}_1$,  
\begin{equation}
	\tau(y_T) = T+1, \qquad \forall y_T \in \mathcal{I}_1.
	\label{eq:tau-T-I1}
\end{equation}
In the sequel, we first single out the following lemma, whose proof is deferred to Appendix~\ref{sec:proof-lem:q1-large-qk-large}. 
\begin{lems}
	\label{lem:q1-large-qk-large}
	Consider any $y_{T}$ and its associated sequence $\{y_{T-1},\cdots,y_1\}$ (see \eqref{eq:defn-yt-sequence-proof}).  
	If $-\log q_1(y_1)\leq c_{6}d\log T$,
	then one has
	\begin{align}
		-\log q_k(y_k)\leq 2c_{6}d\log T
		\label{eq:q_k_yk_UB}
	\end{align}
	for any $1\leq k<\tau(y_T)$ (cf.~\eqref{eq:defn-tao-i}), provided that $c_6\geq 3c_1$. 
\end{lems}

As a consequence of Lemma~\ref{lem:q1-large-qk-large}, 
we are able to control the density ratio $q_t/p_t$ up to the $\big(\tau(y_T)-1\big)$-th iteration, 
as stated in the following lemma. The proof can be found in Appendix~\ref{sec:proof-lem-density-ratio-tau}. 
\begin{lems}
	\label{lem:density-ratio-tau}
	Consider any $y_T$, along with the deterministic sequence $\{y_{T-1},\cdots,y_1\}$ (cf.~\eqref{eq:defn-yt-sequence-proof})), and set $\tau=\tau(y_T)$  (cf.~\eqref{eq:defn-tao-i}). Then one has 
\begin{subequations}
	\label{eq:pt-qt-equiv-ODE-St}
\begin{align}
	\frac{q_{1}(y_{1})}{p_{1}(y_{1})}  = &\left\{ 1+O\Bigg(\frac{d\log^{4}T}{T} + \sum_{t < \tau} \bigg(\zeta_t(y_t) + \Big\|\frac{\partial \phi^{\star}_t(y_t)}{\partial x} - I\Big\|_{\mathrm{F}}^2\bigg)+S_{\tau-1}(y_{\tau-1})\Bigg)\right\} \frac{q_{\tau-1}(y_{\tau-1})}{p_{\tau-1}(y_{\tau-1})},	
	\label{eq:pt-qt-equiv-ODE-St-taui} \\
	&\text{and}
	\qquad \frac{q_{k}(y_{k})}{2p_{k}(y_{k})} \leq \frac{q_{1}(y_{1})}{p_{1}(y_{1})} \leq 2 \frac{q_{k}(y_{k})}{p_{k}(y_{k})}, \qquad \forall k < \tau,
	\label{eq:pt-qt-equiv-ODE-St-k}
\end{align}
\end{subequations}
	where the function $\zeta_t(\cdot)$ is defined in Lemma~\ref{lem:refine}.
\end{lems}
Moreover,  according to the definition in~\eqref{defn:phit-x}, we can invoke
the properties~\eqref{eq:Jt-properties-summary} to obtain
\[
\frac{\partial\phi_{t}^{\star}}{\partial x}(x)-I_{d}=-\frac{1-\alpha_{t}}{2(1-\overline{\alpha}_{t})}J_{t}(x)=\frac{1-\alpha_{t}}{2(1-\overline{\alpha}_{t})}\mathsf{Cov}\bigg(\frac{X_{t}-\sqrt{\overline{\alpha}_{t}}X_{0}}{\sqrt{1-\overline{\alpha}_{t}}}\mymid X_{t}=x\bigg)-\frac{1-\alpha_{t}}{2(1-\overline{\alpha}_{t})}I_{d},
\]
which combined with Lemma~\ref{lem:cond-covariance}(b) and the property~\eqref{eqn:properties-alpha-proof-1} leads to 
\begin{align}
 & \sum_{t=2}^{T}\mathop{\mathbb{E}}_{X_{t}\sim q_{t}}\left[\Big\|\frac{\partial\phi_{t}^{\star}}{\partial x}(X_{t})-I\Big\|_{\mathrm{F}}^{2}\right]\leq\sum_{t=2}^{T}\mathop{\mathbb{E}}_{X_{t}\sim q_{t}}\left[\Big\|\frac{1-\alpha_{t}}{2(1-\overline{\alpha}_{t})}\mathsf{Cov}\bigg(\frac{X_{t}-\sqrt{\overline{\alpha}_{t}}X_{0}}{\sqrt{1-\overline{\alpha}_{t}}}\mymid X_{t}\bigg)\Big\|_{\mathrm{F}}^{2}\right]+\sum_{t=2}^{T}\Big\|\frac{1-\alpha_{t}}{2(1-\overline{\alpha}_{t})}I_{d}\Big\|_{\mathrm{F}}^{2}\notag\\
 & \quad=\sum_{t=2}^{T}\bigg(\frac{1-\alpha_{t}}{2(1-\overline{\alpha}_{t})}\bigg)^{2}\mathop{\mathbb{E}}_{X_{t}\sim q_{t}}\left[\mathsf{Tr}\Bigg(\bigg(\mathsf{Cov}\Big(\frac{X_{t}-\sqrt{\overline{\alpha}_{t}}X_{0}}{\sqrt{1-\overline{\alpha}_{t}}}\mymid X_{t}\Big)\bigg)^{2}\Bigg)\right]+\sum_{t=2}^{T}\Big\|\frac{1-\alpha_{t}}{2(1-\overline{\alpha}_{t})}I_{d}\Big\|_{\mathrm{F}}^{2}\notag\\
 & \quad\lesssim\frac{\log T}{T}\sum_{t=2}^{T}\frac{1-\alpha_{t}}{1-\overline{\alpha}_{t}}\mathsf{Tr}\bigg(\mathop{\mathbb{E}}_{X_{0}\sim\Pdata,Z\sim\mathcal{N}(0,I_{d})}\Big[\Big(\Sigma_{\overline{\alpha}_{t}}\big(\sqrt{\overline{\alpha}_{t}}X_{0}+\sqrt{1-\overline{\alpha}_{t}}Z\big)\Big)^{2}\Big]\bigg)+\sum_{t=2}^{T}\frac{d\log^{2}T}{T^{2}}\notag\\
 & \quad\asymp\frac{d\log^{2}T}{T}.	 
	\label{eq:sum-fro-phit-norm-UB}
\end{align}

Now let us look at the set $\mathcal{I}_1$. Taking $\tau(y_T)=T+1$ (cf.~\eqref{eq:tau-T-I1}) in Lemma~\ref{lem:density-ratio-tau} yields
\begin{align}
 &  \mathop{\mathbb{E}}_{Y_{T}\sim p_{T}}\bigg[\Big(\frac{q_{1}(Y_{1})}{p_{1}(Y_{1})}-1\Big)\ind\left\{ Y_{1}\in\mathcal{E},Y_{T}\in\mathcal{I}_1 \right\} \bigg]\nonumber\\
 & = \mathop{\mathbb{E}}_{Y_{T}\sim p_{T}}\left[\left(\left\{ 1+O\Bigg(\frac{d\log^{4}T}{T} + \sum_{t} \bigg(\zeta_t(y_t) + \Big\|\frac{\partial \phi^{\star}_t(y_t)}{\partial x} - I\Big\|_{\mathrm{F}}^2\bigg)+S_{T}(y_{T})\Bigg)\right\} \frac{q_{T}(Y_{T})}{p_{T}(Y_{T})}-1\right)\ind\left\{ Y_{1}\in\mathcal{E},Y_{T}\in\mathcal{I}_1 \right\} \right]\nonumber\\
 & = {\displaystyle \int}\left\{ \left(1+O\Bigg(\frac{d\log^{4}T}{T} + \sum_{t} \bigg(\zeta_t(y_t) + \Big\|\frac{\partial \phi^{\star}_t(y_t)}{\partial x} - I\Big\|_{\mathrm{F}}^2\bigg)+S_{T}(y_{T})\Bigg)\right)q_{T}(y_{T})-p_{T}(y_{T})\right\} \ind\left\{ y_{1}\in\mathcal{E},y_{T}\in\mathcal{I}_1 \right\} \mathrm{d}y_{T}\nonumber\\
 & \leq 
	{\displaystyle \int}\big|q_{T}(y_{T})-p_{T}(y_{T})\big|\mathrm{d}y_{T}
	+ O\left(\frac{d\log^{4}T}{T}+\sqrt{d\log^{3}T}\varepsilon_{\score}+(d\log T)\varepsilon_{\Jacobi}\right)\nonumber\\
 & \lesssim\frac{d\log^{4}T}{T}+\sqrt{d\log^{3}T}\varepsilon_{\score}+(d\log T)\varepsilon_{\Jacobi}. 
	\label{eq:I1-expectation-UB-ode}
\end{align}
Here, the last line holds since $\mathsf{TV}(p_T,q_T)\lesssim T^{-100}$ (according to Lemma~\ref{lem:main-ODE}), and the penultimate line follows from the observations below: 
\begin{align*}
 & {\displaystyle \int}\left(S_{T}(y_{T}) + \sum_{t} \bigg( |\zeta_t(y_t)| + \Big\|\frac{\partial \phi^{\star}_t(y_t)}{\partial x} - I\Big\|_{\mathrm{F}}^2\bigg)\right)q_{T}(y_{T})\ind\left\{ y_{1}\in\mathcal{E},y_{T}\in\mathcal{I}_1\right\} \mathrm{d}y_{T}\\
 & \quad=\sum_{t=1}^{T}{\displaystyle \int}\left(\frac{\log T}{T}\big(d\varepsilon_{\Jacobi,t}(y_{t})+\sqrt{d\log T}\varepsilon_{\score,t}(y_{t})\big) + |\zeta_t(y_t)| + \Big\|\frac{\partial \phi^{\star}_t(y_t)}{\partial x} - I\Big\|_{\mathrm{F}}^2\right)q_{T}(y_{T})\ind\left\{ y_{1}\in\mathcal{E},y_{T}\in\mathcal{I}_1\right\} \mathrm{d}y_{T}\\
	& \quad\leq 4\sum_{t=1}^{T} \mathop{\mathbb{E}}_{Y_{T}\sim p_{T}}\left[\left(\frac{\log T}{T}\big(d\varepsilon_{\Jacobi,t}(Y_{t})+\sqrt{d\log T}\varepsilon_{\score,t}(Y_{t})\big) + |\zeta_t(Y_t)| + \Big\|\frac{\partial \phi^{\star}_t(Y_t)}{\partial x} - I\Big\|_{\mathrm{F}}^2\right)\frac{q_{t}(Y_{t})}{p_{t}(Y_{t})}\right]\\
	& \quad=4\sum_{t=1}^{T} \mathop{\mathbb{E}}_{Y_{t}\sim q_{t}}\left[\frac{\log T}{T}\big(d\varepsilon_{\Jacobi,t}(Y_{t})+\sqrt{d\log T}\varepsilon_{\score,t}(Y_{t})\big) + |\zeta_t(Y_t)| + \Big\|\frac{\partial \phi^{\star}_t(Y_t)}{\partial x} - I\Big\|_{\mathrm{F}}^2\right]\\
 & \quad\lesssim \frac{d\log^{4}T}{T}+(d\log T)\varepsilon_{\Jacobi} + \sqrt{d\log^{3}T}\varepsilon_{\score},
\end{align*}
where the first inequality is due to \eqref{eq:pt-qt-equiv-ODE-St}, 
and the last relation comes from \eqref{eq:score-assumptions-equiv} and \eqref{eq:sum-fro-phit-norm-UB}.

%
%
%
%
%
%

\paragraph{Step 4: controlling the second term on the right-hand side of \eqref{eq:decompose-I1-I1c}.}
In this step,  we find it helpful to introduce the following sets (in addition to $\mathcal{I}_1$ defined in \eqref{eq:defn-I1-proof-ode}), 
where we again abbreviate $\tau=\tau(y_T)$ as long as it is clear from the context: 
\begin{subequations}
	\label{eq:defn-I2-I3-I4-ode}
\begin{align}
\mathcal{I}_{2} & \coloneqq\Big\{ y_T : c_{14}\leq S_{\tau}\big(y_{T}\big)\leq2c_{14}\Big\},
	\label{eq:defn-I2-I3-I4-ode-I2}\\
\mathcal{I}_{3} & \coloneqq\bigg\{ y_T : S_{\tau-1}\big(y_{T}\big)\leq c_{14},\xi_{\tau}\big(y_T\big)\geq c_{14},\frac{q_{\tau-1}(y_{\tau-1})}{p_{\tau-1}(y_{\tau-1})}\leq\frac{8q_{\tau}(y_{\tau})}{p_{\tau}(y_{\tau})}\bigg\},
	\label{eq:defn-I2-I3-I4-ode-I3}\\
	\mathcal{I}_{4} & \coloneqq\bigg\{ y_T : S_{\tau-1}\big(y_{T}\big)\leq c_{14},\xi_{\tau}\big(y_T\big)\geq c_{14},\frac{q_{\tau-1}(y_{\tau-1})}{p_{\tau-1}(y_{\tau-1})}>\frac{8q_{\tau}(y_{\tau})}{p_{\tau}(y_{\tau})}\bigg\}.
\label{eq:defn-I2-I3-I4-ode-I4}
\end{align}
\end{subequations}
%
It follows immediately from the definition that $\mathcal{I}_{1} \cup \mathcal{I}_{2} \cup \mathcal{I}_{3} \cup \mathcal{I}_{4} = \mathbb{R}^d.$
%
%
In words, 
for any point $y_T$ in $\mathcal{I}_2$, the resulting score error remains well-controlled in the $\tau$-th iteration; 
in comparison, the points in $\mathcal{I}_3$ and $\mathcal{I}_4$ might incur large score errors in the $\tau$-th iteration. 
The difference between $\mathcal{I}_3$ and $\mathcal{I}_4$ then lies in the comparison between the density ratios $q_t/p_t$ 
in the $(\tau-1)$-th and the $\tau$-th iteration. 

We shall tackle each of these sets separately, with the combined result summarized in the lemma below. 
\begin{lems}
	\label{lem:I2-I3-I4-bound}
	It holds that
	\begin{align}
		\mathop{\mathbb{E}}_{Y_{T}\sim p_{T}}\bigg[\frac{q_{1}(Y_{1})}{p_{1}(Y_{1})}\ind\left\{ Y_{1}\in\mathcal{E},Y_{T}\in\mathcal{I}_2\cup \mathcal{I}_3 \cup \mathcal{I}_4\right\} \bigg]
		&\lesssim \frac{d\log^{4}T}{T}+\sqrt{d\log^{3}T}\varepsilon_{\score}+(d\log T)\varepsilon_{\Jacobi}. 
		\label{eq:I2-I3-I4-bound}
	\end{align}
\end{lems}
\noindent See Appendix~\ref{sec:proof-lem:I2-I3-I4-bound} for the proof of this lemma.  

\paragraph{Step 5: putting all pieces together.}
Recall that $\mathcal{I}_1\cup \mathcal{I}_{2} \cup \mathcal{I}_{3} \cup \mathcal{I}_{4} = \mathbb{R}^d.$ 
Taking \eqref{eqn:ode-tv-10}, \eqref{eq:decompose-I1-I1c}, \eqref{eq:I1-expectation-UB-ode} and \eqref{eq:I2-I3-I4-bound} collectively, we conclude that
\begin{align*}
\mathsf{TV}(p_{1},q_{1}) & \leq \mathop{\mathbb{E}}_{Y_{T}\sim p_{T}}\bigg[\Big(\frac{q_{1}(Y_{1})}{p_{1}(Y_{1})}-1\Big)\ind\big\{ Y_{1}\in\mathcal{E},Y_{T}\in \mathcal{I}_1\big\} \bigg] \\
	&\qquad 
	+
	\mathop{\mathbb{E}}_{Y_{T}\sim p_{T}}\bigg[\frac{q_{1}(Y_{1})}{p_{1}(Y_{1})}\ind\left\{ Y_{1}\in\mathcal{E},Y_{T}\in\mathcal{I}_2\cup \mathcal{I}_3 \cup \mathcal{I}_4\right\} \bigg]+\exp(-c_{6}d\log T)\\
 & \lesssim\frac{d\log^{4}T}{T}+\sqrt{d\log^{3}T}\varepsilon_{\score}+d\varepsilon_{\Jacobi}\log T
\end{align*}
as claimed.

\section{Discussion}
\label{sec:discussion}

In this paper, we have developed a new suite of non-asymptotic theory for establishing the convergence and faithfulness of the probability flow ODE based sampler, 
assuming access to reliable estimates of the (Stein) score functions. 
Our analysis framework seeks to track the dynamics of the reverse process directly using elementary tools, 
which eliminates the need to look at the continuous-time limit and invoke the SDE and ODE toolboxes. 
Our result is the first to establish nearly linear dimension dependency for the iteration complexity of this sampler, 
where only very minimal assumptions on the target data distribution are imposed. 
The analysis framework laid out in the current paper might shed light on how to analyze other variants of score-based generative models as well.

Moving forward, there are plenty of questions that require in-depth theoretical understanding. 
For instance, 
can we establish sharp convergence results in terms of the Wasserstein distance for general non-strongly-log-concave data distributions, which 
could sometimes be ``closer'' to how humans differentiate pictures and might potentially help relax Assumption~\ref{assumption:score-estimate-Jacobi} in the case of deterministic samplers?  
To what extent can we further accelerate the sampling process, without requiring much more information than the score functions? 
Ideally, one would hope to achieve acceleration with the aid of the score functions only.    
It would also be of paramount interest to establish end-to-end performance guarantees that take into account both the score learning phase and the sampling phase.

\section*{Acknowledgements}

G.~Li is supported in part by the Chinese University of Hong Kong Direct Grant for Research. 
Y.~Wei is supported in part by the the NSF grants CAREER award DMS-2143215, CCF-2106778, CCF-2418156, and the Google Research Scholar Award. 
Y.~Chi  is supported in part by the grants ONR N00014-19-1-2404, NSF CCF-2106778, DMS-2134080 and ECCS-2126634. 
Y.~Chen is supported in part by the Alfred P.~Sloan Research Fellowship, the Google Research Scholar Award, the AFOSR grant FA9550-22-1-0198, the ONR grant N00014-22-1-2354,  and the NSF grants CCF-2221009 and CCF-1907661. 

\appendix


\section{Proof for several preliminary facts}
\label{sec:proof-preliminary}

\subsection{Proof of~properties \eqref{eq:Jt-x-expression-ij-23}}
Elementary calculations reveal that: the $(i,j)$-th entry of $J_{t}(x)$ is given by
\begin{align}
\big[J_{t}(x)\big]_{i,j} & =\ind\{i=j\} + \frac{1}{1-\overline{\alpha}_{t}}\bigg\{\Big(\int_{x_{0}}p_{X_{0}\mymid X_{t}}(x_{0}\mymid x)\big(x_{i}-\sqrt{\overline{\alpha}_{t}}x_{0,i}\big)\mathrm{d}x_{0}\Big)\Big(\int_{x_{0}}p_{X_{0}\mymid X_{t}}(x_{0}\mymid x)\big(x_{j}-\sqrt{\overline{\alpha}_{t}}x_{0,j}\big)\mathrm{d}x_{0}\Big)\notag\nonumber \\
 & \qquad\qquad-\int_{x_{0}}p_{X_{0}\mymid X_{t}}(x_{0}\mymid x)\big(x_{i}-\sqrt{\overline{\alpha}_{t}}x_{0,i}\big)\big(x_{j}-\sqrt{\overline{\alpha}_{t}}x_{0,j}\big)\mathrm{d}x_{0}\bigg\}.
 \label{eqn:derivative-2}
\end{align}
This immediately establishes the matrix expression \eqref{eq:Jt-x-expression-ij-23}. 

\subsection{Proof of~properties \eqref{eqn:properties-alpha-proof} regarding the learning rates}

\label{sec:proof-properties-alpha}

\paragraph{Proof of property~\eqref{eqn:properties-alpha-proof-00}.} 
From the choice of $\beta_t$ in \eqref{eqn:alpha-t}, we have
\[
\alpha_{t}=1-\beta_{t}\geq1-\frac{c_{1}\log T}{T}\geq\frac{1}{2},\qquad t\geq2.
\]
The case with $t=1$ holds trivially since $\beta_1=1/T^{c_0}$ for some large enough constant $c_0>0$.

\paragraph{Proof of properties~\eqref{eqn:properties-alpha-proof-1} and \eqref{eqn:properties-alpha-proof-3}.}
We start by proving \eqref{eqn:properties-alpha-proof-1}. 
Let $\tau$ be an integer obeying 
\begin{equation}
	\beta_{1}\bigg(1+\frac{c_{1}\log T}{T}\bigg)^{\tau} \le 1 < \beta_{1}\bigg(1+\frac{c_{1}\log T}{T}\bigg)^{\tau+1},
	\label{eq:defn-tau-proof-alpha}
\end{equation}
and we divide into two cases based on $\tau$. 
\begin{itemize}
	\item Consider any $t$ satisfying $t\leq \tau$.  
		In this case, it suffices to prove that
		\begin{align}
			1-\overline{\alpha}_{t-1} \ge \frac{1}{3}\beta_{1}\bigg(1+\frac{c_{1}\log T}{T}\bigg)^{t}.
			\label{eq:induction-proof-alpha}
		\end{align}
		Clearly, if \eqref{eq:induction-proof-alpha} is valid, then any $t\leq \tau$ obeys
		\[
			\frac{1-\alpha_{t}}{1-\overline{\alpha}_{t-1}}
			= \frac{\beta_{t}}{1-\overline{\alpha}_{t-1}}
			\leq\frac{\frac{c_{1}\log T}{T}\beta_{1}\big(1+\frac{c_{1}\log T}{T}\big)^{t}}{\frac{1}{3}\beta_{1}\big(1+\frac{c_{1}\log T}{T}\big)^{t}}=\frac{3c_{1}\log T}{T}
		\]
as claimed. Towards proving \eqref{eq:induction-proof-alpha}, 
first note that the base case with $t=2$ holds true trivially since $1-\overline{\alpha}_{1} =  1-\alpha_1 = \beta_{1} \geq \beta_1 \big(1+\frac{c_{1}\log T}{T}\big)^{2} /3$. 
		Next, let $t_0>2$ be {\em the first time} that Condition \eqref{eq:induction-proof-alpha} fails to hold and suppose that $t_0\leq \tau$.  
		It then follows that
		\begin{align}
			1-\overline{\alpha}_{t_{0}-2}=1-\frac{\overline{\alpha}_{t_{0}-1}}{\alpha_{t_{0}-1}}\le1-\overline{\alpha}_{t_{0}-1}
			< \frac{1}{3}\beta_{1}\bigg(1+\frac{c_{1}\log T}{T}\bigg)^{t_{0}}\leq\frac{1}{2}\beta_{1}\bigg(1+\frac{c_{1}\log T}{T}\bigg)^{t_{0}-1} 
			< \frac{1}{2}, 			
		\end{align}
		where the last inequality result from \eqref{eq:defn-tau-proof-alpha} and the assumption $t_0\leq \tau$. 
		This taken together with the assumptions \eqref{eq:induction-proof-alpha} and $t_0\leq \tau$ implies that 
\[
\frac{(1-\alpha_{t_{0}-1})\overline{\alpha}_{t_{0}-1}}{1-\overline{\alpha}_{t_{0}-2}}\geq\frac{\frac{c_{1}\log T}{T}\beta_{1}\min\Big\{\big(1+\frac{c_{1}\log T}{T}\big)^{t_{0}-1},1\Big\}\cdot\big(1-\frac{1}{2}\big)}{\frac{1}{2}\beta_{1}\big(1+\frac{c_{1}\log T}{T}\big)^{t_{0}-1}}=\frac{\frac{c_{1}\log T}{T}\beta_{1}\big(1+\frac{c_{1}\log T}{T}\big)^{t_{0}-1}}{\beta_{1}\big(1+\frac{c_{1}\log T}{T}\big)^{t_{0}-1}}=\frac{c_{1}\log T}{T}.
\]
		As a result, we can further derive	
		\begin{align*}
1-\overline{\alpha}_{t_{0}-1} & =1-\alpha_{t_{0}-1}\overline{\alpha}_{t_{0}-2}=1-\overline{\alpha}_{t_{0}-2}+(1-\alpha_{t_{0}-1})\overline{\alpha}_{t_{0}-2}\\
 & =\bigg(1+\frac{(1-\alpha_{t_{0}-1})\overline{\alpha}_{t_{0}-2}}{1-\overline{\alpha}_{t_{0}-2}}\bigg)(1-\overline{\alpha}_{t_{0}-2})\\
 & \ge\bigg(1+\frac{c_{1}\log T}{T}\bigg)(1-\overline{\alpha}_{t_{0}-2})\ge\bigg(1+\frac{c_{1}\log T}{T}\bigg)\cdot\bigg\{\frac{1}{3}\beta_{1}\bigg(1+\frac{c_{1}\log T}{T}\bigg)^{t_{0}-1}\bigg\}\\
 & = \frac{1}{3}\beta_{1}\bigg(1+\frac{c_{1}\log T}{T}\bigg)^{t_{0}},			
		\end{align*}
		where the penultimate line holds since \eqref{eq:induction-proof-alpha} is first violated at $t=t_0$;  
		this, however, contradicts with the definition of $t_0$. 
		Consequently, one must have $t_0>\tau$, meaning that \eqref{eq:induction-proof-alpha} holds for all $t \le \tau$. 

	\item We then turn attention to those $t$ obeying $t>\tau$. 
		In this case, it suffices to make the observation that
		\begin{equation}
			1-\overline{\alpha}_{t-1} \ge 1 - \overline{\alpha}_{\tau-1} \geq  \frac{1}{3}\beta_{1}\bigg(1+\frac{c_{1}\log T}{T}\bigg)^{\tau} 
			= \frac{ \frac{1}{3}\beta_{1}  \big(1+\frac{c_{1}\log T}{T}\big)^{\tau+1}  }{ 1+ \frac{c_{1}\log T}{T} } 
			\ge \frac{1}{4},
		\end{equation}
		where the second and the third inequalities come from \eqref{eq:induction-proof-alpha}. 
		Therefore, one obtains
		\[
			\frac{1-\alpha_{t}}{1-\overline{\alpha}_{t-1}}\leq\frac{\frac{c_{1}\log T}{T}}{1/4}\leq\frac{4c_{1}\log T}{T}.
		\]

\end{itemize}
The above arguments taken together establish property \eqref{eqn:properties-alpha-proof-1}.

In addition, it comes immediately from \eqref{eqn:properties-alpha-proof-1} that
\[
1\leq\frac{1-\overline{\alpha}_{t}}{1-\overline{\alpha}_{t-1}}=1+\frac{\overline{\alpha}_{t-1}-\overline{\alpha}_{t}}{1-\overline{\alpha}_{t-1}}=1+\frac{\overline{\alpha}_{t-1}(1-\alpha_{t})}{1-\overline{\alpha}_{t-1}}\leq1+\frac{4c_{1}\log T}{T},
\]
thereby justifying property \eqref{eqn:properties-alpha-proof-3}.



%
%

\paragraph{Proof of property~\eqref{eqn:properties-alpha-proof-alphaT}.}
Turning attention to the second claim \eqref{eqn:properties-alpha-proof-alphaT}, 
we note that for any $t$ obeying $t \ge \frac{T}{2} \gtrsim \frac{T}{\log T}$, 
one has
\[
1-\alpha_{t}=\frac{c_{1}\log T}{T}\min\bigg\{\beta_{1}\Big(1+\frac{c_{1}\log T}{T}\Big)^{t},\,1\bigg\}=\frac{c_{1}\log T}{T}.
\]
This in turn allows one to deduce that
\begin{align*}
\overline{\alpha}_{T} \le \prod_{t: t \ge T/2} \alpha_t \le \Big(1-\frac{c_{1}\log T}{T}\Big)^{T/2} \le \frac{1}{T^{c_2}}
\end{align*}
for an arbitrarily large constant $c_2>0$.

\paragraph{Proof of property~\eqref{eqn:properties-alpha-proof-alpha-ratio}.} 
It follows that
\[
\frac{\frac{\overline{\alpha}_{t}}{1-\overline{\alpha}_{t}}}{\frac{\overline{\alpha}_{t+1}}{1-\overline{\alpha}_{t+1}}}=\frac{\,\frac{1-\overline{\alpha}_{t+1}}{1-\overline{\alpha}_{t}}\,}{\alpha_{t+1}}\in[1,4],
\]
where the last inequality makes use of \eqref{eqn:properties-alpha-proof-00} and \eqref{eqn:properties-alpha-proof-3}.

\paragraph{Proof of property~\eqref{eq:expansion-ratio-1-alpha}.}
It is easily seen from the Taylor expansion that the learning rates $\{\alpha_t\}$ satisfy
\begin{align*}
\Big(\frac{1-\overline{\alpha}_{t}}{\alpha_{t}-\overline{\alpha}_{t}}\Big)^{d/2} & =\bigg(1+\frac{1-\alpha_{t}}{\alpha_{t}-\overline{\alpha}_{t}}\bigg)^{d/2}\nonumber \\
 & =1+\frac{d(1-\alpha_{t})}{2(\alpha_{t}-\overline{\alpha}_{t})}+\frac{d(d-2)(1-\alpha_{t})^{2}}{8(\alpha_{t}-\overline{\alpha}_{t})^{2}}+O\bigg(d^{3}\Big(\frac{1-\alpha_{t}}{\alpha_{t}-\overline{\alpha}_{t}}\Big)^{3}\bigg),
\end{align*}
provided that $\frac{d(1-\alpha_{t})}{\alpha_{t}-\overline{\alpha}_{t}}\lesssim 1$.

\paragraph{Proof of property~\eqref{eq:expansion-ratio-3-alpha}.} 
Finally, recognizing that
$$
	\frac{\exp(dx)-(1+x)^{d}}{\exp(dx)}=1-\bigg(\frac{1+x}{\exp(x)}\bigg)^{d}=1-\left(1-O(x^{2})\right)^{d}=O(dx^{2})
$$
for any $x$ obeying $|dx|<1/4$, one can deduce that
\[
\Big(\frac{1-\overline{\alpha}_{t}}{\alpha_{t}-\overline{\alpha}_{t}}\Big)^{d/2}=\bigg(1+\frac{1-\alpha_{t}}{\alpha_{t}-\overline{\alpha}_{t}}\bigg)^{d/2}=\exp\bigg(\frac{1-\alpha_{t}}{\alpha_{t}-\overline{\alpha}_{t}}\cdot\frac{d}{2}\bigg)\cdot\left(1+O\bigg(d\Big(\frac{1-\alpha_{t}}{\alpha_{t}-\overline{\alpha}_{t}}\Big)^{2}\bigg)\right),
\]
given the fact that $\frac{d(1-\alpha_{t})}{\alpha_{t}-\overline{\alpha}_{t}}\ll 1$.

\subsection{Proof of Lemma~\ref{lem:x0}}
\label{sec:proof-lem:x0}


For notational simplicity, we drop the subscript $t$ and denote $\theta(y):=\theta_t(y)$ throughout this subsection. To establish this lemma, we first make the following claim, whose proof is deferred to the end of this subsection.  
\begin{claim}
\label{eq:claim-123}
Consider any $c_5 \geq 2$ and suppose that $c_6\geq 2c_R$. There exists some $x_0 \in \real^d$ such that
\begin{subequations}
\label{eqn:lemma-x0-claim}
\begin{align}
\label{eqn:lemma-x0-claim-1}
	\|\sqrt{\overline{\alpha}_{t}}x_0 - y\|_2 &\leq c_5 \sqrt{\theta(y) d(1-\overline{\alpha}_{t})\log T} 
	\qquad\quad \text{and} \\
	\mathbb{P}\big(\|X_0 - x_0\|_2 \leq \epsilon\big) &\ge \Big( \frac{\epsilon}{T^{2\theta(y)}} \Big)^d 
	\qquad\quad \text{with} \quad \epsilon = \frac{1}{T^{c_0/2}}
	\label{eqn:lemma-x0-claim-2}
\end{align}
\end{subequations}
hold simultaneously, where $c_0$ is defined in \eqref{eqn:alpha-t}.
\end{claim}
With the above claim in place, we are ready to prove Lemma~\ref{lem:x0}. 
For notational simplicity, we let $X$ represent a random vector whose distribution $p_X(\cdot)$ obeys 
\begin{align}
	p_X(x) = p_{X_0\mid X_t}( x \mymid  y). 
	\label{eq:pX-properties}
\end{align}

Consider the point $x_0$ in Claim~\ref{eq:claim-123}, and let us look at a set:
\begin{align*}
\mathcal{E} \coloneqq \Big\{x : \sqrt{\overline{\alpha}_{t}}\|x - x_0\|_2 > 4c_5 \sqrt{\theta(y) d(1-\overline{\alpha}_{t})\log T}\Big\},
\end{align*}
where $c_5\geq 2$ (see Claim~\ref{eq:claim-123}). 
Combining this with property \eqref{eqn:lemma-x0-claim-1} about $x_0$ results in
\begin{align}
	\mathbb{P}\Big(\ltwo{\sqrt{\overline{\alpha}_{t}}X-y}>5 c_5\sqrt{\theta(y) d(1-\overline{\alpha}_{t})\log T}\Big) & \leq\mathbb{P}(X\in\mathcal{E}). \label{eq:UB-P-set-E}
\end{align}
Consequently, everything boils down to bounding $\mathbb{P}(X\in\mathcal{E})$. 
Towards this, we first invoke the Bayes rule $p_{X_0 \mymid X_t}(x \mymid y) \propto p_{X_0}(x)p_{X_t \mymid X_0}(y \mymid x) $ to derive
%
%
\begin{align}
\mathbb{P}(X_{0}\in\mathcal{E}\mymid X_{t}=y) & =\frac{\int_{x\in\mathcal{E}}p_{X_{0}}(x)p_{X_{t}\mymid X_{0}}(y\mymid x)\diff x}{\int_{x}p_{X_{0}}(x)p_{X_{t}\mymid X_{0}}(y\mymid x)\diff x} \notag\\
 & \le\frac{\int_{x\in\mathcal{E}}p_{X_{0}}(x)p_{X_{t}\mymid X_{0}}(y\mymid x)\diff x}{\int_{x:\|x-x_{0}\|_{2}\leq\epsilon}p_{X_{0}}(x)p_{X_{t}\mymid X_{0}}(y\mymid x)\diff x} \notag\\
 & \le\frac{\sup_{x\in\mathcal{E}}p_{X_{t}\mymid X_{0}}(y\mymid x)}{\inf_{x:\|x-x_{0}\|_{2}\leq\epsilon}p_{X_{t}\mymid X_{0}}(y\mymid x)}\cdot\frac{\mathbb{P}(X_{0}\in\mathcal{E})}{\mathbb{P}(\|X_{0}-x_{0}\|_{2}\leq\epsilon)}. 
	\label{eq:UB1-P-X0-Xt}
\end{align}
To further bound this quantity, 
note that: in view of the definition of $\mathcal{E}$ and 
expression~\eqref{eqn:lemma-x0-claim-1}, one has
\begin{align*}
	\sup_{x\in\mathcal{E}}p_{X_{t}\mymid X_{0}}(y\mymid x) & =\sup_{x:\|\sqrt{\overline{\alpha}_{t}}x-\sqrt{\overline{\alpha}_{t}}x_{0}\|_{2}>4c_5\sqrt{\theta(y) d(1-\overline{\alpha}_{t})\log T}}p_{X_{t}\mymid X_{0}}(y\mymid x)\\
 & \leq\sup_{x:\|\sqrt{\overline{\alpha}_{t}}x-y\|_{2}>3c_5\sqrt{\theta(y) d(1-\overline{\alpha}_{t})\log T}}p_{X_{t}\mymid X_{0}}(y\mymid x)\\
 & \leq\frac{1}{\big(2\pi(1-\overline{\alpha}_{t})\big)^{d/2}}\exp\bigg(-\frac{9c_5^2\theta(y) d\log T}{2}\bigg)
\end{align*}
and
\begin{align*}
\inf_{x:\|x-x_{0}\|_{2}\leq\epsilon}p_{X_{t}\mymid X_{0}}(y\mymid x) & \geq\frac{1}{\big(2\pi(1-\overline{\alpha}_{t})\big)^{d/2}}\inf_{x:\|x-x_{0}\|_{2}\leq\epsilon}\exp\bigg(-\frac{\ltwo{y-\sqrt{\overline{\alpha}_{t}}x}^{2}}{2(1-\overline{\alpha}_{t})}\bigg)\\
 & \geq\frac{1}{\big(2\pi(1-\overline{\alpha}_{t})\big)^{d/2}}\inf_{x:\|x-x_{0}\|_{2}\leq\epsilon}\exp\bigg(-\frac{\ltwo{y-\sqrt{\overline{\alpha}_{t}}x_{0}}^{2}}{1-\overline{\alpha}_{t}}-\frac{\ltwo{\sqrt{\overline{\alpha}_{t}}x-\sqrt{\overline{\alpha}_{t}}x_{0}}^{2}}{1-\overline{\alpha}_{t}}\bigg)\\
 & \geq\frac{1}{\big(2\pi(1-\overline{\alpha}_{t})\big)^{d/2}}\exp\bigg(-\frac{\ltwo{y-\sqrt{\overline{\alpha}_{t}}x_{0}}^{2}}{1-\overline{\alpha}_{t}}-\frac{\epsilon^{2}}{1-\overline{\alpha}_{t}}\bigg)\\
 & \geq\frac{1}{\big(2\pi(1-\overline{\alpha}_{t})\big)^{d/2}}\exp\bigg(-c_5^2\theta(y) d\log T-\frac{1}{T^{c_{0}}}\frac{1}{1-\overline{\alpha}_{t}}\bigg)\\
 & \geq\frac{1}{\big(2\pi(1-\overline{\alpha}_{t})\big)^{d/2}}\exp\big(-2c_5^2\theta(y) d\log T\big),
\end{align*}
where the second line is due to the elementary inequality $\|a+b\|_2^2 \leq 2\|a\|_2^2+2\|b\|_2^2$,  
the penultimate line relies on \eqref{eqn:lemma-x0-claim}, 
and the last line holds true since $1-\overline{\alpha}_{t}\geq1-\alpha_{1}=1/T^{c_{0}}$ (see \eqref{eqn:alpha-t}).  
Substitution of the above two displays into \eqref{eq:UB1-P-X0-Xt}, we arrive at
\begin{align}
	\mathbb{P}(X_{0}\in\mathcal{E}\mymid X_{t}=y) & \le \exp\big(-2.5c_5^2\theta(y) d\log T \big)\cdot\frac{1}{\mathbb{P}(\|X_{0}-x_{0}\|_{2}\leq\epsilon)} \notag\\
 & \le\exp\big(-2.5c_5^2\theta d\log T \big)\cdot\left(T^{2\theta(y) +c_{0}/2}\right)^{d} \notag\\
 & \le\exp\big(-(2.5c_5^2\theta(y) -2\theta(y)-c_{0}/2)d\log T\big), 
\end{align}
where the second inequality invokes \eqref{eqn:lemma-x0-claim-2}. 
Substituting this into \eqref{eq:UB-P-set-E} and recalling the distribution \eqref{eq:pX-properties} of $X$, 
we arrive at
\begin{align*}
	\mathbb{P}\Big(\ltwo{\sqrt{\overline{\alpha}_{t}}X-y}>5 c_5\sqrt{\theta(y)d(1-\overline{\alpha}_{t})\log T}\Big) 
	& \leq \exp\big(-(2.5c_5^2\theta(y)-2\theta(y)-c_{0}/2)d\log T\big)\\
	& \leq \exp\big(- c_5^2\theta(y)d\log T\big),
\end{align*}
with the proviso that $c_5\geq 2$ and $c_6\geq c_0$ (so that $\theta(y)\geq c_6\geq c_0$). 
This concludes the proof of the advertised result \eqref{eq:P-xt-X0-124} when $c_5 \ge 2$ and $c_6\geq 2c_R+c_0$, 
as long as Claim~\ref{eq:claim-123} can be justified.

With the above result in place, it then follows that
\begin{align*}
&\mathbb{E}\left[\big\| x_{t}-\sqrt{\overline{\alpha}_{t}}X_{0}\big\|_{2}\,\big|\,X_{t}=x_{t}\right] \\
	&\quad \leq5c_5\sqrt{\theta(y) d(1-\overline{\alpha}_{t})\log T}+\mathbb{E}\left[\big\| x_{t}-\sqrt{\overline{\alpha}_{t}}X_{0}\big\|_{2}\ind\big\{\|x_{t}-\sqrt{\overline{\alpha}_{t}}X_{0}\|_{2}\geq5c_5\sqrt{\theta(y) d(1-\overline{\alpha}_{t})\log T}\big\}\,\Big|\,X_{t}=x_{t}\right]\\
 &\quad \leq5c_5\sqrt{\theta(y) d(1-\overline{\alpha}_{t})\log T}+\int_{5c_5\sqrt{\theta(y) d(1-\overline{\alpha}_{t})\log T}}^{\infty}\mathbb{P}\big(\|x_{t}-\sqrt{\overline{\alpha}_{t}}x_{0}\|_{2}\geq\tau\mymid X_{t}=x_{t}\big)\mathrm{d}\tau\\
 &\quad \leq5c_5\sqrt{\theta(y) d(1-\overline{\alpha}_{t})\log T}+\int_{5c_5\sqrt{\theta(y) d(1-\overline{\alpha}_{t})\log T}}^{\infty}\exp\bigg(-\frac{\tau^{2}}{25(1-\overline{\alpha}_{t})}\bigg)\mathrm{d}\tau\\
 &\quad \leq5c_5\sqrt{\theta(y) d(1-\overline{\alpha}_{t})\log T}+\exp\big(-c_5^{2} \theta(y) d\log T\big)\\
 &\quad \leq6c_5\sqrt{\theta(y) d(1-\overline{\alpha}_{t})\log T},
\end{align*}
as claimed in \eqref{eq:E-xt-X0} by taking $c_5 = 2$. 
The proofs for \eqref{eq:E2-xt-X0}, \eqref{eq:E3-xt-X0} and \eqref{eq:E4-xt-X0} follow from similar aguments and are hence omitted for the sake of brevity.


\paragraph{Proof of Claim~\ref{eq:claim-123}.}
We prove this claim by contradiction. 
Specifically, suppose instead that: 
for every $x$ obeying $\|\sqrt{\overline{\alpha}_{t}}x - y\|_2 \leq c_5 \sqrt{\theta(y) d(1-\overline{\alpha}_{t})\log T },$ we have
\begin{equation}
	\mathbb{P}(\|X_0 - x\|_2 \le \epsilon) \leq \bigg( \frac{\epsilon}{2T^{\theta(y)}R} \bigg)^d
	\qquad \text{with } 
	\epsilon = \frac{1}{T^{c_0/2}}.
	\label{eq:contradition-x-y}
\end{equation}
Clearly, the choice of $\epsilon$ ensures that $\epsilon< \frac{1}{2}\sqrt{d(1-\overline{\alpha}_t)\log T}$.  
In the following, we would like to show that this assumption leads to contradiction.

First of all, let us look at $p_{X_{t}}$, which obeys
\begin{align}
\notag p_{X_{t}}(y) & =\int_{x}p_{X_{0}}(x)p_{X_{t}\mymid X_{0}}(y\mymid x)\mathrm{d}x\\
\notag & =\int_{x:\,\|\sqrt{\overline{\alpha}_{t}}x-y\|_{2}\geq c_5\sqrt{\theta(y) d(1-\overline{\alpha}_{t})\log T}}p_{X_{0}}(x)p_{X_{t}\mymid X_{0}}(y\mymid x)\mathrm{d}x\\
 & \qquad+\int_{x:\,\|\sqrt{\overline{\alpha}_{t}}x-y\|_{2}< c_5\sqrt{\theta(y) d(1-\overline{\alpha}_{t})\log T}}p_{X_{0}}(x)p_{X_{t}\mymid X_{0}}(y\mymid x)\mathrm{d}x.
%
	\label{eqn:contradiction-12}
\end{align} 
To further control \eqref{eqn:contradiction-12}, 
we make two observations:
\begin{itemize}
	\item[1)] The first term on the right-hand side of \eqref{eqn:contradiction-12} can be bounded by
\begin{align}
	& \notag\int_{x:\,\|\sqrt{\overline{\alpha}_{t}}x-y\|_{2}\geq c_5\sqrt{\theta(y)d(1-\overline{\alpha}_{t})\log T}}p_{X_{0}}(x)p_{X_{t}\mymid X_{0}}(y\mymid x)\mathrm{d}x\\
 & \qquad\le \sup_{z:\,\|z\|_{2}\geq c_5\sqrt{\theta(y) d(1-\overline{\alpha}_{t})\log T}}\frac{1}{\big(2\pi(1-\overline{\alpha}_{t})\big)^{d/2}}\exp\bigg(-\frac{\ltwo{z}^{2}}{2(1-\overline{\alpha}_{t})}\bigg) \notag\\
	& \qquad < \frac{1}{2} \exp\big(-\theta(y) d\log T\big),
\end{align}
provided that $c_5 \ge 2$ and $c_6 > 0$ is large enough (note that $\theta(y) \ge c_6$). 
Here, we have used $X_t\overset{\mathrm{(i)}}{=}\sqrt{\overline{\alpha}_t}X_0+\sqrt{1-\overline{\alpha}_t}W$ with $W\sim \mathcal{N}(0,I_d)$
as well as standard properties about Gaussian distributions.

	\item[2)] Regarding the second term on the right-hand side of \eqref{eqn:contradiction-12}, 
		let us construct an epsilon-net 
		$\mathcal{N}_{\epsilon} = \{z_i\}$ for the following set
		$$
			\big\{x : \|\sqrt{\overline{\alpha}_{t}}x - y\|_2 \leq c_5 \sqrt{\theta(y) d(1-\overline{\alpha}_{t})\log T} \text{ and } \ltwo{x} \leq R \big\},
		$$
		so that for each $x$ in this set, one can find a vector $z_i\in  \mathcal{N}_{\epsilon}$ such that $\|x - z_i\|_2\leq \epsilon$. 
		Clearly, we can choose $\mathcal{N}_{\epsilon}$ so that its cardinality obeys $|\mathcal{N}_{\epsilon}|\leq (2R/\epsilon)^d$.  
		Define $\mathcal{B}_i \defn \{x \mid \ltwo{x - z_i}\leq \epsilon\}$ for each $z_i\in \mathcal{N}_{\epsilon}$.  
		Armed with these sets, we can derive
\begin{align*}
	\int_{x:\|\sqrt{\overline{\alpha}_{t}}x-y\|_{2}<c_5\sqrt{\theta(y)d(1-\overline{\alpha}_{t})\log T}}p_{X_{0}}(x)p_{X_{t}\mymid X_{0}}(y\mymid x) \mathrm{d}x
	&\le \big(2\pi(1-\overline{\alpha}_{t})\big)^{-d/2}\sum_{i=1}^{|\mathcal{N}_{\epsilon}|}\mathbb{P}(X_{0}\in\mathcal{B}_{i}) \\
	& \le \big(2\pi(1-\overline{\alpha}_{t})\big)^{-d/2}\bigg(\frac{\epsilon}{2T^{2\theta(y)}R}\bigg)^{d}\bigg(\frac{ 2R}{\epsilon}\bigg)^{d} \\
	& < \frac{1}{2} \exp\big(-\theta(y) d\log T \big), 
\end{align*}
where the penultimate step comes from the assumption \eqref{eq:contradition-x-y}. 

\end{itemize}
The above results taken collectively lead to 
\begin{align}
	 p_{X_{t}}(y)  & < \exp\big(-\theta(y) d\log T\big),\label{eqn:contradiction}
\end{align} 
thus contradicting the definition 
of $\theta(y)$.

Consequently, we have proven the existence of $x$ obeying 
$\|\sqrt{\overline{\alpha}_{t}}x - y\|_2 \leq c_5 \sqrt{\theta(y) d(1-\overline{\alpha}_{t})\log T }$ and 
\begin{equation*}
	\mathbb{P}(\|X_0 - x\|_2 \le \epsilon) > \bigg( \frac{\epsilon}{2T^{\theta(y)}R} \bigg)^d 
	\geq \bigg( \frac{\epsilon}{T^{2\theta(y)}} \bigg)^d,
\end{equation*}
provided that $\theta(y)\geq c_6 \geq 2c_R$. 
This completes the proof of Claim~\ref{eq:claim-123}. 

\subsection{Proof of Lemma~\ref{lem:cond-covariance}}
\label{sec:proof-lem-cond-covariance}

\paragraph{Part (a).}


Before proceeding, we abuse the notation by introducing the following convenient notation: 
%
\[
X_{\overline{\alpha}}=\sqrt{\overline{\alpha}}X_{0}+\sqrt{1-\overline{\alpha}}Z\qquad\text{and}\qquad X_{\overline{\alpha}'}=\sqrt{\overline{\alpha}'}X_{0}+\sqrt{1-\overline{\alpha}'}Z,
\]
where we recall that $X_0\sim \Pdata$ and $Z\sim \mathcal{N}(0,I_d)$ are independently generated. 
Also, when $\frac{|\overline{\alpha}' - \overline{\alpha}|}{\overline{\alpha}(1-\overline{\alpha})} \lesssim \frac{1}{d\log T}$, 
we make note of several properties that can be easily verified:
\begin{subequations}
	\label{eq:alpha-alphaprime-properties}
\begin{align} 
	\alpha \asymp \alpha' \qquad \text{and} \qquad 1-\alpha \asymp 1-\alpha'; \\ 
\left(\frac{1-\overline{\alpha}'}{1-\overline{\alpha}}\right)^{d/2}=\left(1+\frac{\overline{\alpha}-\overline{\alpha}'}{1-\overline{\alpha}}\right)^{d/2}\lesssim \left(1+O\Big(\frac{1}{d\log T}\Big)\right)^{d/2} \lesssim 1,
	\quad\text{and}\quad \left(\frac{1-\overline{\alpha}}{1-\overline{\alpha}'}\right)^{d/2}\lesssim1; \\
	\frac{|\overline{\alpha}'-\overline{\alpha}|}{\overline{\alpha}(1-\overline{\alpha})(1-\overline{\alpha}')}=\frac{1}{1-\overline{\alpha}}O\left(\frac{1}{d\log T}\right). 
\end{align}
\end{subequations}

Consider any $x'$ and let $$x = \sqrt{\overline{\alpha}/\overline{\alpha}'}x'.$$ 
Our first step is to demonstrate a certain equivalence result between $p_{X_{\overline{\alpha}'}}(x')$ and $p_{X_{\overline{\alpha}}}(x)$. Towards this end, a little algebra reveals that
\begin{align}
\frac{\|x'-\sqrt{\overline{\alpha}'}x_{0}\|_{2}^{2}}{2(1-\overline{\alpha}')} & =\frac{\|x-\sqrt{\overline{\alpha}}x_{0}\|_{2}^{2}}{2(1-\overline{\alpha})}+\frac{\overline{\alpha}'-\overline{\alpha}}{\overline{\alpha}(1-\overline{\alpha}')(1-\overline{\alpha})}\frac{\|x-\sqrt{\overline{\alpha}}x_{0}\|_{2}^{2}}{2}, 
	\label{eq:x-xprime-connection}
\end{align}
and as a result,
\begin{align}
	&p_{X_{\overline{\alpha}'}}(x')  ={\displaystyle \int} \Pdata(x_{0})\frac{1}{\big(2\pi(1-\overline{\alpha}')\big)^{d/2}}\exp\bigg(-\frac{\|x'-\sqrt{\overline{\alpha}'}x_{0}\|_{2}^{2}}{2(1-\overline{\alpha}')}\bigg)\mathrm{d}x_{0} \notag\\
 &~~ =\left(\frac{1-\overline{\alpha}}{1-\overline{\alpha}'}\right)^{d/2}{\displaystyle \int}\Pdata(x_{0})\frac{1}{\big(2\pi(1-\overline{\alpha})\big)^{d/2}}\exp\bigg(-\frac{\|x-\sqrt{\overline{\alpha}}x_{0}\|_{2}^{2}}{2(1-\overline{\alpha})}-\frac{(\overline{\alpha}'-\overline{\alpha})\|x-\sqrt{\overline{\alpha}}x_{0}\|_{2}^{2}}{2\overline{\alpha}(1-\overline{\alpha})(1-\overline{\alpha}')}\bigg)\mathrm{d}x_{0}.
	\label{eq:pX-alpha-prime-display}
\end{align}
Combine this with the assumption $\frac{|\overline{\alpha}' - \overline{\alpha}|}{\overline{\alpha}(1-\overline{\alpha})} \lesssim \frac{1}{d\log T}$ and the properties \eqref{eq:alpha-alphaprime-properties} to yield 
\begin{align}
	& p_{X_{\overline{\alpha}'}}(x') \asymp{\displaystyle \int}\Pdata(x_{0})\frac{1}{\big(2\pi(1-\overline{\alpha})\big)^{d/2}}\exp\bigg(-\left(1+O\Big(\frac{1}{d\log T}\Big)\right)\frac{\|x-\sqrt{\overline{\alpha}}x_{0}\|_{2}^{2}}{2(1-\overline{\alpha})}\bigg)\mathrm{d}x_{0} 
	\notag\\
 &\quad =\left({\displaystyle \int}_{x_{0}\in\mathcal{E}}+{\displaystyle \int}_{x_{0}\notin\mathcal{E}}\right)\Pdata(x_{0})\frac{1}{\big(2\pi(1-\overline{\alpha})\big)^{d/2}}\exp\bigg(-\left(1+O\Big(\frac{1}{d\log T}\Big)\right)\frac{\|x-\sqrt{\overline{\alpha}}x_{0}\|_{2}^{2}}{2(1-\overline{\alpha})}\bigg)\mathrm{d}x_{0},
	\label{eq:pX-alpha-decompose-E}
\end{align}
where 
\[
\mathcal{E}\coloneqq\left\{ x_{0}\mid\frac{1}{\big(2\pi(1-\overline{\alpha})\big)^{d/2}}\exp\bigg(-\left(1+O\Big(\frac{1}{d\log T}\Big)\right)\frac{\|x-\sqrt{\overline{\alpha}}x_{0}\|_{2}^{2}}{2(1-\overline{\alpha})}\bigg)\geq\exp\left(-4c_6d\log T\right)\right\} 
\]
with the constant $c_6>0$ defined in Lemma~\ref{lem:x0}.
Given our assumption that $1-\overline{\alpha}\geq \frac{1}{T^{c_0}}$ and the fact $c_0\leq c_6$, 
a little algebra leads to
\begin{align*}
\frac{\|x-\sqrt{\overline{\alpha}}x_{0}\|_{2}^{2}}{2(1-\overline{\alpha})}&\leq 12c_{6}d\log T,
\qquad \forall x_0\in \mathcal{E}, 
\end{align*}
and as a consequence, 
\begin{align*}
 & {\displaystyle \int}_{x_{0}\in\mathcal{E}}\Pdata(x_{0})\frac{1}{\big(2\pi(1-\overline{\alpha})\big)^{d/2}}\exp\bigg(-\left(1+O\Big(\frac{1}{d\log T}\Big)\right)\frac{\|x-\sqrt{\overline{\alpha}}x_{0}\|_{2}^{2}}{2(1-\overline{\alpha})}\bigg)\mathrm{d}x_{0}\\
 & \qquad\asymp{\displaystyle \int}_{x_{0}\in\mathcal{E}}\Pdata(x_{0})\frac{1}{\big(2\pi(1-\overline{\alpha})\big)^{d/2}}\exp\bigg(-\frac{\|x-\sqrt{\overline{\alpha}}x_{0}\|_{2}^{2}}{2(1-\overline{\alpha})}\bigg)\mathrm{d}x_{0} \\
	& \qquad= p_{X_{\overline{\alpha}}}(x) - {\displaystyle \int}_{x_{0}\notin\mathcal{E}}\Pdata(x_{0})\frac{1}{\big(2\pi(1-\overline{\alpha})\big)^{d/2}}\exp\bigg(-\frac{\|x-\sqrt{\overline{\alpha}}x_{0}\|_{2}^{2}}{2(1-\overline{\alpha})}\bigg)\mathrm{d}x_{0}.
\end{align*}
Regarding those $x_{0}\notin\mathcal{E}$, one can easily derive
\begin{align*}
 & {\displaystyle \int}_{x_{0}\notin\mathcal{E}}\Pdata(x_{0})\frac{1}{\big(2\pi(1-\overline{\alpha})\big)^{d/2}}\exp\bigg(-\left(1+O\Big(\frac{1}{d\log T}\Big)\right)\frac{\|x-\sqrt{\overline{\alpha}}x_{0}\|_{2}^{2}}{2(1-\overline{\alpha})}\bigg)\mathrm{d}x_{0}\\
 & \qquad\leq\exp\left(-4c_{6}d\log T\right){\displaystyle \int}\Pdata(x_{0})\mathrm{d}x_{0}=\exp\left(-4c_{6}d\log T\right);
\end{align*}
similarly, it can also be easily verified that (which we omit here for conciseness)
\begin{align*}
 & {\displaystyle \int}_{x_{0}\notin\mathcal{E}}\Pdata(x_{0})\frac{1}{\big(2\pi(1-\overline{\alpha})\big)^{d/2}}\exp\bigg(-\frac{\|x-\sqrt{\overline{\alpha}}x_{0}\|_{2}^{2}}{2(1-\overline{\alpha})}\bigg)\mathrm{d}x_{0}\leq\exp\left(-1.5c_{6}d\log T\right).
\end{align*}
Combine the above results with \eqref{eq:pX-alpha-decompose-E} to deduce that
\begin{align}
p_{X_{\overline{\alpha}'}}(x') & \asymp{\displaystyle \int}\Pdata(x_{0})\frac{1}{\big(2\pi(1-\overline{\alpha})\big)^{d/2}}\exp\bigg(-\frac{\|x-\sqrt{\overline{\alpha}}x_{0}\|_{2}^{2}}{2(1-\overline{\alpha})}\bigg)\mathrm{d}x_{0}+O\big(\exp\left(-c_{6}d\log T\right)\big) \notag\\
 & =p_{X_{\overline{\alpha}}}(x)+O\big(\exp\left(-1.5c_{6}d\log T\right)\big) \label{eq:pXalpha-pXalphaprime-equiv-0} \\
	&\asymp p_{X_{\overline{\alpha}}}(x),
	\label{eq:pXalpha-pXalphaprime-equiv}
\end{align}
where the last line is valid provided that 
$-\log p_{X_{\overline{\alpha}}}(x)\leq c_{6}d\log T$.

Based on the above results, we can further demonstrate another equivalence result 
concerning $p_{X_{0}|X_{\overline{\alpha}'}}$ and $p_{X_{0}|X_{\overline{\alpha}}}$: 
if $-\log p_{X_{\overline{\alpha}}}(x)\leq c_{6}d\log T$ holds and $$\|x' - \sqrt{\overline{\alpha}'}x_0\|_2^2 \asymp \|x - \sqrt{\overline{\alpha}}x_0\|_2^2 \lesssim d(1-\overline{\alpha})\log T,$$ 
then one has
\begin{align}
p_{X_{0}|X_{\overline{\alpha}'}}(x_{0}\mymid x') & =\frac{\Pdata(x_{0})\frac{1}{(2\pi(1-\overline{\alpha}'))^{d/2}}\exp\big(-\frac{\|x'-\sqrt{\overline{\alpha}'}x_{0}\|_{2}^{2}}{2(1-\overline{\alpha}')}\big)}{p_{X_{\overline{\alpha}'}}(x')} \notag\\
 & \asymp\frac{\Pdata(x_{0})\frac{1}{(2\pi(1-\overline{\alpha}))^{d/2}}\exp\big(-\big(1+O\big(\frac{1}{d\log T}\big)\big)\frac{\|x-\sqrt{\overline{\alpha}}x_{0}\|_{2}^{2}}{2(1-\overline{\alpha})}\big)}{p_{X_{\overline{\alpha}}}(x)} \notag\\
 & \asymp\frac{\Pdata(x_{0})\frac{1}{(2\pi(1-\overline{\alpha}))^{d/2}}\exp\big(-\frac{\|x-\sqrt{\overline{\alpha}}x_{0}\|_{2}^{2}}{2(1-\overline{\alpha})}\big)}{p_{X_{\overline{\alpha}}}(x)}=p_{X_{0}|X_{\overline{\alpha}}}(x_{0}\mymid x).
	\label{eq:cond-prob-x0-xalpha-equiv}
\end{align}
%


%

Now, we are ready to analyze the conditional covariance matrices of interest. 
Recalling that 
\[
\big(x'-\sqrt{\overline{\alpha}'}x_{0}\big)\big(x'-\sqrt{\overline{\alpha}'}x_{0}\big)^{\top}=\frac{\overline{\alpha}'(1-\overline{\alpha})}{\overline{\alpha}(1-\overline{\alpha}')}\big(x-\sqrt{\overline{\alpha}}x_{0}\big)\big(x-\sqrt{\overline{\alpha}}x_{0}\big)^{\top},
\]
we can deduce that
\begin{align*}
\Sigma_{\overline{\alpha}'}(x') & =\mathsf{Cov}\left(Z\mymid\sqrt{\overline{\alpha}'}X_{0}+\sqrt{1-\overline{\alpha}'}Z=x'\right)=\mathsf{Cov}\Bigg(\frac{x'-\sqrt{\overline{\alpha}'}X_{0}}{\sqrt{1-\overline{\alpha}'}}\mymid\sqrt{\overline{\alpha}'}X_{0}+\sqrt{1-\overline{\alpha}'}Z=x'\Bigg)\\
 & =\frac{\overline{\alpha}'(1-\overline{\alpha})}{\overline{\alpha}(1-\overline{\alpha}')}\mathsf{Cov}\Bigg(\frac{x-\sqrt{\overline{\alpha}}X_{0}}{\sqrt{1-\overline{\alpha}}}\mymid\sqrt{\overline{\alpha}'}X_{0}+\sqrt{1-\overline{\alpha}'}Z=x'\Bigg)\\
 & \overset{\mathrm{(i)}}{\preceq}C_{0}\mathsf{Cov}\Bigg(\frac{x-\sqrt{\overline{\alpha}}X_{0}}{\sqrt{1-\overline{\alpha}}}\mymid\sqrt{\overline{\alpha}'}X_{0}+\sqrt{1-\overline{\alpha}'}Z=x'\Bigg)\\
 & \overset{(\mathrm{ii})}{=}C_{0}\inf_{\mu}\mathbb{E}\Bigg[\bigg(\frac{x-\sqrt{\overline{\alpha}}X_{0}}{\sqrt{1-\overline{\alpha}}}-\mu(x')\bigg)\bigg(\frac{x-\sqrt{\overline{\alpha}}X_{0}}{\sqrt{1-\overline{\alpha}}}-\mu(x')\bigg)^{\top}\bigg\}\mymid\sqrt{\overline{\alpha}'}X_{0}+\sqrt{1-\overline{\alpha}'}Z=x'\Bigg]\\
 & \overset{\mathrm{(iii)}}{\preceq}C_{0}\mathbb{E}\Bigg[\bigg(\frac{x-\sqrt{\overline{\alpha}}X_{0}}{\sqrt{1-\overline{\alpha}}}-\mu'(x)\bigg)\bigg(\frac{x-\sqrt{\overline{\alpha}}X_{0}}{\sqrt{1-\overline{\alpha}}}-\mu'(x)\bigg)^{\top}\ind\bigg\{\frac{\big\| x-\sqrt{\overline{\alpha}}X_{0}\big\|_{2}^{2}}{d\big(1-\overline{\alpha}\big)\log T}\lesssim1\bigg\}\mymid\sqrt{\overline{\alpha}'}X_{0}+\sqrt{1-\overline{\alpha}'}Z=x'\Bigg]\\
 & \qquad+C_{2}\exp\big(-C_{1}d\log T\big)I_{d}\\
 & \overset{\mathrm{(iv)}}{\preceq}C_{3}\underset{\eqqcolon\,\widetilde{\Sigma}_{\overline{\alpha}}(x)}{\underbrace{\mathbb{E}\Bigg[\bigg(\frac{x-\sqrt{\overline{\alpha}}X_{0}}{\sqrt{1-\overline{\alpha}}}-\mu'(x)\bigg)\bigg(\frac{x-\sqrt{\overline{\alpha}}X_{0}}{\sqrt{1-\overline{\alpha}}}-\mu'(x)\bigg)^{\top}\ind\bigg\{\frac{\big\| x-\sqrt{\overline{\alpha}}X_{0}\big\|_{2}^{2}}{d\big(1-\overline{\alpha}\big)\log T}\lesssim1\bigg\}\mymid\sqrt{\overline{\alpha}}X_{0}+\sqrt{1-\overline{\alpha}}Z=x\Bigg]}}\\
 & \qquad+C_{2}\exp\big(-C_{1}d\log T\big)I_{d}
%
\end{align*}
for some universal constants $C_0,C_1,C_2,C_3>0$, 
where $$\mu'(x) \coloneqq \mathbb{E}\left[Z\mid\sqrt{\overline{\alpha}}X_{0}+\sqrt{1-\overline{\alpha}}Z=x\right].$$ 
Here, (i) follows since $\alpha\asymp \alpha'$ and $1-\alpha\asymp 1-\alpha'$ (cf.~\eqref{eq:alpha-alphaprime-properties}); 
(ii) holds since 
$\mathsf{Cov}\big(\frac{x-\sqrt{\overline{\alpha}}X_{0}}{\sqrt{1-\overline{\alpha}}}\mymid\sqrt{\overline{\alpha}'}X_{0}+\sqrt{1-\overline{\alpha}'}Z=x'\big)$ represents the error covariance associated with the minimum mean square error (MMSE) estimator for $Z$ 
given $\sqrt{\overline{\alpha}'}X_{0}+\sqrt{1-\overline{\alpha}'}Z=x'$; 
(iii) arises from Lemma~\ref{lem:x0} (particularly the high-probability bound \eqref{eq:P-xt-X0-124} stating that 
the probability of the event $\frac{\| x-\sqrt{\overline{\alpha}}X_{0}\|_{2}^{2}}{d(1-\overline{\alpha})\log T}\gg 1$ is exponentially small); 
and (iv) is an immediate consequence of \eqref{eq:cond-prob-x0-xalpha-equiv}. 
In particular, the matrix $\widetilde{\Sigma}_{\overline{\alpha}}(x)$ defined in the step (iv) obeys
\begin{align*}
\widetilde{\Sigma}_{\overline{\alpha}}(x) & \preceq\mathsf{Cov}\Bigg(\frac{x-\sqrt{\overline{\alpha}}X_{0}}{\sqrt{1-\overline{\alpha}}}\mymid\sqrt{\overline{\alpha}}X_{0}+\sqrt{1-\overline{\alpha}}Z=x\Bigg)=\Sigma_{\overline{\alpha}}(x),\\
\big\|\widetilde{\Sigma}_{\overline{\alpha}}(x)\big\| & \leq\mathbb{E}\Bigg(\bigg\|\frac{x-\sqrt{\overline{\alpha}}X_{0}}{\sqrt{1-\overline{\alpha}}}\bigg\|_{2}^{2}\ind\bigg\{\frac{\big\| x-\sqrt{\overline{\alpha}}X_{0}\big\|_{2}^{2}}{d\big(1-\overline{\alpha}\big)\log T}\lesssim1\bigg\}\mymid\sqrt{\overline{\alpha}}X_{0}+\sqrt{1-\overline{\alpha}}Z=x\Bigg)\lesssim d\log T, 
\end{align*}
provided that $-\log p_{X_{\overline{\alpha}}}(x)\leq c_{6}d\log T$. 
These results in turn imply that
\begin{align}
\big(\Sigma_{\overline{\alpha}'}(x')\big)^{2} & \preceq\Big(C_{3}\widetilde{\Sigma}_{\overline{\alpha}}(x)+C_{2}\exp\big(-C_{1}d\log T\big)I_{d}\Big)^{2}\preceq C_{3}^{2}\big(\Sigma_{\overline{\alpha}}(x)\big)^{2}+C_{4}\exp\big(-C_{5}d\log T\big)I_{d}
	\label{eq:pointwise-Sigma-order}
\end{align}
for some universal constants $C_4,C_5>0$, 
as long as $-\log p_{X_{\overline{\alpha}}}(x)\leq c_{6}d\log T$.

Treating $x'$ as a random vector with the same distribution as $\sqrt{\overline{\alpha}'}X_{0}+\sqrt{1-\overline{\alpha}'}Z$ --- 
so that $x$ is a random vector with the same distribution as $\sqrt{\overline{\alpha}}X_{0}+\sqrt{1-\overline{\alpha}}Z$ --- 
and taking expectation over $x'$ (and hence $x$) on both sides of \eqref{eq:pointwise-Sigma-order}, 
we arrive at
\begin{align}
 & \mathbb{E}\left[\Big(\Sigma_{\overline{\alpha}'}\big(\sqrt{\overline{\alpha}'}X_{0}+\sqrt{1-\overline{\alpha}'}Z\big)\Big)^{2}\right]=\mathop{\mathbb{E}}_{x'\sim p_{X_{\overline{\alpha}'}}}\left[\big(\Sigma_{\overline{\alpha}'}(x')\big)^{2}\right]\notag\\
 & =\mathop{\mathbb{E}}_{x'\sim p_{X_{\overline{\alpha}'}}}\left[\big(\Sigma_{\overline{\alpha}'}(x')\big)^{2}\ind\big\{-\log p_{X_{\overline{\alpha}}}(x)\leq c_{6}d\log T\big\}\right]+\mathop{\mathbb{E}}_{x'\sim p_{X_{\overline{\alpha}'}}}\left[\big(\Sigma_{\overline{\alpha}'}(x')\big)^{2}\ind\big\{-\log p_{X_{\overline{\alpha}}}(x)>c_{6}d\log T\big\}\right]\notag\\
 & \preceq C_{3}^{2}\mathop{\mathbb{E}}_{x\sim p_{X_{\overline{\alpha}}}}\left[\big(\Sigma_{\overline{\alpha}}(x)\big)^{2}\right]+C_{4}\exp\big(-C_{5}d\log T\big)I_{d}+\mathop{\mathbb{E}}_{x'\sim p_{X_{\overline{\alpha}'}}}\left[\big(\Sigma_{\overline{\alpha}'}(x')\big)^{2}\ind\big\{-\log p_{X_{\overline{\alpha}}}(x)>c_{6}d\log T\big\}\right],
	\label{eq:exp-cond-cov-UB1}
\end{align}
where we use $p_{X_{\overline{\alpha}'}}$ and $p_{X_{\overline{\alpha}}}$ to denote the distribution of $X_{\overline{\alpha}'}$ and $X_{\overline{\alpha}}$, respectively. 
To bound the last term in the last line of \eqref{eq:exp-cond-cov-UB1}, 
note that for any $x$ obeying $-\log p_{X_{\overline{\alpha}}}(x)>c_{6}d\log T$,
it follows from (\ref{eq:pXalpha-pXalphaprime-equiv-0}) that
\begin{align*}
p_{X_{\overline{\alpha}'}}(x') & \asymp p_{X_{\overline{\alpha}}}(x)+O\big(\exp(-1.5c_{6}d\log T)\big)=o\big(\exp(-c_{6}d\log T)\big),
\end{align*}
and hence
\begin{align*}
 & \mathop{\mathbb{E}}_{x'\sim p_{X_{\overline{\alpha}'}}}\left[\big(\Sigma_{\overline{\alpha}'}(x')\big)^{2}\ind\bigg\{\frac{-\log p_{X_{\overline{\alpha}}}(x)}{d\log T}>c_{6}\bigg\}\right]\preceq\mathop{\mathbb{E}}_{x'\sim p_{X_{\overline{\alpha}'}}}\left[\big\|\Sigma_{\overline{\alpha}'}(x')\big\|^{2}\ind\bigg\{\frac{-\log p_{X_{\overline{\alpha}'}}(x')}{d\log T}\geq c_{6}\bigg\}\right]I_{d}.
\end{align*}
Defining
\[
\theta(x')=\max\bigg\{\frac{-\log p_{X_{\overline{\alpha}'}}(x')}{d\log T},c_{6}\bigg\},
\]
we can invoke Lemma~\ref{lem:x0} with a little algebra to derive (details are omitted for brevity)
\begin{align*}
\mathop{\mathbb{E}}_{x'\sim p_{X_{\overline{\alpha}'}}}\left[\big\|\Sigma_{\overline{\alpha}'}(x')\big\|^{2}\ind\bigg\{\frac{-\log p_{X_{\overline{\alpha}'}}(x')}{d\log T}>c_{0}\bigg\}\right] & \leq\sum_{k=1}^{\infty}\mathop{\mathbb{E}}_{x'\sim p_{X_{\overline{\alpha}'}}}\left[\big\|\Sigma_{\overline{\alpha}'}(x')\big\|^{2}\ind\bigg\{2^{k-1}c_{6}\leq\theta(x')\leq2^{k}c_{6}\bigg\}\right]\\
 & \leq C_{6}\exp\left(-C_{7}d\log T\right)
\end{align*}
for some universal constants $C_6,C_7>0$, 
where we have made use of the basic fact that
\[
\big\|\Sigma_{\overline{\alpha}'}(x')\big\|\leq\mathbb{E}\left[\|ZZ^{\top}\|\mymid\sqrt{\overline{\alpha}'}X_{0}+\sqrt{1-\overline{\alpha}'}Z=x'\right]=\mathbb{E}\left[\bigg\|\frac{x'-\sqrt{\overline{\alpha}'}X_{0}}{\sqrt{1-\overline{\alpha}'}}\bigg\|_{2}^{2}\mymid\sqrt{\overline{\alpha}'}X_{0}+\sqrt{1-\overline{\alpha}'}Z=x'\right].
\]
%

%

Putting the preceding resutls together, we can conclude that
\begin{align*}
\mathbb{E}\left[\Big(\Sigma_{\overline{\alpha}'}\big(\sqrt{\overline{\alpha}'}X_{0}+\sqrt{1-\overline{\alpha}'}Z\big)\Big)^{2}\right] & \preceq C_{3}^{2}\mathop{\mathbb{E}}_{x\sim p_{X_{\overline{\alpha}}}}\left[\big(\Sigma_{\overline{\alpha}}(x)\big)^{2}\right]+\left\{ C_{4}\exp\big(-C_{5}d\log T\big)+C_{6}\exp\left(-C_{7}d\log T\right)\right\} I_{d}\\
 & \preceq C_{3}^{2}\mathbb{E}\left[\Big(\Sigma_{\overline{\alpha}}\big(\sqrt{\overline{\alpha}}X_{0}+\sqrt{1-\overline{\alpha}}Z\big)\Big)^{2}\right]+C_{8}\exp\big(-C_{9}d\log T\big)I_{d}
\end{align*}
for some universal constants $C_8,C_9>0$, as claimed.

%
%
%
%

\paragraph{Part (b).}

First, we find it convenient to introduce another conditional covariance, defined as follows: 
%
%
\begin{align}
A_{s}(x) & \coloneqq\mathsf{Cov}\left(X_{0}\,|\,sX_{0}+\sqrt{s}Z=x\right),
\label{eq:defn-As-x}
\end{align}
which clearly satisfies
\begin{align}
\mathsf{Cov}\left(Z\,|\,sX_{0}+\sqrt{s}Z=x\right) & =\mathsf{Cov}\left(\frac{1}{\sqrt{s}}x-\sqrt{s}X_{0}\,|\,sX_{0}+\sqrt{s}Z=x\right)=sA_{s}(x).\label{eq:consequence-As-x-z}
\end{align}
It is easily seen that (by taking $s=\frac{\overline{\alpha}}{1-\overline{\alpha}}$)
\begin{align}
\Sigma_{\overline{\alpha}}(x) & =\mathsf{Cov}\Big(Z\,|\,\sqrt{\overline{\alpha}}X_{0}+\sqrt{1-\overline{\alpha}}Z=x\Big)
	= \frac{\overline{\alpha}}{1-\overline{\alpha}}A_{\frac{\overline{\alpha}}{1-\overline{\alpha}}}\Big(\frac{\sqrt{\overline{\alpha}}}{1-\overline{\alpha}}x\Big).\label{eq:defn-Sigma-s-2}
\end{align}

Let us single out a basic property about $A_s$ and $\Sigma_{\overline{\alpha}}$ that plays 
 an important role in the subsequent proof.  
First of all, it has been shown in previous work (see, e.g., \cite{eldan2020taming,el2022information}) that the time-differential of (\ref{eq:defn-As-x}) admits a simple form\footnote{While this result was originally established by \cite{eldan2020taming}
using stochastic localization, it can also be derived using an elementary
estimation-theoretic approach without introducing any SDEs (see  \cite{el2022information}).}
\begin{equation}
\mathrm{d}\mathbb{E}\big[A_{s}(sX_{0}+\sqrt{s}Z)\big]=-\mathbb{E}\big[\big(A_{s}(sX_{0}+\sqrt{s}Z)\big)^{2}\big]\mathrm{d}s.\label{eq:time-differential-As-1}
\end{equation}
Replacing $s$ with $\frac{\overline{\alpha}}{1-\overline{\alpha}}$
and using (\ref{eq:defn-Sigma-s}), we have
\[
A_{s}\big(sX_{0}+\sqrt{s}Z\big)=\frac{1-\overline{\alpha}}{\overline{\alpha}}\Sigma_{\overline{\alpha}}\big(\sqrt{\overline{\alpha}}X_{0}+\sqrt{1-\overline{\alpha}}Z\big),
\]
and hence (\ref{eq:time-differential-As-1}) immediately tells us that
\begin{equation}
\mathrm{d}\bigg(\frac{1-\overline{\alpha}}{\overline{\alpha}}\mathbb{E}\Big[\Sigma_{\overline{\alpha}}\big(\sqrt{\overline{\alpha}}X_{0}+\sqrt{1-\overline{\alpha}}Z\big)\Big]\bigg)=-\frac{(1-\overline{\alpha})^{2}}{\overline{\alpha}^{2}}\mathbb{E}\Big[\Big(\Sigma_{\overline{\alpha}}\big(\sqrt{\overline{\alpha}}X_{0}+\sqrt{1-\overline{\alpha}}Z\big)\Big)^{2}\Big]\mathrm{d}\Big(\frac{\overline{\alpha}}{1-\overline{\alpha}}\Big).\label{eq:SL-cov}
\end{equation}
%

From now on, let us consider any $0<\overline{\alpha}_{\mathsf{l}}<\overline{\alpha}_{\mathsf{u}}<1$
obeying $\frac{\overline{\alpha}_{\mathsf{l}}}{1-\overline{\alpha}_{\mathsf{l}}}\leq\frac{\overline{\alpha}_{\mathsf{u}}}{1-\overline{\alpha}_{\mathsf{u}}}\leq\frac{4\overline{\alpha}_{\mathsf{l}}}{1-\overline{\alpha}_{\mathsf{l}}}$, 
and the monotonicity of $f(x)=\frac{x}{1-x}$ in $x$ gives
\[
\frac{\overline{\alpha}_{\mathsf{l}}}{1-\overline{\alpha}_{\mathsf{l}}}\leq\frac{\overline{\alpha}}{1-\overline{\alpha}}\leq\frac{4\overline{\alpha}_{\mathsf{l}}}{1-\overline{\alpha}_{\mathsf{l}}}\qquad\text{for any }\overline{\alpha}\in[\overline{\alpha}_{\mathsf{l}},\overline{\alpha}_{\mathsf{u}}].
\]
Use the positive semidefiniteness of the covariance matrix and the
fact $\mathrm{d}\big(\frac{\overline{\alpha}}{1-\overline{\alpha}}\big)=\frac{\mathrm{d}\overline{\alpha}}{(1-\overline{\alpha})^{2}}$
to derive
\begin{align}
 & 
	{\displaystyle \int}_{\overline{\alpha}_{\mathsf{l}}}^{\overline{\alpha}_{\mathsf{u}}}\frac{(1-\overline{\alpha})^{2}}{\overline{\alpha}^{2}}\mathsf{Tr}\bigg(\mathbb{E}\Big[\Big(\Sigma_{\overline{\alpha}}\big(\sqrt{\overline{\alpha}}X_{0}+\sqrt{1-\overline{\alpha}}Z\big)\Big)^{2}\Big]\bigg)\frac{1}{(1-\overline{\alpha})^{2}}\mathrm{d}\overline{\alpha}\notag\nonumber \\
 & \qquad=-\int_{\overline{\alpha}_{\mathsf{l}}}^{\overline{\alpha}_{\mathsf{u}}}\frac{\mathrm{d}\Big(\frac{1-\overline{\alpha}}{\overline{\alpha}}\mathsf{Tr}\big(\mathbb{E}\big[\Sigma_{\overline{\alpha}}\big(\sqrt{\overline{\alpha}}X_{0}+\sqrt{1-\overline{\alpha}}Z\big)\big]\big)\Big)}{\mathrm{d}\overline{\alpha}}\mathrm{d}\overline{\alpha}\notag\nonumber \\
 & \qquad=\frac{1-\overline{\alpha}_{\mathsf{l}}}{\overline{\alpha}_{\mathsf{l}}}\mathsf{Tr}\Big(\mathbb{E}\big[\Sigma_{\overline{\alpha}_{\mathsf{l}}}\big(\sqrt{\overline{\alpha}_{\mathsf{l}}}X_{0}+\sqrt{1-\overline{\alpha}_{\mathsf{l}}}Z\big)\big]\Big)-\frac{1-\overline{\alpha}_{\mathsf{u}}}{\overline{\alpha}_{\mathsf{u}}}\mathsf{Tr}\Big(\mathbb{E}\big[\Sigma_{\overline{\alpha}_{\mathsf{u}}}\big(\sqrt{\overline{\alpha}_{\mathsf{u}}}X_{0}+\sqrt{1-\overline{\alpha}_{\mathsf{u}}}Z\big)\big]\Big)\geq0,\label{eq:int-Tr-cov-inequality-1}
\end{align}
where the penultimate line arises from (\ref{eq:SL-cov}). 

Moreover, recalling that $\overline{\alpha}_{t}=\prod_{i=1}^{t}\alpha_{i}\geq\overline{\alpha}_{t+1}$, we have 
\begin{align}
 \notag & \left(\frac{1-\overline{\alpha}_{t}}{\overline{\alpha}_{t}}\right)^{2}\mathsf{Tr}\bigg(\mathbb{E}\Big[\Big(\Sigma_{\overline{\alpha}_{t}}\big(\sqrt{\overline{\alpha}_{t}}X_{0}+\sqrt{1-\overline{\alpha}_{t}}Z\big)\Big)^{2}\Big]\bigg)\cdot{\displaystyle \int}_{\overline{\alpha}_{t+1}}^{\overline{\alpha}_{t}}\frac{1}{(1-\overline{\alpha})^{2}}\mathrm{d}\overline{\alpha}\\
 \notag & \qquad\overset{\mathrm{(i)}}{\geq}\frac{1-\overline{\alpha}_{t+1}}{4\overline{\alpha}_{t+1}}\cdot\frac{1-\overline{\alpha}_{t}}{\overline{\alpha}_{t}}\cdot\mathsf{Tr}\bigg(\mathbb{E}\Big[\Big(\Sigma_{\overline{\alpha}_{t}}\big(\sqrt{\overline{\alpha}_{t}}X_{0}+\sqrt{1-\overline{\alpha}_{t}}Z\big)\Big)^{2}\Big]\bigg)\cdot\frac{\overline{\alpha}_{t}-\overline{\alpha}_{t+1}}{(1-\overline{\alpha}_{t})(1-\overline{\alpha}_{t+1})}\\
 \notag & \qquad\overset{\mathrm{(ii)}}{=}\frac{1-\alpha_{t+1}}{4\overline{\alpha}_{t+1}}\mathsf{Tr}\bigg(\mathbb{E}\Big[\Big(\Sigma_{\overline{\alpha}_{t}}\big(\sqrt{\overline{\alpha}_{t}}X_{0}+\sqrt{1-\overline{\alpha}_{t}}Z\big)\Big)^{2}\Big]\bigg)\\
 & \qquad\overset{\mathrm{(iii)}}{\geq}\frac{1-\alpha_{t}}{4\overline{\alpha}_{t}}\mathsf{Tr}\bigg(\mathbb{E}\Big[\Big(\Sigma_{\overline{\alpha}_{t}}\big(\sqrt{\overline{\alpha}_{t}}X_{0}+\sqrt{1-\overline{\alpha}_{t}}Z\big)\Big)^{2}\Big]\bigg) \label{eqn:brahms2}
\end{align}
for any $t\geq2$, where 
(i) results from \eqref{eqn:properties-alpha-proof}, 
(ii) is valid since $\overline{\alpha}_{t}-\overline{\alpha}_{t+1}=\overline{\alpha}_{t}(1-\alpha_{t+1})$, 
and (iii) uses the property $\alpha_{t}\leq1$ (cf.~\eqref{eqn:properties-alpha-proof-00}) 
and the fact that $1-\alpha_{t}\leq1-\alpha_{t+1}$ for $t\geq2$. 
Recall that (cf.~\eqref{eqn:properties-alpha-proof-1})
\[
0\leq \frac{\overline{\alpha}_{t}-\overline{\alpha}_{t+1}}{\overline{\alpha}_{t}(1-\overline{\alpha}_{t})}=\frac{1-\alpha_{t+1}}{1-\overline{\alpha}_{t}}\lesssim\frac{\log T}{T}\lesssim\frac{1}{d\log T}
\]
and $1-\overline{\alpha}_{t+1} \geq 1-\overline{\alpha}_{t} \geq 1-\overline{\alpha}_{1}= T^{-c_0}$. 
Taking inequality~\eqref{eqn:brahms2} together with Lemma~\ref{lem:cond-covariance}(a) yields 
\begin{align}
 \notag & {\displaystyle \int}_{\overline{\alpha}_{t+1}}^{\overline{\alpha}_{t}}\left(\frac{1-\overline{\alpha}}{\overline{\alpha}}\right)^{2}\mathsf{Tr}\bigg(\mathbb{E}\Big[\Big(\Sigma_{\overline{\alpha}}\big(\sqrt{\overline{\alpha}}X_{0}+\sqrt{1-\overline{\alpha}}Z\big)\Big)^{2}\Big]\bigg)\frac{1}{(1-\overline{\alpha})^{2}}\mathrm{d}\overline{\alpha}\\
 \notag & \qquad\geq C_{8}\left(\frac{1-\overline{\alpha}_{t}}{\overline{\alpha}_{t}}\right)^{2}\mathsf{Tr}\bigg(\mathbb{E}\Big[\Big(\Sigma_{\overline{\alpha}_{t}}\big(\sqrt{\overline{\alpha}_{t}}X_{0}+\sqrt{1-\overline{\alpha}_{t}}Z\big)\Big)^{2}\Big]\bigg)\cdot{\displaystyle \int}_{\overline{\alpha}_{t+1}}^{\overline{\alpha}_{t}}\frac{1}{(1-\overline{\alpha})^{2}}\mathrm{d}\overline{\alpha} -C_{10}\exp(-C_{9}d\log T){\displaystyle \int}_{\overline{\alpha}_{t+1}}^{\overline{\alpha}_{t}}\frac{1}{\overline{\alpha}^{2}}\mathrm{d}\overline{\alpha}\\
 \notag & \qquad \geq C_{8}\frac{1-\alpha_{t}}{4\overline{\alpha}_{t}}\mathsf{Tr}\bigg(\mathbb{E}\Big[\Big(\Sigma_{\overline{\alpha}_{t}}\big(\sqrt{\overline{\alpha}_{t}}X_{0}+\sqrt{1-\overline{\alpha}_{t}}Z\big)\Big)^{2}\Big]\bigg)-C_{10}\exp(-C_{9}d\log T)\frac{\overline{\alpha}_{t}-\overline{\alpha}_{t+1}}{\overline{\alpha}_{t+1}\overline{\alpha}_{t}}\\
 & \qquad\geq C_{8}\frac{1-\alpha_{t}}{4\overline{\alpha}_{t}}\left\{ \mathsf{Tr}\bigg(\mathbb{E}\Big[\Big(\Sigma_{\overline{\alpha}_{t}}\big(\sqrt{\overline{\alpha}_{t}}X_{0}+\sqrt{1-\overline{\alpha}_{t}}Z\big)\Big)^{2}\Big]\bigg)-C_{12}\exp(-C_{11}d\log T)\right\} \label{eqn:vive}
\end{align}
for some universal constants $C_{8},C_{9},C_{10},C_{11},C_{12}>0$,
where the first inequality invokes Lemma~\ref{lem:cond-covariance}(a),
and the last inequality is valid since $\frac{\overline{\alpha}_{t}-\overline{\alpha}_{t+1}}{\overline{\alpha}_{t+1}\overline{\alpha}_{t}}=\frac{1-\alpha_{t+1}}{\overline{\alpha}_{t+1}}\asymp\frac{1-\alpha_{t}}{\overline{\alpha}_{t}}$.

Combine inequality~\eqref{eqn:vive} with (\ref{eq:int-Tr-cov-inequality-1}) (with $\overline{\alpha}_{\mathsf{l}}=\overline{\alpha}_{t+1}$
and $\overline{\alpha}_{\mathsf{u}}=\overline{\alpha}_{t}$) to reach
\begin{align*}
	&	C_{8}\frac{1-\alpha_{t}}{4\overline{\alpha}_{t}}\left\{ \mathsf{Tr}\bigg(\mathbb{E}\Big[\Big(\Sigma_{\overline{\alpha}_{t}}\big(\sqrt{\overline{\alpha}_{t}}X_{0}+\sqrt{1-\overline{\alpha}_{t}}Z\big)\Big)^{2}\Big]\bigg)-C_{12}\exp(-C_{11}d\log T)\right\} \\
 & \leq\frac{1-\overline{\alpha}_{t+1}}{\overline{\alpha}_{t+1}}\mathsf{Tr}\Big(\mathbb{E}\big[\Sigma_{\overline{\alpha}_{t+1}}\big(\sqrt{\overline{\alpha}_{t+1}}X_{0}+\sqrt{1-\overline{\alpha}_{t+1}}Z\big)\big]\Big)-\frac{1-\overline{\alpha}_{t}}{\overline{\alpha}_{t}}\mathsf{Tr}\Big(\mathbb{E}\big[\Sigma_{\overline{\alpha}_{t}}\big(\sqrt{\overline{\alpha}_{t}}X_{0}+\sqrt{1-\overline{\alpha}_{t}}Z\big)\big]\Big).
\end{align*}
Multiplying both sides by $\frac{\overline{\alpha}_{t}}{1-\overline{\alpha}_{t}}$,
we are left with
\begin{align}
 & C_{8}\frac{1-\alpha_{t}}{4(1-\overline{\alpha}_{t})}\left\{ \mathsf{Tr}\bigg(\mathbb{E}\Big[\Big(\Sigma_{\overline{\alpha}_{t}}\big(\sqrt{\overline{\alpha}_{t}}X_{0}+\sqrt{1-\overline{\alpha}_{t}}Z\big)\Big)^{2}\Big]\bigg)-C_{12}\exp(-C_{11}d\log T)\right\} 
	\nonumber\\	
 & \leq\frac{1-\overline{\alpha}_{t+1}}{\alpha_{t+1}(1-\overline{\alpha}_{t})}\mathsf{Tr}\Big(\mathbb{E}\big[\Sigma_{\overline{\alpha}_{t+1}}\big(\sqrt{\overline{\alpha}_{t+1}}X_{0}+\sqrt{1-\overline{\alpha}_{t+1}}Z\big)\big]\Big)-\mathsf{Tr}\Big(\mathbb{E}\big[\Sigma_{\overline{\alpha}_{t}}\big(\sqrt{\overline{\alpha}_{t}}X_{0}+\sqrt{1-\overline{\alpha}_{t}}Z\big)\big]\Big)\nonumber\\
 & \leq\mathsf{Tr}\Big(\mathbb{E}\big[\Sigma_{\overline{\alpha}_{t+1}}\big(\sqrt{\overline{\alpha}_{t+1}}X_{0}+\sqrt{1-\overline{\alpha}_{t+1}}Z\big)\big]\Big)-\mathsf{Tr}\Big(\mathbb{E}\big[\Sigma_{\overline{\alpha}_{t}}\big(\sqrt{\overline{\alpha}_{t}}X_{0}+\sqrt{1-\overline{\alpha}_{t}}Z\big)\big]\Big)\nonumber\\
 & \qquad+\frac{32c_{1}\log T}{T}\mathsf{Tr}\Big(\mathbb{E}\big[\Sigma_{\overline{\alpha}_{t+1}}\big(\sqrt{\overline{\alpha}_{t+1}}X_{0}+\sqrt{1-\overline{\alpha}_{t+1}}Z\big)\big]\Big)\nonumber\\
 & \leq\mathsf{Tr}\Big(\mathbb{E}\big[\Sigma_{\overline{\alpha}_{t+1}}\big(\sqrt{\overline{\alpha}_{t+1}}X_{0}+\sqrt{1-\overline{\alpha}_{t+1}}Z\big)\big]\Big)-\mathsf{Tr}\Big(\mathbb{E}\big[\Sigma_{\overline{\alpha}_{t}}\big(\sqrt{\overline{\alpha}_{t}}X_{0}+\sqrt{1-\overline{\alpha}_{t}}Z\big)\big]\Big)+\frac{32c_{1}d\log T}{T}
	\label{eq:Tr-Z-cov-UB-telescope}
\end{align}
for any $t\geq2$. Here, the penultimate inequality in (\ref{eq:Tr-Z-cov-UB-telescope})
holds since (according to \eqref{eqn:properties-alpha-proof-00},
\eqref{eqn:properties-alpha-proof-1} and \eqref{eqn:properties-alpha-proof-3})
\[
\frac{1-\overline{\alpha}_{t+1}}{\alpha_{t+1}(1-\overline{\alpha}_{t})}-1=\frac{1-\alpha_{t+1}}{\alpha_{t+1}(1-\overline{\alpha}_{t})}\leq\frac{4(1-\alpha_{t+1})}{1-\overline{\alpha}_{t+1}}\leq\frac{32c_{1}\log T}{T};
\]
and the last inequality in (\ref{eq:Tr-Z-cov-UB-telescope}) follows
since, for any $\overline{\alpha}\in(0,1)$, 
\begin{equation}
\mathbb{E}\big[\Sigma_{\overline{\alpha}}\big(\sqrt{\overline{\alpha}}X_{0}+\sqrt{1-\overline{\alpha}}Z\big)\big]=\mathbb{E}\left[\mathsf{Cov}\big(Z\mid\sqrt{\overline{\alpha}}X_{0}+\sqrt{1-\overline{\alpha}}Z\big)\right]\preceq\mathsf{Cov}(Z)=I_{d}.\label{eq:cov-Z-UB}
\end{equation}

Consequently, sum over $t=2,\dots,T$ to form a telescopic
sum and derive
\begin{align*}
 & C_{8}\sum_{t=2}^{T}\frac{1-\alpha_{t}}{4(1-\overline{\alpha}_{t})}\mathsf{Tr}\bigg(\mathbb{E}\Big[\Big(\Sigma_{\overline{\alpha}_{t}}\big(\sqrt{\overline{\alpha}_{t}}X_{0}+\sqrt{1-\overline{\alpha}_{t}}Z\big)\Big)^{2}\Big]\bigg)\nonumber\\
 & \leq\sum_{t=2}^{T}\left\{ \mathsf{Tr}\Big(\mathbb{E}\big[\Sigma_{\overline{\alpha}_{t+1}}\big(\sqrt{\overline{\alpha}_{t+1}}X_{0}+\sqrt{1-\overline{\alpha}_{t+1}}Z\big)\big]\Big)-\mathsf{Tr}\Big(\mathbb{E}\big[\Sigma_{\overline{\alpha}_{t}}\big(\sqrt{\overline{\alpha}_{t}}X_{0}+\sqrt{1-\overline{\alpha}_{t}}Z\big)\big]\Big)\right\} +32c_{1}d\log T\\
 & \qquad+C_{8}\sum_{t=2}^{T}\frac{1-\alpha_{t}}{4(1-\overline{\alpha}_{t})}C_{12}\exp(-C_{11}d\log T)\\
 & \leq\mathsf{Tr}\Big(\mathbb{E}\big[\Sigma_{\overline{\alpha}_{T+1}}\big(\sqrt{\overline{\alpha}_{T+1}}X_{0}+\sqrt{1-\overline{\alpha}_{T+1}}Z\big)\big]\Big)+32c_{1}d\log T+2c_{1}C_{8}C_{12}\exp(-C_{11}d\log T)\log T\\
 & \leq34c_{1}d\log T,
\end{align*}
where the last inequality uses \eqref{eq:cov-Z-UB} and property~\eqref{eqn:properties-alpha-proof-1}.
This concludes the proof.

\subsection{Proof of Lemma~\ref{lem:KL-T}}
\label{sec:proof-lem-KL-T}

Recognizing that $Y_T\sim \mathcal{N}(0, I_d)$ and that
$X_T \overset{\mathrm{d}}{=} \sqrt{\overline{\alpha}_T} X_0 + \sqrt{1 - \overline{\alpha}_T}\,\overline{W}_t$ with $\overline{W}_t\sim \mathcal{N}(0, I_d)$ (independent from $X_0$),  
one has 
\begin{align}
\notag\mathsf{KL}(p_{X_{T}}\parallel p_{Y_{T}}) & =\int p_{X_{T}}(x)\log\frac{p_{X_{T}}(x)}{p_{Y_{T}}(x)}\diff x\\
\notag & \stackrel{(\text{i})}{=}\int p_{X_{T}}(x)\log\frac{\int_{y:\|y\|_{2}\leq\sqrt{\overline{\alpha}_{T}}T^{c_{R}}}p_{\sqrt{\overline{\alpha}_{T}}X_{0}}(y)p_{\sqrt{1-\overline{\alpha}_{T}}\,\overline{W}_{t}}(x-y)\diff y}{p_{Y_{T}}(x)}\diff x\\
\notag & \leq\int p_{X_{T}}(x)\log\frac{\sup_{y:\|y\|_{2}\leq\sqrt{\overline{\alpha}_{T}}T^{c_{R}}}p_{\sqrt{1-\overline{\alpha}_{T}}\,\overline{W}_{t}}(x-y)}{p_{Y_{T}}(x)}\diff x\\
\notag & =\int p_{X_{T}}(x)\Bigg(-d/2\log(1-\overline{\alpha}_{T})+\sup_{y:\|y\|_{2}\leq\sqrt{\overline{\alpha}_{T}}T^{c_{R}}}\bigg(-\frac{\|x-y\|_{2}^{2}}{2(1-\overline{\alpha}_{T})}+\frac{\|x\|_{2}^{2}}{2}\Bigg)\diff x\\
\notag & \stackrel{(\text{ii})}{\leq}\int p_{X_{T}}(x)\bigg(-d/2\log(1-\overline{\alpha}_{T})+\|x\|_{2}\sup_{y:\|y\|_{2}\leq\sqrt{\overline{\alpha}_{T}}T^{c_{R}}}\frac{\|y\|_{2}}{1-\overline{\alpha}_{T}}\bigg)\diff x\\
\notag & \leq-d/2\log(1-\overline{\alpha}_{T})+\frac{\sqrt{\overline{\alpha}_{T}}T^{c_{R}}}{2(1-\overline{\alpha}_{T})}\mathbb{E}\left[\|X_{T}\|_{2}\right]\\
 & \stackrel{(\text{iii})}{\lesssim}\overline{\alpha}_{T}d+\frac{\sqrt{\overline{\alpha}_{T}}T^{c_{R}}}{2(1-\overline{\alpha}_{T})}\left(\sqrt{\overline{\alpha}_{T}}T^{c_{R}}+\sqrt{d}\right)\stackrel{(\text{iv})}{\lesssim}\frac{1}{T^{200}},\label{eqn:KL-T-123}
\end{align}
where $(\text{i})$ arises from the assumption that $\ltwo{X_0} \leq T^{c_R}$, 
(ii) applies the Cauchy-Schwarz inequality,  (iii) holds true since
\[
\mathbb{E}\left[\|X_{T}\|_{2}\right]\leq\sqrt{\overline{\alpha}_{T}}\|X_{0}\|_{2}+\mathbb{E}\left[\|\overline{W}_{t}\|_{2}\right]\leq\sqrt{\overline{\alpha}_{T}}T^{c_{R}}+\sqrt{\mathbb{E}\left[\|\overline{W}_{t}\|_{2}^{2}\right]}\leq\sqrt{\overline{\alpha}_{T}}T^{c_{R}}+\sqrt{d},
\]
and (iv) makes use of \eqref{eqn:properties-alpha-proof-alphaT} given that $c_2\geq 1000$.
The proof is thus completed by invoking the Pinsker inequality \citep[Lemma~2.5]{tsybakov2009introduction}.

\section{Proof of auxiliary lemmas}
\label{sec:proof-lem-auxiliary-ode}


\subsection{Proof of Lemma~\ref{lem:main-ODE}}
\label{sec:proof-lem:main-ODE}


\subsubsection{Proof of relations~\eqref{eq:xt_up} and \eqref{eq:xt}}

Recall the definition of $\phi_t$ and $\phi_t^{\star}$ in \eqref{defn:phit-x}, 
and introduce the following vector: 
\begin{align}
	u &\coloneqq x - \phi_t(x) = x - \phi_t^{\star}(x) + \phi_t^{\star}(x) - \phi_t(x)  \notag\\
	&= \frac{1-\alpha_{t}}{2(1-\overline{\alpha}_{t})} \int_{x_0} \big(x - \sqrt{\overline{\alpha}_{t}}x_0\big) p_{X_0 \mymid X_{t}}(x_0 \,|\, x) \mathrm{d} x_0 - 
	\frac{1-\alpha_{t}}{2}\big(s_t(x) - s_t^{\star}(x)\big).
	\label{eq:defn-u-Lemma-main-ODE}
\end{align}
The proof is composed of the following steps.

\paragraph{Step 1: decomposing $p_{\sqrt{\alpha_{t}}X_{t-1}}\big(\phi_{t}(x)\big)/p_{X_{t}}(x)$.} 
Recognizing that 
\begin{equation}
	X_t \overset{\mathrm{d}}{=} \sqrt{\overline{\alpha}_t}X_0+ \sqrt{1-\overline{\alpha}_t}\,W 
	\qquad \text{with }W\sim \mathcal{N}(0,I_d)
	\label{eq:Xt-dist-proof1}
\end{equation}
and making use of the Bayes rule, we can express the conditional distribution  as
\begin{align}
\label{eqn:bayes}
p_{X_{0}\mymid X_{t}}(x_{0}\,|\,x)=\frac{p_{X_{0}}(x_{0})}{p_{X_{t}}(x)}p_{X_{t}\mymid X_{0}}(x\,|\,x_{0})=\frac{p_{X_{0}}(x_{0})}{p_{X_{t}}(x)}\cdot\frac{1}{\big(2\pi(1-\overline{\alpha}_{t})\big)^{d/2}}\exp\bigg(-\frac{\big\| x-\sqrt{\overline{\alpha}_{t}}x_{0}\big\|_{2}^{2}}{2(1-\overline{\alpha}_{t})}\bigg).	
\end{align}
Moreover, it follows from \eqref{eq:Xt-dist-proof1} that
\begin{equation}
	\sqrt{\alpha_t} X_{t-1} \overset{\mathrm{d}}{=} \sqrt{\alpha_t} \big(\sqrt{\overline{\alpha}_{t-1}}X_0+ \sqrt{1-\overline{\alpha}_{t-1}}\,W\big) 
	= \sqrt{\overline{\alpha}_{t}}X_0+ \sqrt{\alpha_t-\overline{\alpha}_{t}}\,W .
	\label{eq:Xt-1-rescale-expression}
\end{equation}
These taken together allow one to rewrite $p_{\sqrt{\alpha_t}X_{t-1}}$ such that: 
%
\begin{align}
\frac{p_{\sqrt{\alpha_{t}}X_{t-1}}\big(\phi_{t}(x)\big)}{p_{X_{t}}(x)} & \overset{\mathrm{(i)}}{=}\frac{1}{p_{X_{t}}(x)}\int_{x_{0}}p_{X_{0}}(x_{0})\frac{1}{\big(2\pi(\alpha_{t}-\overline{\alpha}_{t})\big)^{d/2}}\exp\bigg(-\frac{\big\|\phi_{t}(x)-\sqrt{\overline{\alpha}_{t}}x_{0}\big\|_{2}^{2}}{2(\alpha_{t}-\overline{\alpha}_{t})}\bigg)\mathrm{d}x_{0}\notag\\
 & \overset{\mathrm{(ii)}}{=}\frac{1}{p_{X_{t}}(x)}\int_{x_{0}}p_{X_{0}}(x_{0})\frac{1}{\big(2\pi(\alpha_{t}-\overline{\alpha}_{t})\big)^{d/2}}\exp\bigg(-\frac{\big\| x-\sqrt{\overline{\alpha}_{t}}x_{0}\big\|_{2}^{2}}{2(1-\overline{\alpha}_{t})}\bigg)\notag\\
 & \qquad\cdot\exp\bigg(-\frac{(1-\alpha_{t})\big\| x-\sqrt{\overline{\alpha}_{t}}x_{0}\big\|_{2}^{2}}{2(\alpha_{t}-\overline{\alpha}_{t})(1-\overline{\alpha}_{t})}-\frac{\|u\|_{2}^{2}-2u^{\top}\big(x-\sqrt{\overline{\alpha}_{t}}x_{0}\big)}{2(\alpha_{t}-\overline{\alpha}_{t})}\bigg)\mathrm{d}x_{0}\notag\\
 & \overset{\mathrm{(iii)}}{=}\Big(\frac{1-\overline{\alpha}_{t}}{\alpha_{t}-\overline{\alpha}_{t}}\Big)^{d/2}\cdot\int_{x_{0}}p_{X_{0}\mymid X_{t}}(x_{0}\mymid x)\cdot\notag\\
 & \qquad\qquad\qquad\exp\bigg(-\frac{(1-\alpha_{t})\big\| x-\sqrt{\overline{\alpha}_{t}}x_{0}\big\|_{2}^{2}}{2(\alpha_{t}-\overline{\alpha}_{t})(1-\overline{\alpha}_{t})}-\frac{\|u\|_{2}^{2}-2u^{\top}\big(x-\sqrt{\overline{\alpha}_{t}}x_{0}\big)}{2(\alpha_{t}-\overline{\alpha}_{t})}\bigg)\mathrm{d}x_{0}\label{eqn:fei-0}\\
 & \overset{\mathrm{(iv)}}{=}\left\{ 1+\frac{d(1-\alpha_{t})}{2(\alpha_{t}-\overline{\alpha}_{t})}+O\bigg(d^{2}\Big(\frac{1-\alpha_{t}}{\alpha_{t}-\overline{\alpha}_{t}}\Big)^{2}\bigg)\right\} \cdot \notag\\
 & \quad\int_{x_{0}}p_{X_{0}\mymid X_{t}}(x_{0}\mymid x)\exp\bigg(-\frac{(1-\alpha_{t})\big\| x-\sqrt{\overline{\alpha}_{t}}x_{0}\big\|_{2}^{2}}{2(\alpha_{t}-\overline{\alpha}_{t})(1-\overline{\alpha}_{t})}-\frac{\|u\|_{2}^{2}-2u^{\top}\big(x-\sqrt{\overline{\alpha}_{t}}x_{0}\big)}{2(\alpha_{t}-\overline{\alpha}_{t})}\bigg)\mathrm{d}x_{0}.
	\label{eqn:fei}
\end{align}
Here,  identity (i) holds due to \eqref{eq:Xt-1-rescale-expression} 
and hence
\[
p_{\sqrt{\alpha_{t}}X_{t-1}}(x) =\int_{x_{0}}p_{X_{0}}(x_{0})p_{\sqrt{\alpha_{t}-\overline{\alpha}_{t}}W}\big(x-\sqrt{\overline{\alpha}_{t}}x_{0}\big)\mathrm{d}x_{0}; 
\]
identity (ii) follows from \eqref{eq:defn-u-Lemma-main-ODE} and elementary algebra; 
relation (iii) is a consequence of the Bayes rule \eqref{eqn:bayes}; 
and relation (iv) results from \eqref{eq:expansion-ratio-1-alpha}.

\paragraph{Step 2: controlling the integral in the decomposition~\eqref{eqn:fei}.}
In order to further control the right-hand side of expression~\eqref{eqn:fei}, we need to evaluate the integral in~\eqref{eqn:fei}. 
To this end, we make a few observations. 
\begin{itemize}
	\item To begin with,  Lemma~\ref{lem:x0} tells us that 
\begin{subequations}
\label{eqn:BBB}
\begin{align}
\label{eqn:brahms}
	\mathbb{P}\Big(\big\|\sqrt{\overline{\alpha}_{t}}X_0 - x\big\|_2 > 5c_5 \sqrt{\theta_t(x) d(1 - \overline{\alpha}_{t})\log T}\mymid X_t = x\Big) \leq \exp\big(-c_5^2\theta_t(x) d\log T\big) 
\end{align}
for any quantity $c_5 \ge 2$, provided that $c_6\geq 2c_R+c_0$. 

	\item A little algebra based on this relation allows one to bound $u$ (cf.~\eqref{eq:defn-u-Lemma-main-ODE}) as follows:  
\begin{align}
\|u\|_{2} & \le \frac{1-\alpha_{t}}{2}\varepsilon_{\score, t}(x) + \frac{1-\alpha_{t}}{2(1-\overline{\alpha}_{t})}\mathbb{E}\left[\big\| \sqrt{\overline{\alpha}_{t}}X_{0} - x \big\|_{2}\,\big|\,X_{t}=x\right] \notag\\
	& \leq \frac{1-\alpha_{t}}{2}\varepsilon_{\score, t}(x) + \frac{6(1-\alpha_{t})}{1-\overline{\alpha}_{t}} \sqrt{\theta_t(x) d(1 - \overline{\alpha}_{t})\log T},
	\label{eqn:johannes}
\end{align}
\end{subequations}
where the last inequality arises from Lemma~\ref{lem:x0}. 
\end{itemize}
Next, let us define
\begin{equation}
	\mathcal{E}_c^{\mathsf{typical}}
	\coloneqq\Big\{ x_0:\big\| x-\sqrt{\overline{\alpha}_{t}}x_{0}\big\|_{2}\leq 5c\sqrt{\theta_t(x) d(1-\overline{\alpha}_{t})\log T}\Big\}	 
	\label{eq:defn-Ec-typical-set}
\end{equation}
for any quantity $c>0$. 
Then for any $x_0 \in \mathcal{E}_c^{\mathsf{typical}}$, it is clearly seen from \eqref{eqn:BBB} and \eqref{eqn:properties-alpha-proof} that
\begin{subequations}
	\label{eq:long-UB-123}
\begin{align}
\frac{(1-\alpha_{t})\big\| x-\sqrt{\overline{\alpha}_{t}}x_{0}\big\|_{2}^{2}}{2(\alpha_{t}-\overline{\alpha}_{t})(1-\overline{\alpha}_{t})} & \leq\frac{25c^{2}}{2}\frac{(1-\alpha_{t})\theta_t(x)d\log T}{\alpha_{t}-\overline{\alpha}_{t}}\leq\frac{100c_{1}c^{2}\theta_t(x)d\log^{2}T}{T};\label{eq:long-UB-1}\\
\frac{\|u\|_{2}^{2}}{2(\alpha_{t}-\overline{\alpha}_{t})} & \leq\frac{(1-\alpha_{t})^{2}}{4(\alpha_{t}-\overline{\alpha}_{t})}\varepsilon_{\score,t}(x)^{2}+\frac{36(1-\alpha_{t})^{2}}{(1-\overline{\alpha}_{t})(\alpha_{t}-\overline{\alpha}_{t})}\theta_t(x)d\log T\label{eq:long-UB-2}\\
 & \leq\frac{2c_{1}^{2}\log^{2}T}{T^{2}}\varepsilon_{\score,t}(x)^{2}+\frac{2304c_{1}^{2}}{T^{2}}\theta_t(x)d\log^{3}T,\notag\\
\left|\frac{u^{\top}\big(x-\sqrt{\overline{\alpha}_{t}}x_{0}\big)}{\alpha_{t}-\overline{\alpha}_{t}}\right| & \leq\frac{\|u\|_{2}\big\| x-\sqrt{\overline{\alpha}_{t}}x_{0}\big\|_{2}}{\alpha_{t}-\overline{\alpha}_{t}}\nonumber\\
 & \leq\frac{5c(1-\alpha_{t})}{2(\alpha_{t}-\overline{\alpha}_{t})}\varepsilon_{\score,t}(x)\sqrt{\theta_t(x)d(1-\overline{\alpha}_{t})\log T}+\frac{30c(1-\alpha_{t})\theta_t(x)d\log T}{\alpha_{t}-\overline{\alpha}_{t}}\label{eq:long-UB-3}\\
 & \leq\frac{20cc_{1}}{T}\varepsilon_{\score,t}(x)\sqrt{\theta_t(x)d(1-\overline{\alpha}_{t})\log^{3}T}+\frac{240cc_{1}\theta_t(x)d\log^{2}T}{T}. \label{eq:long-UB-4}
\end{align}
\end{subequations}
As a consequence, for any $x_0\in \mathcal{E}_c^{\mathsf{typical}}$ for $c \ge 2$, we have seen from \eqref{eq:long-UB-4} and \eqref{eqn:properties-alpha-proof} that
\begin{align}
 & -\frac{(1-\alpha_{t})\big\| x-\sqrt{\overline{\alpha}_{t}}x_{0}\big\|_{2}^{2}}{2(\alpha_{t}-\overline{\alpha}_{t})(1-\overline{\alpha}_{t})}-\frac{\|u\|_{2}^{2}}{2(\alpha_{t}-\overline{\alpha}_{t})}+\frac{u^{\top}\big(x-\sqrt{\overline{\alpha}_{t}}x_{0}\big)}{\alpha_{t}-\overline{\alpha}_{t}} 
	\leq\frac{u^{\top}\big(x-\sqrt{\overline{\alpha}_{t}}x_{0}\big)}{\alpha_{t}-\overline{\alpha}_{t}}
	\notag\\
 &\qquad \leq \frac{5c(1-\alpha_{t})}{2(\alpha_{t}-\overline{\alpha}_{t})}\varepsilon_{\score,t}(x)\sqrt{\theta_t(x)d(1-\overline{\alpha}_{t})\log T}+\frac{30c(1-\alpha_{t})\theta_t(x)d\log T}{\alpha_{t}-\overline{\alpha}_{t}} \label{eq:multi-term-UB-45678}\\	
 & \qquad
 \le\frac{20cc_{1}}{T}\varepsilon_{\score,t}(x)\sqrt{\theta_t(x)d\log^3 T}+\frac{240cc_{1}}{T}\theta_t(x)d\log^{2}T 
	\le c\theta_t(x)d, \label{eq:multi-term-UB-456}
\end{align}
provided that 
\[
\frac{40c_{1}\varepsilon_{\score,t}(x)\log^{\frac{3}{2}}T}{T}\leq\sqrt{\theta_t(x)d}\qquad\text{and}\qquad T\geq480c_{1}\log^{2}T.
\]

\paragraph{Step 2(a): proof of relation \eqref{eq:xt_up}.}

Substituting \eqref{eq:multi-term-UB-45678} into \eqref{eqn:fei} and making use of \eqref{eqn:properties-alpha-proof} under our assumption on $T$ yield
\begin{align*}
\frac{p_{\sqrt{\alpha_{t}}X_{t-1}}\big(\phi_{t}(x)\big)}{p_{X_{t}}(x)} & \leq2\exp\bigg(\frac{5c(1-\alpha_{t})}{2(\alpha_{t}-\overline{\alpha}_{t})}\varepsilon_{\score,t}(x)\sqrt{\theta_t(x)d\log T}+\frac{30c(1-\alpha_{t})}{\alpha_{t}-\overline{\alpha}_{t}}\theta_t(x)d\log T\bigg)\int_{x_{0}}p_{X_{0}\mymid X_{t}}(x_{0}\mymid x)\mathrm{d}x_{0}\\
 & \leq2\exp\bigg(\frac{5c(1-\alpha_{t})}{2(\alpha_{t}-\overline{\alpha}_{t})}\varepsilon_{\score,t}(x)\sqrt{\theta_t(x)d\log T}+\frac{30c(1-\alpha_{t})}{\alpha_{t}-\overline{\alpha}_{t}}\theta_t(x)d\log T\bigg),
\end{align*}
thus establishing \eqref{eq:xt_up} by taking $c=2$.

\paragraph{Step 2(b): proof of relation \eqref{eq:xt}.}
Suppose now that
\begin{equation}
C_{10}\frac{\theta_t(x)d\log^{2}T+\varepsilon_{\score,t}(x)\sqrt{\theta_t(x)d\log^{3}T}}{T}\leq1
	\label{eq:Lemma3-strong-assump}
\end{equation}
holds for some large enough constant $C_{10}>0$. 
Under this additional condition, it can be easily verified that
\begin{align}
 & \left|-\frac{(1-\alpha_{t})\big\| x-\sqrt{\overline{\alpha}_{t}}x_{0}\big\|_{2}^{2}}{2(\alpha_{t}-\overline{\alpha}_{t})(1-\overline{\alpha}_{t})}-\frac{\|u\|_{2}^{2}}{2(\alpha_{t}-\overline{\alpha}_{t})}+\frac{u^{\top}\big(x-\sqrt{\overline{\alpha}_{t}}x_{0}\big)}{\alpha_{t}-\overline{\alpha}_{t}}\right|\notag\\
 & \qquad \leq c_{10}\left(\theta_t(x)d\log T+\varepsilon_{\score,t}(x)\sqrt{\theta_t(x)d\log T}\right)\frac{1-\alpha_{t}}{\alpha_{t}-\overline{\alpha}_{t}}
	\label{eq:exponent-UB-1357}
\end{align}
for any $x_0 \in \mathcal{E}_2^{\mathsf{typical}}$ (with $c=2$), where  $c_{10}>0$ is some sufficiently small constant. 
Therefore,   
the Taylor expansion  $e^{-z} = 1 - z + O(z^2)$ (for all $|z| < 1$) gives 
\begin{align}
 & \exp\bigg(-\frac{(1-\alpha_{t})\big\| x-\sqrt{\overline{\alpha}_{t}}x_{0}\big\|_{2}^{2}}{2(\alpha_{t}-\overline{\alpha}_{t})(1-\overline{\alpha}_{t})}-\frac{\|u\|_{2}^{2}-2u^{\top}\big(x-\sqrt{\overline{\alpha}_{t}}x_{0}\big)}{2(\alpha_{t}-\overline{\alpha}_{t})}\bigg)\notag\\
%
 & =1-\frac{(1-\alpha_{t})\big\| x-\sqrt{\overline{\alpha}_{t}}x_{0}\big\|_{2}^{2}}{2(\alpha_{t}-\overline{\alpha}_{t})(1-\overline{\alpha}_{t})}+\frac{u^{\top}\big(x-\sqrt{\overline{\alpha}_{t}}x_{0}\big)}{\alpha_{t}-\overline{\alpha}_{t}}+O\bigg(\big(\theta_t(x)^2d^{2}\log^{2}T + \varepsilon_{\score, t}(x)^2\theta_t(x) d\log T\big)\Big(\frac{1-\alpha_{t}}{\alpha_{t}-\overline{\alpha}_{t}}\Big)^{2}\bigg)
	\label{eq:exp-UB-13579}
\end{align}
for any $x_0 \in \mathcal{E}_2^{\mathsf{typical}}$, 
which invokes \eqref{eq:exponent-UB-1357} and \eqref{eq:long-UB-2} (under the assumption~\eqref{eq:Lemma3-strong-assump}).  
%
%
Combine \eqref{eq:exp-UB-13579} and \eqref{eq:multi-term-UB-456} to show that
%
\begin{align}
 & \int_{x_{0}}p_{X_{0}\mymid X_{t}}(x_{0}\mymid x)\exp\bigg(-\frac{(1-\alpha_{t})\big\| x-\sqrt{\overline{\alpha}_{t}}x_{0}\big\|_{2}^{2}}{2(\alpha_{t}-\overline{\alpha}_{t})(1-\overline{\alpha}_{t})}-\frac{\|u\|_{2}^{2}-2u^{\top}\big(x-\sqrt{\overline{\alpha}_{t}}x_{0}\big)}{2(\alpha_{t}-\overline{\alpha}_{t})}\bigg)\mathrm{d}x_{0}\notag\\
 & =\left(\int_{x_{0}\in\mathcal{E}_{2}^{\mathsf{typical}}}+\int_{x_{0}\notin\mathcal{E}_{2}^{\mathsf{typical}}}\right)p_{X_{0}\mymid X_{t}}(x_{0}\mymid x)\exp\bigg(-\frac{(1-\alpha_{t})\big\| x-\sqrt{\overline{\alpha}_{t}}x_{0}\big\|_{2}^{2}}{2(\alpha_{t}-\overline{\alpha}_{t})(1-\overline{\alpha}_{t})}-\frac{\|u\|_{2}^{2}-2u^{\top}\big(x-\sqrt{\overline{\alpha}_{t}}x_{0}\big)}{2(\alpha_{t}-\overline{\alpha}_{t})}\bigg)\mathrm{d}x_{0}\notag\\
 & =\int_{x_{0}\in\mathcal{E}_{2}^{\mathsf{typical}}}p_{X_{0}\mymid X_{t}}(x_{0}\mymid x)\bigg(1-\frac{(1-\alpha_{t})\big\| x-\sqrt{\overline{\alpha}_{t}}x_{0}\big\|_{2}^{2}}{2(\alpha_{t}-\overline{\alpha}_{t})(1-\overline{\alpha}_{t})}+\frac{u^{\top}\big(x-\sqrt{\overline{\alpha}_{t}}x_{0}\big)}{\alpha_{t}-\overline{\alpha}_{t}}\bigg)\mathrm{d}x_{0}\notag\\
 & ~~+O\bigg(\big(\theta_{t}(x)^{2}d^{2}\log^{2}T+\varepsilon_{\score,t}(x)^{2}\theta_{t}(x)d\log T\big)\Big(\frac{1-\alpha_{t}}{\alpha_{t}-\overline{\alpha}_{t}}\Big)^{2}\bigg)+O\left(\sum_{c=3}^{\infty}\int_{x_{0}\in\mathcal{E}_{c}^{\mathsf{typical}}\backslash\mathcal{E}_{c-1}^{\mathsf{typical}}}p_{X_{0}\mymid X_{t}}(x_{0}\mymid x)\exp\left(c\theta_{t}(x)d\right)\mathrm{d}x_{0}\right)\notag\\
 & =1-\frac{(1-\alpha_{t})\big(\int_{x_{0}}p_{X_{0}\mymid X_{t}}(x_{0}\mymid x)\big\| x-\sqrt{\overline{\alpha}_{t}}x_{0}\big\|_{2}^{2}\mathrm{d}x_{0}-\big\|\int_{x_{0}}p_{X_{0}\mymid X_{t}}(x_{0}\mymid x)\big(x-\sqrt{\overline{\alpha}_{t}}x_{0}\big)\mathrm{d}x_{0}\big\|_{2}^{2}\big)}{2(\alpha_{t}-\overline{\alpha}_{t})(1-\overline{\alpha}_{t})}\notag\\
 & \qquad+O\bigg(\theta_{t}(x)^{2}d^{2}\Big(\frac{1-\alpha_{t}}{\alpha_{t}-\overline{\alpha}_{t}}\Big)^{2}\log^{2}T+\varepsilon_{\score,t}(x)\sqrt{\theta_{t}(x)d\log T}\Big(\frac{1-\alpha_{t}}{\alpha_{t}-\overline{\alpha}_{t}}\Big)\bigg)+O\Big(\exp\big(-\theta_{t}(x)d\log T\big)\Big)\notag\\
 & %
=1-\frac{(1-\alpha_{t})\big(\int_{x_{0}}p_{X_{0}\mymid X_{t}}(x_{0}\mymid x)\big\| x-\sqrt{\overline{\alpha}_{t}}x_{0}\big\|_{2}^{2}\mathrm{d}x_{0}-\big\|\int_{x_{0}}p_{X_{0}\mymid X_{t}}(x_{0}\mymid x)\big(x-\sqrt{\overline{\alpha}_{t}}x_{0}\big)\mathrm{d}x_{0}\big\|_{2}^{2}\big)}{2(\alpha_{t}-\overline{\alpha}_{t})(1-\overline{\alpha}_{t})}\notag\\
 & %
\qquad+O\bigg(\theta_{t}(x)^{2}d^{2}\Big(\frac{1-\alpha_{t}}{\alpha_{t}-\overline{\alpha}_{t}}\Big)^{2}\log^{2}T+\varepsilon_{\score,t}(x)\sqrt{\theta_{t}(x)d\log T}\Big(\frac{1-\alpha_{t}}{\alpha_{t}-\overline{\alpha}_{t}}\Big)\bigg), 
	\label{eq:exp-UB-135702}
\end{align}
where the penultimate relation holds since, according to \eqref{eqn:brahms},
\begin{align*}
\sum_{c=3}^{\infty}\int_{x_{0}\in\mathcal{E}_{c}^{\mathsf{typical}}\backslash\mathcal{E}_{c-1}^{\mathsf{typical}}}p_{X_{0}\mymid X_{t}}(x_{0}\mymid x)\exp\left(c\theta_t(x)d\right)\mathrm{d}x_{0} & \leq\sum_{c=3}^{\infty}\exp\left(-c^{2}\theta_t(x)d\log T\right)\exp\left(c\theta_t(x)d\right)\\
 & \leq\sum_{c=3}^{\infty}\exp\left(-\frac{1}{2}c^{2}\theta_t(x)d\log T\right)\leq\exp\big(-\theta_t(x)d\log T\big),
\end{align*}
and the last line in \eqref{eq:exp-UB-135702} again utilizes \eqref{eqn:properties-alpha-proof} and the fact that $\theta_t(x)\geq c_6$ for some large enough constant $c_6>0$.

%
Putting \eqref{eq:exp-UB-135702} and \eqref{eqn:fei} together yields
\begin{align*}
\frac{p_{\sqrt{\alpha_{t}}X_{t-1}}\big(\phi_{t}(x)\big)}{p_{X_{t}}(x)} & =1+\frac{d(1-\alpha_{t})}{2(\alpha_{t}-\overline{\alpha}_{t})}+O\bigg(\theta_t(x)^2d^{2}\Big(\frac{1-\alpha_{t}}{\alpha_{t}-\overline{\alpha}_{t}}\Big)^{2}\log^{2}T + \varepsilon_{\score, t}(x)\sqrt{\theta_t(x) d\log T}\Big(\frac{1-\alpha_{t}}{\alpha_{t}-\overline{\alpha}_{t}}\Big)\bigg)-\notag\\
 & \quad\frac{(1-\alpha_{t})\big(\int_{x_{0}}p_{X_{0}\mymid X_{t}}(x_{0}\mymid x)\big\| x-\sqrt{\overline{\alpha}_{t}}x_{0}\big\|_{2}^{2}\mathrm{d}x_{0}-\big\|\int_{x_{0}}p_{X_{0}\mymid X_{t}}(x_{0}\mymid x)\big(x-\sqrt{\overline{\alpha}_{t}}x_{0}\big)\mathrm{d}x_{0}\big\|_{2}^{2}\big)}{2(\alpha_{t}-\overline{\alpha}_{t})(1-\overline{\alpha}_{t})}
\end{align*}
as claimed.

\subsubsection{Proof of relation~\eqref{eq:yt}} 

Consider any random vector $Y$. 
To understand the density ratio $p_{\phi_t(Y)}(\phi_t(x))/p_{Y}(x)$, we make note of the transformation
\begin{subequations}
\begin{align} \label{eq:dist_tranform}
	p_{\phi_t(Y)}\big(\phi_t(x)\big) &= \mathsf{det}\Big(\frac{\partial \phi_t(x)}{\partial x}\Big)^{-1}p_{Y}(x), \\
	p_{\phi_t^{\star}(Y)}\big(\phi^{\star}_t(x)\big) &= \mathsf{det}\Big(\frac{\partial \phi^{\star}_t(x)}{\partial x}\Big)^{-1}p_{Y}(x),
\end{align}
\end{subequations}
where $\frac{\partial \phi_t(x)}{\partial x}$ and $\frac{\partial \phi^{\star}_t(x)}{\partial x}$  denote the Jacobian matrices. 
It thus suffices to control the quantity $\mathsf{det}\Big(\frac{\partial \phi_t(x)}{\partial x}\Big)^{-1}$.

To begin with, recall from \eqref{defn:phit-x} and \eqref{eq:st-MMSE-expression} that
\[
	\phi_{t}^{\star}(x)=x-\frac{1-\alpha_{t}}{2(1-\overline{\alpha}_{t})}g_{t}(x).
\]
As a result, one can use \eqref{eq:Jacobian-Thm4} and \eqref{eq:Jt-x-expression-ij-23} to derive
\begin{align}
I-\frac{\partial\phi_{t}^{\star}(x)}{\partial x}=\frac{1-\alpha_{t}}{2(1-\overline{\alpha}_{t})}J_{t}(x) & =\frac{1-\alpha_{t}}{2(1-\overline{\alpha}_{t})}\Bigg\{ I-\frac{1}{1-\overline{\alpha}_{t}}\underset{\eqqcolon\,B}{\underbrace{\mathsf{Cov}\big(X_{t}-\sqrt{\overline{\alpha}_{t}}X_{0}\mid X_{t}=x\big)}}\Bigg\}.	
	\label{eq:defn-B-intermediate}
\end{align}
This allows one to show that
\begin{subequations}
\begin{align}
 & \mathsf{Tr}\Big(I-\frac{\partial\phi^{\star}_{t}(x)}{\partial x}\Big)=\frac{d(1-\alpha_{t})}{2(1-\overline{\alpha}_{t})}+ \notag\\
 & \quad\quad\frac{(1-\alpha_{t})\big(\big\|\int_{x_{0}}p_{X_{0}\mymid X_{t}}(x_{0}\mymid x)\big(x-\sqrt{\overline{\alpha}_{t}}x_{0}\big)\mathrm{d}x_{0}\big\|_{2}^{2}-\int_{x_{0}}p_{X_{0}\mymid X_{t}}(x_{0}\mymid x)\big\| x-\sqrt{\overline{\alpha}_{t}}x_{0}\big\|_{2}^{2}\mathrm{d}x_{0}\big)}{2(1-\overline{\alpha}_{t})^{2}}. 
	\label{eq:defn-B-intermediate-2}
\end{align}
Moreover, the matrix $B$ defined in \eqref{eq:defn-B-intermediate} satisfies
\begin{align*}
\|B\|_{\mathrm{F}} 
	\leq \Big\| \mathbb{E}\Big[\big(X_{t}-\sqrt{\overline{\alpha}_{t}}X_{0}\big)\big(X_{t}-\sqrt{\overline{\alpha}_{t}}X_{0}\big)^{\top}\mid X_{t}=x\Big] \Big\|_{\mathrm{F}}
	\leq \int_{x_{0}}p_{X_{0}\mymid X_{t}}(x_{0}\mymid x)\big\| x-\sqrt{\overline{\alpha}_{t}}x_{0}\big\|_{2}^{2}\mathrm{d}x_{0}
\end{align*}
due to Jensen's inequality. 
Taking this together with \eqref{eq:defn-B-intermediate} and Lemma~\ref{lem:x0} reveals that
\begin{align}
	\Big\|\frac{\partial \phi^{\star}_t(x)}{\partial x} - I\Big\| \leq 
\Big\|\frac{\partial \phi^{\star}_t(x)}{\partial x} - I\Big\|_{\mathrm{F}} 
	&\lesssim \frac{1-\alpha_{t}}{1-\overline{\alpha}_{t}}
\bigg(\sqrt{d} + \frac{\int_{x_0}   p_{X_0 \mymid X_{t}}(x_0 \mymid x)\big\|x - \sqrt{\overline{\alpha}_{t}}x_0\big\|_2^2\mathrm{d} x_0}{1-\overline{\alpha}_{t}}\bigg) \notag \\
&\lesssim \frac{\theta_t(x) d(1-\alpha_{t})\log T}{1-\overline{\alpha}_{t}}.
	\label{eq:phix-minus-I-fro-UB}
\end{align}
\end{subequations}
Additionally, the Taylor expansion guarantees that for any $A$ and $\Delta$, 
\begin{subequations}
\begin{align}
\label{eqn:matrix-det}
	\mathsf{det}\big(I+A+\Delta\big) &= 1 + \mathsf{Tr}(A) + O\big((\mathsf{Tr}(A))^2 + \|A\|_{\mathrm{F}}^2 + d\|\Delta\|\big) \\
	\mathsf{det}\big(I+A+\Delta\big)^{-1} &= 1 - \mathsf{Tr}(A) + O\big((\mathsf{Tr}(A))^2 + \|A\|_{\mathrm{F}}^2 + d\|\Delta\|\big)
\end{align}
\end{subequations}
hold as long as  $d\|A\| + d\|\Delta\| \leq c_{11}$ for some small enough constant $c_{11}>0$. 
The above properties taken collectively with \eqref{defn:phit-x} and \eqref{eq:pointwise-epsilon-score-J} allow us to demonstrate that 
\begin{align}
\notag & \frac{p_{\phi_{t}(Y)}(\phi_{t}(x))}{p_{Y}(x)}=\mathsf{det}\Big(\frac{\partial\phi_{t}(x)}{\partial x}\Big)^{-1}
=\left(\mathsf{det}\bigg(\frac{\partial\phi_{t}^{\star}(x)}{\partial x}+\frac{1-\alpha_{t}}{2}\Big[J_{s_{t}}(x)-J_{s_{t}^{\star}}(x)\Big]\bigg)\right)^{-1}\notag\\
	&\quad =\left(\mathsf{det}\bigg(I + \frac{ \partial\phi_{t}^{\star}(x)}{\partial x} - I +\frac{1-\alpha_{t}}{2}\Big[J_{s_{t}}(x)-J_{s_{t}^{\star}}(x)\Big]\bigg)\right)^{-1}\label{eq:det-part-phi-expression}\\
 & \quad=1-\mathsf{Tr}\Big(\frac{\partial\phi_{t}^{\star}(x)}{\partial x}-I\Big)+O\bigg(\theta_t(x)^{2}d^{2}\Big(\frac{1-\alpha_{t}}{\alpha_{t}-\overline{\alpha}_{t}}\Big)^{2}\log^{2}T+\theta^{3}d^{6}\log^{3}T\Big(\frac{1-\alpha_{t}}{\alpha_{t}-\overline{\alpha}_{t}}\Big)^{3}+(1-\alpha_{t})d\varepsilon_{\Jacobi,t}(x)\bigg)\notag\\
 & \quad=1+\frac{d(1-\alpha_{t})}{2(\alpha_{t}-\overline{\alpha}_{t})}+\frac{(1-\alpha_{t})\Big(\big\|\int_{x_{0}}p_{X_{0}\mymid X_{t}}(x_{0}\mymid x)\big(x-\sqrt{\overline{\alpha}_{t}}x_{0}\big)\mathrm{d}x_{0}\big\|_{2}^{2}-\int_{x_{0}}p_{X_{0}\mymid X_{t}}(x_{0}\mymid x)\big\| x-\sqrt{\overline{\alpha}_{t}}x_{0}\big\|_{2}^{2}\mathrm{d}x_{0}\Big)}{2(\alpha_{t}-\overline{\alpha}_{t})(1-\overline{\alpha}_{t})}\notag\\
 & \quad\qquad+O\bigg(\theta_t(x)^{2}d^{2}\Big(\frac{1-\alpha_{t}}{\alpha_{t}-\overline{\alpha}_{t}}\Big)^{2}\log^{2}T+(1-\alpha_{t})d\varepsilon_{\Jacobi,t}(x)\bigg),	 
\end{align}
with the proviso that 
\begin{align*}
\frac{d^{2}(1-\alpha_{t})\log T}{\alpha_{t}-\overline{\alpha}_{t}} & \leq\frac{8c_{1}d^{2}\log^{2}T}{T}\leq c_{12}\qquad\text{and}\qquad(1-\alpha_{t})d\varepsilon_{\Jacobi,t}(x)\leq\frac{c_{1}d\varepsilon_{\Jacobi,t}(x)\log T}{T}\leq c_{12}
\end{align*}
for some sufficiently small constant $c_{12}>0$ (see \eqref{eqn:properties-alpha-proof}).   
%

\subsection{Proof of Lemma~\ref{lem:refine}}
\label{sec:lem-refine}

Before proceeding, let us make note of several basic facts: for any $x$ with $\theta_t(x)\lesssim 1$, 
Lemma~\ref{lem:x0} and \eqref{eqn:properties-alpha-proof}  taken together reveal that: 
%
%
\begin{subequations}
\label{eq:bound-covariance-norm-UB0123-sum}
\begin{align}
 & \left|\frac{1-\alpha_{t}}{(\alpha_{t}-\overline{\alpha}_{t})(1-\overline{\alpha}_{t})}
 \Big(\big\|\mathbb{E}\big[X_{t}-\sqrt{\overline{\alpha}_{t}}X_{0}\mymid X_{t}=y_t\big]\big\|_{2}^{2}-\mathbb{E}\big[\big\| X_{t}-\sqrt{\overline{\alpha}_{t}}X_{0}\big\|_{2}^{2}\mymid X_{t}=y_t\big]\Big)\right| \notag\\
 & \qquad\leq\left|\frac{(1-\alpha_{t})\mathbb{E}\big[\big\| X_{t}-\sqrt{\overline{\alpha}_{t}}X_{0}\big\|_{2}^{2}\mymid X_{t}=y_t\big]}{(\alpha_{t}-\overline{\alpha}_{t})(1-\overline{\alpha}_{t})}\right|\lesssim\frac{(1-\alpha_{t})d\log T}{\alpha_{t}-\overline{\alpha}_{t}}\lesssim\frac{d\log^2 T}{T}=o(1) 
	\label{eq:bound-covariance-norm-UB0123}
\end{align}
\begin{align}
	\text{and}\qquad\frac{d(1-\alpha_{t})}{\alpha_{t}-\overline{\alpha}_{t}}\lesssim\frac{d\log T}{T}=o(1) .
\end{align}
\end{subequations}
Our proof consists of several steps below. 

\paragraph{Step 1: obtaining a refined approximation of $\frac{p_{\sqrt{\alpha_{t}}X_{t-1}}(\phi_{t}(x))}{p_{X_{t}}(x)}$.} 

To begin with, 
recalling the definition of $\mathcal{E}_{c}^{\mathsf{typical}}$ in \eqref{eq:defn-Ec-typical-set}, 
we can repeat the arguments in \eqref{eqn:fei-0} and \eqref{eq:exp-UB-135702} to reach 
\begin{align}
 & \frac{p_{\sqrt{\alpha_{t}}X_{t-1}}\big(\phi_{t}(x)\big)}{p_{X_{t}}(x)}\notag\\
 & =\Big(\frac{1-\overline{\alpha}_{t}}{\alpha_{t}-\overline{\alpha}_{t}}\Big)^{d/2}\int_{x_{0}}p_{X_{0}\mymid X_{t}}(x_{0}\mymid x)\exp\bigg(-\frac{(1-\alpha_{t})\big\| x-\sqrt{\overline{\alpha}_{t}}x_{0}\big\|_{2}^{2}}{2(\alpha_{t}-\overline{\alpha}_{t})(1-\overline{\alpha}_{t})}+\frac{2u^{\top}\big(x-\sqrt{\overline{\alpha}_{t}}x_{0}\big)-\|u\|_{2}^{2}}{2(\alpha_{t}-\overline{\alpha}_{t})}\bigg)\mathrm{d}x_{0}\notag\\
 & =\Big(\frac{1-\overline{\alpha}_{t}}{\alpha_{t}-\overline{\alpha}_{t}}\Big)^{d/2}\Bigg\{\int_{x_{0}\in\mathcal{E}_{2}^{\mathsf{typical}}}p_{X_{0}\mymid X_{t}}(x_{0}\mymid x)\exp\bigg(-\frac{(1-\alpha_{t})\big\| x-\sqrt{\overline{\alpha}_{t}}x_{0}\big\|_{2}^{2}}{2(\alpha_{t}-\overline{\alpha}_{t})(1-\overline{\alpha}_{t})}+\frac{2u^{\top}\big(x-\sqrt{\overline{\alpha}_{t}}x_{0}\big)-\|u\|_{2}^{2}}{2(\alpha_{t}-\overline{\alpha}_{t})}\bigg)\mathrm{d}x_{0}\notag\\
 & \qquad+O\Big(\exp\left(-\theta_{t}(x)d\log T\right)\Big)\Bigg\}\notag\\
 & =\Big(\frac{1-\overline{\alpha}_{t}}{\alpha_{t}-\overline{\alpha}_{t}}\Big)^{d/2}\Bigg\{\int_{x_{0}\in\mathcal{E}_{2}^{\mathsf{typical}}}p_{X_{0}\mymid X_{t}}(x_{0}\mymid x)\exp\bigg(\frac{(1-\alpha_{t})\big[\big(x-\sqrt{\overline{\alpha}_{t}}x_{0}\big)^{\top}\mathbb{E}\big[x-\sqrt{\overline{\alpha}_{t}}X_{0}\mid X_{t}=x\big]-\big\| x-\sqrt{\overline{\alpha}_{t}}x_{0}\big\|_{2}^{2}\big]}{2(\alpha_{t}-\overline{\alpha}_{t})(1-\overline{\alpha}_{t})}\bigg)\notag\\
 & \qquad\qquad\cdot\exp\bigg(-\frac{(1-\alpha_{t})\big(s_{t}(x)-s_{t}^{\star}(x)\big)^{\top}\big(x-\sqrt{\overline{\alpha}_{t}}x_{0}\big)}{2(\alpha_{t}-\overline{\alpha}_{t})}-\frac{\|u\|_{2}^{2}}{2(\alpha_{t}-\overline{\alpha}_{t})}\bigg)\mathrm{d}x_{0}+O\Big(\exp\left(-\theta_{t}(x)d\log T\right)\Big)\Bigg\}\notag\\
 & =O\Big(\exp\left(-\theta_{t}(x)d\log T\right)\Big)+\Bigg(1+O\bigg(\frac{d\log^{2}T}{T^{2}}\bigg)\Bigg)\cdot\notag\\
 & \ \ \int_{x_{0}\in\mathcal{E}_{2}^{\mathsf{typical}}}p_{X_{0}\mymid X_{t}}(x_{0}\mymid x)\exp\bigg(\frac{d(1-\alpha_{t})}{2(\alpha_{t}-\overline{\alpha}_{t})}+\frac{(1-\alpha_{t})\big[\big(x-\sqrt{\overline{\alpha}_{t}}x_{0}\big)^{\top}\mathbb{E}\big[x-\sqrt{\overline{\alpha}_{t}}X_{0}\mid X_{t}=x\big]-\big\| x-\sqrt{\overline{\alpha}_{t}}x_{0}\big\|_{2}^{2}\big]}{2(\alpha_{t}-\overline{\alpha}_{t})(1-\overline{\alpha}_{t})}\bigg)\notag\\
 & \qquad\qquad\cdot\exp\bigg(-\frac{(1-\alpha_{t})\big(s_{t}(x)-s_{t}^{\star}(x)\big)^{\top}\big(x-\sqrt{\overline{\alpha}_{t}}x_{0}\big)}{2(\alpha_{t}-\overline{\alpha}_{t})}-\frac{\|u\|_{2}^{2}}{2(\alpha_{t}-\overline{\alpha}_{t})}\bigg)\mathrm{d}x_{0},
	\label{eq:ratio-pX-phit-equiv-refined-1}
\end{align}
where we remind the reader of the definition of $u$ in \eqref{eq:defn-u-Lemma-main-ODE}.  
Here, the last line in \eqref{eq:ratio-pX-phit-equiv-refined-1} follows since 
\[
\Big(\frac{1-\overline{\alpha}_{t}}{\alpha_{t}-\overline{\alpha}_{t}}\Big)^{d/2}
=\Bigg(1+O\bigg(\frac{d\log^{2}T}{T^{2}}\bigg)\Bigg)\exp\Big(\frac{d(1-\alpha_{t})}{2(\alpha_{t}-\overline{\alpha}_{t})}\Big) \asymp 1,
\]
a consequence of the property \eqref{eq:expansion-ratio-3-alpha} and the fact 
 $\frac{1-\alpha_{t}}{\alpha_{t}-\overline{\alpha}_{t}}\lesssim \frac{\log T}{T}$ (cf.~\eqref{eqn:properties-alpha-proof-1}). 
 %
%
%

Moreover, following the arguments in \eqref{eq:long-UB-123}, we can easily
derive that: for any $x_{0}\in\mathcal{E}_{2}^{\mathsf{typical}}$, 
\begin{align}
 & \exp\bigg(-\frac{(1-\alpha_{t})\big(s_{t}(x)-s_{t}^{\star}(x)\big)^{\top}\big(x-\sqrt{\overline{\alpha}_{t}}x_{0}\big)}{2(\alpha_{t}-\overline{\alpha}_{t})}-\frac{\|u\|_{2}^{2}}{2(\alpha_{t}-\overline{\alpha}_{t})}\bigg) \notag\\
 & =1+O\left(\frac{(1-\alpha_{t})\|s_{t}(x)-s_{t}^{\star}(x)\|_{2}\|x-\sqrt{\overline{\alpha}_{t}}x_{0}\|_{2}}{\alpha_{t}-\overline{\alpha}_{t}}+\frac{\|u\|_{2}^{2}}{\alpha_{t}-\overline{\alpha}_{t}}\right) \notag\\
 & =1+O\Bigg(\frac{\varepsilon_{\mathsf{score},t}(x)\sqrt{\theta_{t}(x)d\log^{3}T}}{T}+\frac{\varepsilon_{\mathsf{score},t}(x)^{2}\log^{2}T}{T^{2}}+\frac{\theta_{t}(x)d\log^{3}T}{T^{2}}\Bigg) \notag\\
 & =1+O\left(\frac{\varepsilon_{\mathsf{score},t}(x)\sqrt{d\log^{3}T}}{T}+\frac{d\log^{3}T}{T^{2}}\right),
	\label{eq:exp-usquare-approx}
\end{align}
where the last line invokes the assumptions $\theta_{t}(x)\lesssim1$
and $\frac{\varepsilon_{\mathsf{score},t}(x)\log^{3/2}T}{T}\lesssim\sqrt{\theta_{t}(x)d}\lesssim \sqrt{d}$.  
With \eqref{eq:ratio-pX-phit-equiv-refined-1} and \eqref{eq:exp-usquare-approx} in place, we obtain
%
\begin{align}
 & \frac{p_{\sqrt{\alpha_{t}}X_{t-1}}\big(\phi_{t}(x)\big)}{p_{X_{t}}(x)}=O\Big(\exp\big(-\theta_{t}(x)d\log T\big)\Big)+\Bigg(1+O\bigg(\frac{d\log^{3}T}{T^{2}}+\frac{\varepsilon_{\mathsf{score},t}(x)\sqrt{d\log^{3}T}}{T}\bigg)\Bigg)\cdot\notag\\
 & \ \ \int_{x_{0}\in\mathcal{E}_{2}^{\mathsf{typical}}}p_{X_{0}\mymid X_{t}}(x_{0}\mymid x)\exp\bigg(\frac{d(1-\alpha_{t})}{2(\alpha_{t}-\overline{\alpha}_{t})}+\frac{(1-\alpha_{t})\big[\big(x-\sqrt{\overline{\alpha}_{t}}x_{0}\big)^{\top}\mathbb{E}\big[x-\sqrt{\overline{\alpha}_{t}}X_{0}\mid X_{t}=x\big]-\big\| x-\sqrt{\overline{\alpha}_{t}}x_{0}\big\|_{2}^{2}\big]}{2(\alpha_{t}-\overline{\alpha}_{t})(1-\overline{\alpha}_{t})}\bigg)\mathrm{d}x_{0}\notag\\
 & =O\bigg(\frac{d\log^{3}T}{T^{2}}+\frac{\varepsilon_{\mathsf{score},t}(x)\sqrt{d\log^{3}T}}{T}\bigg)+\notag\\
 & \ \ \int_{x_{0}\in\mathcal{E}_{2}^{\mathsf{typical}}}p_{X_{0}\mymid X_{t}}(x_{0}\mymid x)\exp\bigg(\frac{d(1-\alpha_{t})}{2(\alpha_{t}-\overline{\alpha}_{t})}+\frac{(1-\alpha_{t})\big[\big(x-\sqrt{\overline{\alpha}_{t}}x_{0}\big)^{\top}\mathbb{E}\big[x-\sqrt{\overline{\alpha}_{t}}X_{0}\mid X_{t}=x\big]-\big\| x-\sqrt{\overline{\alpha}_{t}}x_{0}\big\|_{2}^{2}\big]}{2(\alpha_{t}-\overline{\alpha}_{t})(1-\overline{\alpha}_{t})}\bigg)\mathrm{d}x_{0}\notag\\
 & =O\bigg(\frac{d\log^{3}T}{T^{2}}+\frac{\varepsilon_{\mathsf{score},t}(x)\sqrt{d\log^{3}T}}{T}\bigg)+\notag\\
 & \ \ \int p_{X_{0}\mymid X_{t}}(x_{0}\mymid x)\exp\bigg(\frac{d(1-\alpha_{t})}{2(\alpha_{t}-\overline{\alpha}_{t})}+\frac{(1-\alpha_{t})\big[\big(x-\sqrt{\overline{\alpha}_{t}}x_{0}\big)^{\top}\mathbb{E}\big[x-\sqrt{\overline{\alpha}_{t}}X_{0}\mid X_{t}=x\big]-\big\| x-\sqrt{\overline{\alpha}_{t}}x_{0}\big\|_{2}^{2}\big]}{2(\alpha_{t}-\overline{\alpha}_{t})(1-\overline{\alpha}_{t})}\bigg)\mathrm{d}x_{0},
	\label{eq:refined-p-ratio-X0-Xt-135}
\end{align}
where the penultimate inequality invokes \eqref{eq:bound-covariance-norm-UB0123-sum} (so that the integral above is at most $O(1)$), 
and the last inequality repeats the arguments in \eqref{eq:exp-UB-135702} once again to demonstrate that
\begin{align*}
 & \int_{x_{0}\notin\mathcal{E}_{2}^{\mathsf{typical}}}p_{X_{0}\mymid X_{t}}(x_{0}\mymid x)\exp\bigg(\frac{d(1-\alpha_{t})}{2(\alpha_{t}-\overline{\alpha}_{t})}+\frac{(1-\alpha_{t})\big[\big(x-\sqrt{\overline{\alpha}_{t}}x_{0}\big)^{\top}\mathbb{E}\big[x-\sqrt{\overline{\alpha}_{t}}X_{0}\mid X_{t}=x\big]-\big\| x-\sqrt{\overline{\alpha}_{t}}x_{0}\big\|_{2}^{2}\big]}{2(\alpha_{t}-\overline{\alpha}_{t})(1-\overline{\alpha}_{t})}\bigg)\mathrm{d}x_{0}\\
 & \qquad\qquad\lesssim\exp\big(-\theta_{t}(x)d\log T\big)\lesssim\frac{d\log^{3}T}{T^{2}}.
\end{align*}
%

%

\paragraph{Step 2: obtaining a refined approximation on $\frac{p_{\phi_{t}(Y)}(\phi_{t}(x))}{p_{Y}(x)}$.}

For any matrix $\Delta\in\mathbb{R}^{d\times d}$ and any symmetric
matrix $A\in\mathbb{R}^{d\times d}$ obeying $\|A\|<1/2$ and $\|\Delta\|<1/2$,
elementary linear algebra (e.g., Weyl's inequality) tells us that
\begin{align*}
\sum_{i=1}^{d}\sigma_{i}(I+A+\Delta) & =\sum_{i=1}^{d}\big(\sigma_{i}(I+A)+O(\|\Delta\|)\big)=\sum_{i=1}^{d}\lambda_{i}(I+A)+O(d\|\Delta\|)=d+\mathsf{Tr}(A)+O(d\|\Delta\|),\\
\sum_{i=1}^{d}\big(\sigma_{i}(I+A)-1\big)^{2} & =\sum_{i=1}^{d}\big(\lambda_{i}(I+A)-1\big)^{2}=\sum_{i=1}^{d}\big(\lambda_{i}(A)\big)^{2}=\|A\|_{\mathrm{F}}^{2},\\
\sum_{i=1}^{d}\big(\sigma_{i}(I+A+\Delta)-1\big)^{2} & \leq2\sum_{i=1}^{d}\big(\sigma_{i}(I+A)-1\big)^{2}+2\sum_{i=1}^{d}\big(\sigma_{i}(I+A+\Delta)-\sigma_{i}(I+A)\big)^{2}\\
 & \leq2\|A\|_{\mathrm{F}}^{2}+2d\|\Delta\|^{2},
\end{align*}
with $\sigma_{i}(Z)$ (resp.~$\lambda_{i}(Z)$) representing the
$i$-th largest singular value (resp.~eigenvalue) of a matrix $Z$.
These properties in turn allow one to derive
\begin{align*}
\log\big|\det(I+A+\Delta)\big| & =\sum_{i=1}^{d}\log\big(\sigma_{i}(I+A+\Delta)\big)=\sum_{i=1}^{d}\big(\sigma_{i}(I+A+\Delta)-1\big)+O\bigg(\sum_{i=1}^{d}\big(\sigma_{i}(I+A+\Delta)-1\big)^{2}\bigg)\\
 & =\mathsf{Tr}(A)+O\big(d\|\Delta\|+\|A\|_{\mathrm{F}}^{2}+d\|\Delta\|^{2}\big)=\mathsf{Tr}(A)+O\big(\|A\|_{\mathrm{F}}^{2}+d\|\Delta\|\big).
\end{align*}

With this approximation for the log-determinant function in mind, we can invoke \eqref{eq:det-part-phi-expression} to obtain  
\begin{align*}
\notag & \log\frac{p_{\phi_{t}(Y)}(\phi_{t}(x))}{p_{Y}(x)}=-\log\bigg|\mathsf{det}\bigg(I_{d}+\frac{\partial\phi_{t}^{\star}(x)}{\partial x}-I+\frac{1-\alpha_{t}}{2}\Big[J_{s_{t}}(x)-J_{s_{t}^{\star}}(x)\Big]\bigg)\bigg|\notag\\
 & =-\mathsf{Tr}\bigg(\frac{\partial\phi_{t}^{\star}(x)}{\partial x}-I\bigg)+O\bigg(\bigg\|\frac{\partial\phi_{t}^{\star}(x)}{\partial x}-I\bigg\|_{\mathrm{F}}^{2}+d(1-\alpha_{t})\varepsilon_{\mathsf{Jacobi},t}(x)\bigg)\\
 & =O\bigg(\bigg\|\frac{\partial\phi_{t}^{\star}(x)}{\partial x}-I\bigg\|_{\mathrm{F}}^{2}+d(1-\alpha_{t})\varepsilon_{\mathsf{Jacobi},t}(x)\bigg)+\left(1+O\bigg(\frac{\log T}{T}\bigg)\right)\\
 & \quad\cdot\Bigg\{\frac{d(1-\alpha_{t})}{2(\alpha_{t}-\overline{\alpha}_{t})}+\frac{(1-\alpha_{t})\Big(\big\|\int_{x_{0}}p_{X_{0}\mymid X_{t}}(x_{0}\mymid x)\big(x-\sqrt{\overline{\alpha}_{t}}x_{0}\big)\mathrm{d}x_{0}\big\|_{2}^{2}-\int_{x_{0}}p_{X_{0}\mymid X_{t}}(x_{0}\mymid x)\big\| x-\sqrt{\overline{\alpha}_{t}}x_{0}\big\|_{2}^{2}\mathrm{d}x_{0}\Big)}{2(\alpha_{t}-\overline{\alpha}_{t})(1-\overline{\alpha}_{t})}\Bigg\}\notag\\
 & =\frac{d(1-\alpha_{t})}{2(\alpha_{t}-\overline{\alpha}_{t})}+\frac{(1-\alpha_{t})\Big(\big\|\int_{x_{0}}p_{X_{0}\mymid X_{t}}(x_{0}\mymid x)\big(x-\sqrt{\overline{\alpha}_{t}}x_{0}\big)\mathrm{d}x_{0}\big\|_{2}^{2}-\int_{x_{0}}p_{X_{0}\mymid X_{t}}(x_{0}\mymid x)\big\| x-\sqrt{\overline{\alpha}_{t}}x_{0}\big\|_{2}^{2}\mathrm{d}x_{0}\Big)}{2(\alpha_{t}-\overline{\alpha}_{t})(1-\overline{\alpha}_{t})}\\
 & \quad\quad+O\bigg(\bigg\|\frac{\partial\phi_{t}^{\star}(x)}{\partial x}-I\bigg\|_{\mathrm{F}}^{2}+d(1-\alpha_{t})\varepsilon_{\mathsf{Jacobi},t}(x)+\frac{d\log^{2}T}{T^{2}}\bigg),\notag
\end{align*}
where the penultimate relation arises from \eqref{eq:defn-B-intermediate-2} and the following fact (which uses \eqref{eqn:properties-alpha-proof-1})
\[
\Bigg|\frac{\frac{1}{\alpha_{t}-\overline{\alpha}_{t}}-\frac{1}{1-\overline{\alpha}_{t}}}{\frac{1}{1-\overline{\alpha}_{t}}}\Bigg|=\frac{\,\frac{1-\alpha_{t}}{(\alpha_{t}-\overline{\alpha}_{t})(1-\overline{\alpha}_{t})}\,}{\frac{1}{1-\overline{\alpha}_{t}}}=\frac{1-\alpha_{t}}{\alpha_{t}-\overline{\alpha}_{t}}=O\bigg(\frac{\log T}{T}\bigg),
\]
and the last relation applies Lemma~\ref{lem:x0} (under the assumption $\theta_t(x)\lesssim 1$) and \eqref{eqn:properties-alpha-proof}.  It is then easily seen that
\begin{align}
 & \frac{p_{\phi_{t}(Y)}(\phi_{t}(x))}{p_{Y}(x)}=\exp\Bigg(O\bigg(\bigg\|\frac{\partial\phi_{t}^{\star}(x)}{\partial x}-I\bigg\|_{\mathrm{F}}^{2}+d(1-\alpha_{t})\varepsilon_{\mathsf{Jacobi},t}(x)+\frac{d\log^{2}T}{T^{2}}\bigg)\Bigg)\nonumber\\
 & \ \cdot\exp\Bigg(\frac{d(1-\alpha_{t})}{2(\alpha_{t}-\overline{\alpha}_{t})}+\frac{(1-\alpha_{t})\Big(\big\|\int_{x_{0}}p_{X_{0}\mymid X_{t}}(x_{0}\mymid x)\big(x-\sqrt{\overline{\alpha}_{t}}x_{0}\big)\mathrm{d}x_{0}\big\|_{2}^{2}-\int_{x_{0}}p_{X_{0}\mymid X_{t}}(x_{0}\mymid x)\big\| x-\sqrt{\overline{\alpha}_{t}}x_{0}\big\|_{2}^{2}\mathrm{d}x_{0}\Big)}{2(\alpha_{t}-\overline{\alpha}_{t})(1-\overline{\alpha}_{t})}\Bigg) \notag\\
 & =O\bigg(\bigg\|\frac{\partial\phi_{t}^{\star}(x)}{\partial x}-I\bigg\|_{\mathrm{F}}^{2}+d(1-\alpha_{t})\varepsilon_{\mathsf{Jacobi},t}(x)+\frac{d\log^{2}T}{T^{2}}\bigg) \notag\\
 & +\exp\Bigg(\frac{d(1-\alpha_{t})}{2(\alpha_{t}-\overline{\alpha}_{t})}+\frac{(1-\alpha_{t})\Big(\big\|\int_{x_{0}}p_{X_{0}\mymid X_{t}}(x_{0}\mymid x)\big(x-\sqrt{\overline{\alpha}_{t}}x_{0}\big)\mathrm{d}x_{0}\big\|_{2}^{2}-\int_{x_{0}}p_{X_{0}\mymid X_{t}}(x_{0}\mymid x)\big\| x-\sqrt{\overline{\alpha}_{t}}x_{0}\big\|_{2}^{2}\mathrm{d}x_{0}\Big)}{2(\alpha_{t}-\overline{\alpha}_{t})(1-\overline{\alpha}_{t})}\Bigg). 
	\label{eq:p-phit-py-ratio-refined}
\end{align}
%


\paragraph{Step 3: computing the density ratio of interest.} 
From relations \eqref{eq:xt} and \eqref{eq:yt} in Lemma~\ref{lem:main-ODE} as well as \eqref{eq:bound-covariance-norm-UB0123-sum}, we see that
\[
	\frac{p_{\phi_{t}(Y_{t})}(\phi_{t}(x))}{p_{Y_{t}}(x)}=1+o(1) \qquad \text{and} \qquad
	\frac{p_{\sqrt{\alpha_{t}}X_{t-1}}\big(\phi_{t}(x)\big)}{p_{X_{t}}(x)}=1+o(1).
\]
Then, compare the preceding two results \eqref{eq:refined-p-ratio-X0-Xt-135} and \eqref{eq:p-phit-py-ratio-refined} 
(with $Y$ chosen to be $Y_t$) to arrive at
\begin{align*}
 & \frac{p_{\phi_{t}(Y_{t})}(\phi_{t}(x))}{p_{Y_{t}}(x)}/\frac{p_{\sqrt{\alpha_{t}}X_{t-1}}\big(\phi_{t}(x)\big)}{p_{X_{t}}(x)}\\
 & \qquad =\frac{g_{1}(x)}{g_{2}(x)}+O\bigg(\bigg\|\frac{\partial\phi_{t}^{\star}(x)}{\partial x}-I\bigg\|_{\mathrm{F}}^{2}+\frac{\varepsilon_{\mathsf{score},t}(x)\sqrt{d\log^{3}T}}{T}+d(1-\alpha_{t})\varepsilon_{\mathsf{Jacobi},t}(x)+\frac{d\log^{3}T}{T^{2}}\bigg),
\end{align*}
where the two functions $g_{1}(\cdot)$ and $g_2(\cdot)$ are defined as
\begin{align*}
g_{1}(x) & \coloneqq\exp\Bigg(\frac{(1-\alpha_{t})\Big(\big\|\int_{x_{0}}p_{X_{0}\mymid X_{t}}(x_{0}\mymid x)\big(x-\sqrt{\overline{\alpha}_{t}}x_{0}\big)\mathrm{d}x_{0}\big\|_{2}^{2}-\int_{x_{0}}p_{X_{0}\mymid X_{t}}(x_{0}\mymid x)\big\| x-\sqrt{\overline{\alpha}_{t}}x_{0}\big\|_{2}^{2}\mathrm{d}x_{0}\Big)}{2(\alpha_{t}-\overline{\alpha}_{t})(1-\overline{\alpha}_{t})}\Bigg),\\
g_{2}(x) & \coloneqq\int p_{X_{0}\mymid X_{t}}(x_{0}\mymid x)\exp\bigg(\frac{(1-\alpha_{t})\big[\big(x-\sqrt{\overline{\alpha}_{t}}x_{0}\big)^{\top}\mathbb{E}\big[x-\sqrt{\overline{\alpha}_{t}}X_{0}\mid X_{t}=x\big]-\big\| x-\sqrt{\overline{\alpha}_{t}}x_{0}\big\|_{2}^{2}\big]}{2(\alpha_{t}-\overline{\alpha}_{t})(1-\overline{\alpha}_{t})}\bigg)\mathrm{d}x_{0}.
\end{align*}
Jensen's inequality tells us that $g_1(x)\leq g_2(x)$, and hence we can write 
\begin{align*}
 & \frac{p_{\phi_{t}(Y_{t})}(\phi_{t}(x))}{p_{Y_{t}}(x)}/\frac{p_{\sqrt{\alpha_{t}}X_{t-1}}\big(\phi_{t}(x)\big)}{p_{X_{t}}(x)}\\
 & \qquad=1+\zeta_{t}(x)+O\bigg(\bigg\|\frac{\partial\phi_{t}^{\star}(x)}{\partial x}-I\bigg\|_{\mathrm{F}}^{2}+\frac{\varepsilon_{\mathsf{score},t}(x)\sqrt{d\log^{3}T}}{T}+\frac{d \log T \varepsilon_{\mathsf{Jacobi},t}(x)}{T}+\frac{d\log^{3}T}{T^{2}}\bigg)
\end{align*}
for those $x$ obeying the assumptions of this lemma, where $\zeta_{t}(\cdot) = g_{1}(\cdot)/g_{2}(\cdot) - 1$ is some function obeying $\zeta_{t}(x)\leq0$.

Similarly, replacing $\phi_t$ (resp.~$Y_t$) with $\phi_t^{\star}$ (resp.~$X_t$) in the above display and repeating the same arguments, we arrive at
\begin{align}
\frac{p_{\phi_{t}^{\star}(X_{t})}(\phi_{t}^{\star}(x))}{p_{\sqrt{\alpha_{t}}X_{t-1}}\big(\phi_{t}^{\star}(x)\big)} & =\frac{p_{\phi_{t}^{\star}(X_{t})}(\phi_{t}^{\star}(x))}{p_{X_{t}}(x)}/\frac{p_{\sqrt{\alpha_{t}}X_{t-1}}\big(\phi_{t}^{\star}(x)\big)}{p_{X_{t}}(x)}=1+\zeta_{t}(x)+O\bigg(\bigg\|\frac{\partial\phi_{t}^{\star}(x)}{\partial x}-I\bigg\|_{\mathrm{F}}^{2}+\frac{d\log^{3}T}{T^{2}}\bigg).
\label{eq:ratio-of-interest-refined-2}
\end{align}
%
%
%
%
%
%
%
%
%

The careful reader would immediately note that we have not yet defined $\zeta_{t}(\cdot)$ for  all $x$. 
To ease presentation, 
we shall simply take 
$\zeta_{t}(x) =0$ for any $x$ that does not satisfy the assumptions of this lemma. 
%

%
%

%
%

\paragraph{Step 4: bounding the expectation of $\zeta_t(\cdot)$.}  
Define the set 
\[
	\mathcal{E}_{\zeta}\coloneqq\left\{ x\mid\theta_{t}(x)\leq2C_{12},\frac{C_{10}\theta_{t}(x)d\log^{2}T}{T}\leq1\right\} 
\]
for some large enough constant $C_{12}>0$. 
With \eqref{eq:ratio-of-interest-refined-2} in place, we have 
\begin{align*}
	&p_{\sqrt{\alpha_{t}}X_{t-1}}\big(\phi^{\star}_{t}(x)\big)\zeta_{t}(x)
	= p_{\phi_{t}(X_{t})}\big(\phi^{\star}_{t}(x)\big)- p_{\sqrt{\alpha_{t}}X_{t-1}}\big(\phi^{\star}_{t}(x)\big) 
	- p_{\sqrt{\alpha_{t}}X_{t-1}}\big(\phi^{\star}_{t}(x)\big)O\bigg(\bigg\|\frac{\partial\phi_{t}^{\star}(x)}{\partial x}-I\bigg\|_{\mathrm{F}}^{2}+\frac{d\log^{3}T}{T^{2}}\bigg)
\end{align*}
for any $x\in \mathcal{E}_{\zeta}$.  
In addition, according to the properties \eqref{eq:det-part-phi-expression}, \eqref{eq:bound-covariance-norm-UB0123}, Lemma~\ref{lem:x0}, and the assumption that $T \gtrsim d^2\log^5T$, one can easily derive
\begin{subequations}
	\label{eq:det-UB-LB-135}
\begin{align}
\bigg|\mathsf{det}\bigg(\frac{\partial\phi_{t}^{\star}(x)}{\partial x}\bigg)\bigg| & =\mathsf{det}\Bigg(\bigg(1-\frac{1-\alpha_{t}}{2(1-\overline{\alpha}_{t})}\bigg)I_{d}+\frac{1-\alpha_{t}}{2(1-\overline{\alpha}_{t})}\mathsf{Cov}\bigg(\frac{X_{t}-\sqrt{\overline{\alpha}_{t}}X_{0}}{\sqrt{1-\overline{\alpha}_{t}}}\mid X_{t}=x\bigg)\Bigg)\notag\\
 & \leq\bigg(1+O\Big(\frac{d\log^{2}T}{T}\Big)\bigg)^{d}=1+O\Big(\frac{d^{2}\log^{2}T}{T}\Big)\leq2 \\
%
	\text{and}\qquad 
	\bigg|\mathsf{det}\bigg(\frac{\partial\phi_{t}^{\star}(x)}{\partial x}\bigg)\bigg| 
 & \geq\bigg(1-\frac{1-\alpha_{t}}{2(1-\overline{\alpha}_{t})}\bigg)^{d}=\bigg(1-O\Big(\frac{\log T}{T}\Big)\bigg)^{d}\geq\frac{1}{2} 
\end{align}
\end{subequations}
for any $x\in \mathcal{E}_{\zeta}$. 
These properties in turn allow one to derive 
\begin{align}
0 & \leq 
-{\displaystyle \int}p_{\sqrt{\alpha_{t}}X_{t-1}}\big(\phi_{t}^{\star}(x)\big)\zeta_{t}(x)\mathrm{d}x =
	-{\displaystyle \int}_{x\in\mathcal{E}_{\zeta}}p_{\sqrt{\alpha_{t}}X_{t-1}}\big(\phi_{t}^{\star}(x)\big)\zeta_{t}(x)\mathrm{d}x\nonumber\\
 & \asymp-{\displaystyle \int}_{x\in\mathcal{E}_{\zeta}}p_{\sqrt{\alpha_{t}}X_{t-1}}\big(\phi_{t}^{\star}(x)\big)\bigg|\det\Big(\frac{\partial\phi_{t}^{\star}(x)}{\partial x}\Big)\bigg|\zeta_{t}(x)\mathrm{d}x\nonumber\\
 & =-{\displaystyle \int}_{x\in\mathcal{E}_{\zeta}}p_{\phi_{t}^{\star}(X_{t})}\big(\phi_{t}^{\star}(x)\big)\bigg|\det\Big(\frac{\partial\phi_{t}^{\star}(x)}{\partial x}\Big)\bigg|\mathrm{d}x+{\displaystyle \int}_{x\in\mathcal{E}_{\zeta}}p_{\sqrt{\alpha_{t}}X_{t-1}}\big(\phi_{t}^{\star}(x)\big)\bigg|\det\Big(\frac{\partial\phi_{t}^{\star}(x)}{\partial x}\Big)\bigg|\mathrm{d}x\notag\\
 & \qquad+{\displaystyle \int}_{x\in\mathcal{E}_{\zeta}}p_{\sqrt{\alpha_{t}}X_{t-1}}\big(\phi_{t}^{\star}(x)\big)\bigg|\det\Big(\frac{\partial\phi_{t}^{\star}(x)}{\partial x}\Big)\bigg|O\bigg(\bigg\|\frac{\partial\phi_{t}^{\star}(x)}{\partial x}-I\bigg\|_{\mathrm{F}}^{2}+\frac{d\log^{3}T}{T^{2}}\bigg)\mathrm{d}x\nonumber\\
 & \leq-{\displaystyle \int}_{x\in\mathcal{E}_{\zeta}}p_{\phi_{t}^{\star}(X_{t})}\big(\phi_{t}^{\star}(x)\big)\bigg|\det\Big(\frac{\partial\phi_{t}^{\star}(x)}{\partial x}\Big)\bigg|\mathrm{d}x+1+{\displaystyle \int}p_{\sqrt{\alpha_{t}}X_{t-1}}\big(\phi_{t}^{\star}(x)\big)O\bigg(\bigg\|\frac{\partial\phi_{t}^{\star}(x)}{\partial x}-I\bigg\|_{\mathrm{F}}^{2}\bigg)\mathrm{d}x+
	O\bigg( \frac{d\log^{3}T}{T^{2}} \bigg),
	\label{eq:int-p-Xt-1-135}
\end{align}
where the last line is valid since
\begin{align*}
{\displaystyle \int}_{x\in\mathcal{E}_{\zeta}}p_{\sqrt{\alpha_{t}}X_{t-1}}\big(\phi_{t}^{\star}(x)\big)\bigg|\det\Big(\frac{\partial\phi_{t}^{\star}(x)}{\partial x}\Big)\bigg|\mathrm{d}x & \leq{\displaystyle \int}p_{\sqrt{\alpha_{t}}X_{t-1}}\big(\phi_{t}^{\star}(x)\big)\bigg|\det\Big(\frac{\partial\phi_{t}^{\star}(x)}{\partial x}\Big)\bigg|\mathrm{d}x={\displaystyle \int}p_{\sqrt{\alpha_{t}}X_{t-1}}(x)\mathrm{d}x=1.
\end{align*}
%

%
%

It then boils down to evaluating ${\int}_{x\in\mathcal{E}_{\zeta}}p_{\phi_{t}^{\star}(X_{t})}\big(\phi_{t}^{\star}(x)\big)\big|\det\frac{\partial\phi_{t}^{\star}(x)}{\partial x}\big|\mathrm{d}x$. 
Towards this end, we make the observation that
\begin{align}
{\displaystyle \int}_{x\in\mathcal{E}_{\zeta}}p_{\phi_{t}^{\star}(X_{t})}\big(\phi_{t}^{\star}(x)\big)\bigg|\det\Big(\frac{\partial\phi_{t}^{\star}(x)}{\partial x}\Big)\bigg|\mathrm{d}x & ={\displaystyle \int}_{x\in\mathcal{E}_{\zeta}}p_{X_{t}}(x)\bigg|\mathsf{det}\bigg(\frac{\partial\phi_{t}^{\star}(x)}{\partial x}\bigg)^{-1}\bigg|\,\bigg|\det\Big(\frac{\partial\phi_{t}^{\star}(x)}{\partial x}\Big)\bigg|\mathrm{d}x\notag\\
 & ={\displaystyle \int}_{x\in\mathcal{E}_{\zeta}}p_{X_{t}}(x)\mathrm{d}x=1-{\displaystyle \int}_{x\notin\mathcal{E}_{\zeta}}p_{X_{t}}(x)\mathrm{d}x. 
	\label{eq:int-p-phit-Ezeta-out}
\end{align}
%
%
Moreover, it is seen from \eqref{eq:Xt-2range-ODE} that
\begin{align*}
\mathbb{P}\left(\|X_{t}\|_{2}>T^{c_{R}+2}\right)\leq\exp(-c_6d\log T),
\end{align*}
thereby allowing us to derive that
\begin{align*}
{\displaystyle \int}_{x\notin\mathcal{E}_{\zeta}}p_{X_{t}}(x)\mathrm{d}x & \leq{\displaystyle \int}_{x:\,\theta_{t}(x)\leq C_{12},\|x\|_{2}\leq T^{c_{R}+2}}p_{X_{t}}(x)\mathrm{d}x+{\displaystyle \int}_{\|x\|_{2}>T^{c_{R}+2}}p_{X_{t}}(x)\mathrm{d}x\\
 & \leq(T^{c_{R}+2})^{d}\exp\left(-2C_{12}d\log T\right)+\exp(-c_6 d\log T)\\
	& \leq 2\exp\big(- \min\{C_{12},c_6\} d\log T\big).
\end{align*}
Combine this with \eqref{eq:int-p-phit-Ezeta-out} to reach
\begin{align*}
{\displaystyle \int}_{x\in\mathcal{E}_{\zeta}}p_{\phi_{t}^{\star}(X_{t})}\big(\phi_{t}^{\star}(x)\big)\bigg|\det\Big(\frac{\partial\phi_{t}^{\star}(x)}{\partial x}\Big)\bigg|\mathrm{d}x	
	& =1-O\big(\exp\left(- \min\{C_{12},c_6\} d\log T\right)\big). 
\end{align*}
Substitution into \eqref{eq:int-p-Xt-1-135} then gives
\begin{align}
0 & \leq-{\displaystyle \int}p_{\sqrt{\alpha_{t}}X_{t-1}}\big(\phi_{t}^{\star}(x)\big)\zeta_{t}(x)\mathrm{d}x \notag\\
 & \leq-1+O\big(\exp\left(-\min\{C_{12},c_6\}d\log T\right)\big)+1+{\displaystyle \int}p_{\sqrt{\alpha_{t}}X_{t-1}}\big(\phi_{t}^{\star}(x)\big)O\bigg(\bigg\|\frac{\partial\phi_{t}^{\star}(x)}{\partial x}-I\bigg\|_{\mathrm{F}}^{2}\bigg)\mathrm{d}x+ 
	O\bigg(\frac{d\log^{3}T}{T^{2}} \bigg) \notag\\
 & \asymp{\displaystyle \int}p_{\sqrt{\alpha_{t}}X_{t-1}}\big(\phi_{t}^{\star}(x)\big)\bigg\|\frac{\partial\phi_{t}^{\star}(x)}{\partial x}-I\bigg\|_{\mathrm{F}}^{2}\mathrm{d}x+\frac{d\log^{3}T}{T^{2}}. 
	\label{eq:int-p-Xt-1-zeta-UB}
\end{align}
%

To finish up, note that Lemma~\ref{lem:main-ODE} together with Lemma~\ref{lem:x0} and properties~\eqref{eqn:properties-alpha-proof} tells us that, for any $x\in \mathcal{E}_{\zeta}$, 
\[
	p_{X_{t}}(x) \asymp p_{\sqrt{\alpha_{t}}X_{t-1}}\big(\phi_{t}^{\star}(x)\big) .
\]
This taken collectively with \eqref{eq:int-p-Xt-1-zeta-UB} leads to the advertised result
\begin{align*}
0 &\le -\int p_{X_{t}}(x)\zeta_t(x)\mathrm{d}x 
	= -\int_{x\in \mathcal{E}_{\zeta}} p_{X_{t}}(x)\zeta_t(x)\mathrm{d}x
	\asymp-\int_{x\in \mathcal{E}_{\zeta}} p_{\sqrt{\alpha_{t}}X_{t-1}}\big(\phi_{t}^{\star}(x)\big)\zeta_t(x)\mathrm{d}x
\\
	&\lesssim {\displaystyle \int}p_{\sqrt{\alpha_{t}}X_{t-1}}\big(\phi_{t}^{\star}(x)\big)\bigg\|\frac{\partial\phi_{t}^{\star}(x)}{\partial x}-I\bigg\|_{\mathrm{F}}^{2}\mathrm{d}x+\frac{d\log^{3}T}{T^{2}}
	\asymp \int p_{X_{t}}(x)\Big\|\frac{\partial \phi^{\star}_t(x)}{\partial x} - I\Big\|_{\mathrm{F}}^2\mathrm{d}x
+\frac{d\log^{3}T}{T^{2}}.
\end{align*}

\subsection{Proof of Lemma~\ref{lem:q1-large-qk-large}}
\label{sec:proof-lem:q1-large-qk-large}

In view of the definition \eqref{eq:defn-tao-i}, 
one has 
\begin{align}
	S_k(y_T) \leq c_{14}, \qquad \text{for any }k < \tau(y_T). 
	\label{eq:Sk-yT-all-Omegai}
\end{align}
%
Suppose instead that \eqref{eq:q_k_yk_UB} does not hold true, namely, 
$-\log q_k(y_k)> 2c_{6}d\log T$ for some $k<\tau(y_T)$, and we would like to show that this leads to contradiction. 

Towards this, let $1 < t \le k$ be the smallest time step obeying
\begin{align}
	\theta_t(y_t) = \max\bigg\{ -\frac{\log q_t(y_t)}{d\log T}, c_6 \bigg\} > 2 c_6 = 2 \theta_1(y_1),
\end{align}
where the last identity holds since $-\log q_1(y_1)\leq c_{6}d\log T$ and hence $\theta_1(y_1)=\max\big\{ -\frac{\log q_1(y_1)}{d\log T}, c_6 \big\}=c_6$. 
We claim that $t$ necessarily obeys 
\begin{align}
	2 c_6 < \theta_t(y_t) \leq 4 c_6.   
	\label{eq:thetat-c6-2-4}
\end{align}
Assuming the validity of Claim~\eqref{eq:thetat-c6-2-4} for the moment, 
it necessarily satisfies 
\[
	\theta_1(y_1),\theta_2(y_2),\cdots,\theta_t(y_t)\in [c_6,4c_6].
\]
According to the relations~\eqref{eq:crude-ratio-qt-1-qt} and \eqref{eq:Sk-yT-all-Omegai}, we derive
\begin{align*}
	c_{6}=\theta_{1}(y_{1}) & \le\theta_{t}(y_{t})-\theta_{1}(y_{1})=-\frac{\log q_{t}(y_{t})}{d\log T}-\theta_{1}(y_{1})\leq\frac{-\log q_{t}(y_{t})+\log q_{1}(y_{1})}{d\log T}\\
 & =\frac{1}{d\log T}\sum_{j=1}^{t-1}\big(\log q_{j}(y_{j})-\log q_{j+1}(y_{j+1})\big)\\
	& \leq2c_{1}+C_{10}\left\{ \frac{d\log^{3}T}{T}+\frac{S_{\tau(y_T)-1}(y_{T})}{d\log T}\right\} < 3c_{1}  
\end{align*}
under our sample size condition. This, however, cannot possibly hold if $c_6 \geq 3c_1$ as assumed for Lemma~\ref{lem:q1-large-qk-large}.

To finish up, it suffices to justify Claim~\eqref{eq:thetat-c6-2-4}.  
In order to see this, suppose instead that $\theta_t(y_t) > 4 c_6$. 
Given relation~\eqref{eq:Sk-yT-all-Omegai} that $S_k(y_T)\leq c_{14}$,  it can be readily seen from 
\eqref{eq:xt_up}, \eqref{eq:Sk-yT-all-Omegai} as well as the learning rate properties \eqref{eqn:properties-alpha-proof} that 
\begin{align*}
\theta_{t-1}(y_{t-1}) & =\theta_{t}(y_{t})+\theta_{t-1}(y_{t-1})-\theta_{t}(y_{t})\\
 & =\theta_{t}(y_{t})+\theta_{t-1}(y_{t-1})+\frac{\log q_{t}(y_{t})}{d\log T}\geq\theta_{t}(y_{t})-\frac{\log q_{t-1}(y_{t-1})-\log q_{t}(y_{t})}{d\log T}\\
 & \geq\theta_{t}(y_{t})-\frac{4c_{1}\Big(5\varepsilon_{\score,t}(y_{t})\sqrt{\theta_{t}(y_{t})d\log T}+60\theta_{t}(y_{t})d\log T\Big)}{dT}-\frac{\log2}{d\log T}\\
 & \geq\theta_{t}(y_{t})-\frac{4c_{1}\Big(5\varepsilon_{\score,t}(y_{t})\sqrt{d\log T}+60d\log T\Big)}{dT}\theta_{t}(y_{t})-\frac{\log2}{d\log T}\\
 & >\frac{1}{2}\theta_{t}(y_{t})>2c_{6},
\end{align*}
which is contradictory with the assumption that $t$ is the smallest step obeying $\theta_t(y_t) > 2 c_6$. Thus, we complete the proof of relation~\eqref{eq:q_k_yk_UB} as required. 
%
%
%
%
%
%
%

\subsection{Proof of Lemma~\ref{lem:density-ratio-tau}}
\label{sec:proof-lem-density-ratio-tau}

Next, consider any $y_T$, with $\{y_{T-1},\cdots,y_1\}$ being the associated deterministic sequence (cf.~\eqref{eq:defn-yt-sequence-proof})).  
As an immediate consequence of Lemma~\ref{lem:q1-large-qk-large} and the definition \eqref{eqn:choice-y} of $\theta_t(\cdot)$, one has 
\begin{equation}
	\theta_t(y_t)\leq 2c_6, \qquad \forall  t < \tau(y_T) 
	\label{eq:theta-t-all-small-ST}
\end{equation}
We then intend to invoke Lemma~\ref{lem:refine} to control the term of interest. 
To do so, note that Lemma~\ref{lem:x0}, \eqref{eqn:properties-alpha-proof} and the definition \eqref{eq:defn-tao-i} of $\tau(y_T)$ taken together reveal that: 
for all $t<\tau(y_T)$ one has 
\[
	\frac{d(1-\alpha_{t})}{2(\alpha_{t}-\overline{\alpha}_{t})}\lesssim\frac{d\log T}{T}=o(1) ,
\]
\begin{align*}
 & \theta_{t}(y_{t})^{2}d^{2}\Big(\frac{1-\alpha_{t}}{\alpha_{t}-\overline{\alpha}_{t}}\Big)^{2}\log^{2}T+\varepsilon_{\score,t}(y_{t})\sqrt{\theta_{t}(y_{t})d\log T}\Big(\frac{1-\alpha_{t}}{\alpha_{t}-\overline{\alpha}_{t}}\Big)+(1-\alpha_{t})d\varepsilon_{\Jacobi,t}(y_{t})\\
 & \qquad\lesssim\frac{d^{2}\log^{4}T}{T^{2}}+\frac{\varepsilon_{\score,t}(y_{t})\sqrt{d\log^{3}T}}{T}+\frac{d\varepsilon_{\Jacobi,t}(y_{t})\log T}{T}=o(1),
\end{align*}
and
\begin{align}
 & \left|\frac{(1-\alpha_{t})\Big(\big\|\mathbb{E}\big[X_{t}-\sqrt{\overline{\alpha}_{t}}X_{0}\mymid X_{t}=y_t\big]\big\|_{2}^{2}-\mathbb{E}\big[\big\| X_{t}-\sqrt{\overline{\alpha}_{t}}X_{0}\big\|_{2}^{2}\mymid X_{t}=y_t\big]\Big)}{(\alpha_{t}-\overline{\alpha}_{t})(1-\overline{\alpha}_{t})}\right| \notag\\
 & \qquad\leq\left|\frac{(1-\alpha_{t})\mathbb{E}\big[\big\| X_{t}-\sqrt{\overline{\alpha}_{t}}X_{0}\big\|_{2}^{2}\mymid X_{t}=y_t\big]}{(\alpha_{t}-\overline{\alpha}_{t})(1-\overline{\alpha}_{t})}\right|\lesssim\frac{(1-\alpha_{t})d\log T}{\alpha_{t}-\overline{\alpha}_{t}}\lesssim\frac{d\log^2 T}{T}=o(1) .
	\label{eq:bound-covariance-norm-UB}
\end{align}
With these bounds in mind, applying relations~\eqref{eq:xt} and~\eqref{eq:yt} in Lemma~\ref{lem:main-ODE} leads to
\begin{align*}
 & \frac{p_{\sqrt{\alpha_{t}}Y_{t-1}}\big(\phi_{t}(y_{t})\big)}{p_{Y_{t}}(y_{t})}\bigg(\frac{p_{\sqrt{\alpha_{t}}X_{t-1}}\big(\phi_{t}(y_{t})\big)}{p_{X_{t}}(y_{t})}\bigg)^{-1}=\frac{p_{\phi_{t}(Y_{t})}\big(\phi_{t}(y_{t})\big)}{p_{Y_{t}}(y_{t})}\bigg(\frac{p_{\sqrt{\alpha_{t}}X_{t-1}}\big(\phi_{t}(y_{t})\big)}{p_{X_{t}}(y_{t})}\bigg)^{-1}\\
 & \qquad=1+O\Bigg(\frac{d^{2}\log^{4}T}{T^{2}}
	+\frac{\varepsilon_{\score,t}(y_{t})\sqrt{d\log^{3}T}}{T}+\frac{d\varepsilon_{\Jacobi,t}(y_{t})\log T}{T}\Bigg)
\end{align*}
for all $t<\tau(y_T)$. 
Using the fact that $y_{t-1}=\frac{1}{\sqrt{\alpha_t}} \phi_t(y_t)$ and invoking the relation \eqref{eq:recursion},  
we arrive at
\begin{align*}
\frac{p_{t-1}(y_{t-1})}{q_{t-1}(y_{t-1})} & =\left\{ 1+O\Bigg(\frac{d^{2}\log^{4}T}{T^{2}}
	+\frac{\varepsilon_{\score,t}(y_{t})\sqrt{d\log^{3}T}}{T}+\frac{d\varepsilon_{\Jacobi,t}(y_{t})\log T}{T}\Bigg) \right\} \frac{p_{t}(y_{t})}{q_{t}(y_{t})}
\end{align*}
for any $t<\tau(y_T)$. 
By abbreviating $\tau=\tau(y_{T})$ for notational simplicity, 
we reach
\begin{subequations}
	\label{eq:pt-qt-equiv-ODE-St-temp}
\begin{align}
\frac{p_{1}(y_{1})}{q_{1}(y_{1})} & =\left\{ 1+O\Bigg(\frac{d^{2}\log^{4}T}{T}
	+S_{\tau-1}(y_{\tau-1})\Bigg)\right\} \frac{p_{\tau-1}(y_{\tau-1})}{q_{\tau-1}(y_{\tau-1})}\notag\\
 & \in\left[\frac{p_{\tau-1}(y_{\tau-1})}{2q_{\tau-1}(y_{\tau-1})},\frac{2p_{\tau-1}(y_{\tau-1})}{q_{\tau-1}(y_{\tau-1})}\right],	
	\label{eq:pt-qt-equiv-ODE-St-taui-temp}
\end{align}
and similarly, 
\begin{align}
	\frac{q_{k}(y_{k})}{2p_{k}(y_{k})} \leq \frac{q_{1}(y_{1})}{p_{1}(y_{1})} \leq 2 \frac{q_{k}(y_{k})}{p_{k}(y_{k})}, \qquad \forall k < \tau. 
	\label{eq:pt-qt-equiv-ODE-St-k-temp}
\end{align}
\end{subequations}
This finishes the proof of the claim~\eqref{eq:pt-qt-equiv-ODE-St-k}.

Regarding the other claim~\eqref{eq:pt-qt-equiv-ODE-St-taui}, 
we first observe from \eqref{eq:phix-minus-I-fro-UB} that 
\[
	\bigg\|\frac{\partial\phi_{t}^{\star}(y_{t})}{\partial x}-I\bigg\|_{\mathrm{F}}^{2}\lesssim\bigg(\frac{\theta_{t}(y_{t})d(1-\alpha_{t})\log T}{1-\overline{\alpha}_{t}}\bigg)^{2}\lesssim\frac{d^{2}\log^{3}T}{T^{2}}\lesssim \frac{1}{T\log T} ,
\]
given our assumption that $T\gtrsim d^2 \log^4 T$. 
%
%
Applying Lemma~\ref{lem:main-ODE} leads to 
%
%
\begin{align*}
 & \frac{p_{\sqrt{\alpha_{t}}Y_{t-1}}\big(\phi_{t}(y_{t})\big)}{p_{Y_{t}}(y_{t})}\bigg(\frac{p_{\sqrt{\alpha_{t}}X_{t-1}}\big(\phi_{t}(y_{t})\big)}{p_{X_{t}}(y_{t})}\bigg)^{-1}=\frac{p_{\phi_{t}(Y_{t})}\big(\phi_{t}(y_{t})\big)}{p_{Y_{t}}(y_{t})}\bigg(\frac{p_{\sqrt{\alpha_{t}}X_{t-1}}\big(\phi_{t}(y_{t})\big)}{p_{X_{t}}(y_{t})}\bigg)^{-1}\\
 & \quad=1+\zeta_{t}(y_{t})+O\bigg(\bigg\|\frac{\partial\phi_{t}^{\star}}{\partial x}(y_{t})-I\bigg\|_{\mathrm{F}}^{2}+\frac{\varepsilon_{\mathsf{score},t}(y_{t})\sqrt{d\log^{3}T}}{T}+\frac{d \log T \varepsilon_{\mathsf{Jacobi},t}(x)}{T}+\frac{d\log^{3}T}{T^{2}}\bigg) 
\end{align*}
for all $t<\tau(y_T)$, where $\zeta_t(\cdot)$ is the function defined in Lemma~\ref{lem:refine}. 
Recall the fact that $y_{t-1}=\frac{1}{\sqrt{\alpha_t}} \phi_t(y_t)$ and invoke the relation \eqref{eq:recursion} to arrive at
\begin{align*}
\frac{p_{t-1}(y_{t-1})}{q_{t-1}(y_{t-1})}=\left\{ 1+\zeta_{t}(y_{t})+O\Bigg(\bigg\|\frac{\partial\phi_{t}^{\star}}{\partial x}(y_{t})-I\bigg\|_{\mathrm{F}}^{2}+\frac{d\log^{3}T}{T^{2}}+\frac{\varepsilon_{\score,t}(y_{t})\sqrt{d\log^{3}T}}{T}+\frac{d\varepsilon_{\Jacobi,t}(y_{t})\log T}{T}\Bigg)\right\} \frac{p_{t}(y_{t})}{q_{t}(y_{t})}
\end{align*}
for any $t<\tau(y_T)$. 
Apply this relation recursively over $1<t<\tau$ to conclude the proof of the claim~\eqref{eq:pt-qt-equiv-ODE-St-taui}.



\subsection{Proof of Lemma~\ref{lem:I2-I3-I4-bound}}
\label{sec:proof-lem:I2-I3-I4-bound}

In the following, we shall tackle $\mathcal{I}_{2}$, $\mathcal{I}_{3}$ and $\mathcal{I}_{4}$ separately. 
Throughout this proof, we shall abbreviate $\tau=\tau(Y_T)$ (cf.~\eqref{eq:defn-tao-i}) whenever it is clear from the context.

\paragraph{The sub-collection in $\mathcal{I}_{2}$.}
By virtue of the definition~\eqref{eq:defn-I2-I3-I4-ode-I2} of $\mathcal{I}_{2}$, we make the observation that
\begin{align}
 &  \mathop{\mathbb{E}}_{Y_{T}\sim p_{T}}\bigg[\frac{q_{1}(Y_{1})}{p_{1}(Y_{1})}\ind\left\{ Y_{1}\in\mathcal{E},Y_{T}\in\mathcal{I}_2\right\} \bigg]
	\overset{\mathrm{(i)}}{\leq}  \mathop{\mathbb{E}}_{Y_{T}\sim p_{T}}\bigg[\frac{q_{1}(Y_{1})}{p_{1}(Y_{1})}\ind\left\{ Y_{1}\in\mathcal{E},Y_{T}\in\mathcal{I}_2\right\} \frac{S_{\tau}(Y_{T})}{c_{14}}\bigg]\nonumber \\
	& \quad \overset{\mathrm{(ii)}}{=}\frac{\log T}{c_{14}T} \sum_{t=2}^{\tau}\mathop{\mathbb{E}}_{Y_{T}\sim p_{T}}\bigg[\frac{q_{1}(Y_{1})}{p_{1}(Y_{1})}\ind\left\{ Y_{1}\in\mathcal{E},Y_{T}\in\mathcal{I}_2\right\} \left(d\varepsilon_{\Jacobi,t}(Y_{t})+\sqrt{d\log T}\varepsilon_{\score,t}(Y_{t})\right)\bigg]\nonumber \\
 & \quad\overset{\mathrm{(iii)}}{\leq} \frac{2\log T}{c_{14}T} \sum_{t=2}^{\tau}\mathop{\mathbb{E}}_{Y_{T}\sim p_{T}}\bigg[\frac{q_{t}(Y_{t})}{p_{t}(Y_{t})}\ind\left\{ Y_{1}\in\mathcal{E},Y_{T}\in\mathcal{I}_2\right\} \left(d\varepsilon_{\Jacobi,t}(Y_{t})+\sqrt{d\log T}\varepsilon_{\score,t}(Y_{t})\right)\bigg]\nonumber \\
 %
 & \quad\leq\frac{2\log T}{c_{14}T}\sum_{t=2}^{T}\mathop{\mathbb{E}}_{Y_{T}\sim p_{T}}\bigg[\frac{q_{t}(Y_{t})}{p_{t}(Y_{t})}\left(d\varepsilon_{\Jacobi,t}(Y_{t})+\sqrt{d\log T}\varepsilon_{\score,t}(Y_{t})\right)\bigg]\nonumber \\
 & \quad=\frac{2\log T}{c_{14}T}\sum_{t=2}^{T}\mathop{\mathbb{E}}_{Y_{t}\sim p_{t}}\bigg[\frac{q_{t}(Y_{t})}{p_{t}(Y_{t})}\left(d\varepsilon_{\Jacobi,t}(Y_{t})+\sqrt{d\log T}\varepsilon_{\score,t}(Y_{t})\right)\bigg]\nonumber \\
 & \quad=\frac{2\log T}{c_{14}T}\sum_{t=2}^{T}\mathop{\mathbb{E}}_{Y_{t}\sim q_{t}}\Big[d\varepsilon_{\Jacobi,t}(Y_{t})+\sqrt{d\log T}\varepsilon_{\score,t}(Y_{t})\Big]\nonumber \\
	& \quad \overset{\mathrm{(iv)}}{\lesssim} \big(d\varepsilon_{\Jacobi}+\sqrt{d\log T}\varepsilon_{\score}\big)\log T.\label{eq:UB-I2-ode}
\end{align}
Here, (i) follows since $S_{\tau}\big(y_{T}\big)\geq c_{14}$ in $\mathcal{I}_2$ (see \eqref{eq:defn-I2-I3-I4-ode-I2}); 
(ii) comes from the definition of $S_{t}(\cdot)$ (see \eqref{eq:defn-xik-Stk-proof}); 
(iii) holds since (by repeating the same proof arguments as for \eqref{eq:pt-qt-equiv-ODE-St} as long as $2c_{14}$ is small enough)
\[
	\frac{p_{1}(y_{1})}{q_{1}(y_{1})}\leq  \frac{2p_{t}(y_{t})}{q_{t}(y_{t})}, \qquad \forall t\le \tau ;
\]
%
and (iv) arises from \eqref{eq:score-assumptions-equiv}.

\paragraph{The sub-collection in $\mathcal{I}_{3}$.}
With regards to  $\mathcal{I}_{3}$ (cf.~\eqref{eq:defn-I2-I3-I4-ode-I3}), we can derive the following bound in a way similar to \eqref{eq:UB-I2-ode}: 
\begin{align}
 & \mathop{\mathbb{E}}_{Y_{T}\sim p_{T}}\bigg[\frac{q_{1}(Y_{1})}{p_{1}(Y_{1})}\ind\left\{ Y_{1}\in\mathcal{E},Y_{T}\in \mathcal{I}_3\right\} \bigg]\overset{\mathrm{(i)}}{\leq} \mathop{\mathbb{E}}_{Y_{T}\sim p_{T}}\bigg[\frac{q_{1}(Y_{1})}{p_{1}(Y_{1})}\ind\left\{ Y_{1}\in\mathcal{E},Y_{T}\in \mathcal{I}_3\right\} \frac{\xi_{\tau}(Y_{T})}{c_{14}}\bigg]\nonumber\\
 & =\frac{\log T}{c_{14}T}\mathop{\mathbb{E}}_{Y_{T}\sim p_{T}}\bigg[\frac{q_{1}(Y_{1})}{p_{1}(Y_{1})}\ind\left\{ Y_{1}\in\mathcal{E},Y_{T}\in \mathcal{I}_3\right\} \left(d\varepsilon_{\Jacobi,\tau}(Y_{\tau})+\sqrt{d\log T}\varepsilon_{\score,\tau}(Y_{\tau})\right)\bigg]\nonumber\\
 & \overset{\mathrm{(ii)}}{\leq}\frac{2\log T}{c_{14}T}\mathop{\mathbb{E}}_{Y_{T}\sim p_{T}}\bigg[\frac{q_{\tau-1}(Y_{\tau-1})}{p_{\tau-1}(Y_{\tau-1})}\ind\left\{ Y_{1}\in\mathcal{E},Y_{T}\in \mathcal{I}_3\right\} \left(d\varepsilon_{\Jacobi,\tau}(Y_{\tau})+\sqrt{d\log T}\varepsilon_{\score,\tau}(Y_{\tau})\right)\bigg]\notag\\
 & =\frac{2\log T}{c_{14}T}\sum_{t=2}^{T}
	\mathop{\mathbb{E}}_{Y_{T}\sim p_{T}}\bigg[\frac{q_{t-1}(Y_{t-1})}{p_{t-1}(Y_{t-1})}\ind\left\{ Y_{1}\in\mathcal{E},Y_{T}\in \mathcal{I}_3\right\} \left(d\varepsilon_{\Jacobi,t}(Y_{t})+\sqrt{d\log T}\varepsilon_{\score,t}(Y_{t})\right) \ind\{\tau=t\} \bigg]
	\label{eq:E-I3-UB-579}\\
 & \overset{\mathrm{(iii)}}{\leq}\frac{16\log T}{c_{14}T}\mathop{\mathbb{E}}_{Y_{T}\sim p_{T}}\bigg[\frac{q_{\tau}(Y_{\tau})}{p_{\tau}(Y_{\tau})}\ind\left\{ Y_{1}\in\mathcal{E},Y_{T}\in\mathcal{I}_3\right\} \left(d\varepsilon_{\Jacobi,\tau}(Y_{\tau})+\sqrt{d\log T}\varepsilon_{\score,\tau}(Y_{\tau})\right)\bigg]\notag\\
 & \leq\frac{16\log T}{c_{14}T}\sum_{t=2}^{T}\mathop{\mathbb{E}}_{Y_{T}\sim p_{T}}\bigg[\frac{q_{t}(Y_{t})}{p_{t}(Y_{t})}\left(d\varepsilon_{\Jacobi,t}(Y_{t})+\sqrt{d\log T}\varepsilon_{\score,t}(Y_{t})\right)\bigg]\notag\\
 & =\frac{16\log T}{c_{14}T}\sum_{t=2}^{T}\mathop{\mathbb{E}}_{Y_{t}\sim p_{t}}\bigg[\frac{q_{t}(Y_{t})}{p_{t}(Y_{t})}\left(d\varepsilon_{\Jacobi,t}(Y_{t})+\sqrt{d\log T}\varepsilon_{\score,t}(Y_{t})\right)\bigg]\notag\\
 & =\frac{16\log T}{c_{14}T}\sum_{t=2}^{T}\mathop{\mathbb{E}}_{Y_{t}\sim q_{t}}\Big[d\varepsilon_{\Jacobi,t}(Y_{t})+\sqrt{d\log T}\varepsilon_{\score,t}(Y_{t})\Big]\notag\\
 & \lesssim\left(d\varepsilon_{\Jacobi}+\sqrt{d\log T}\varepsilon_{\score}\right)\log T.
	\label{eq:UB-I3-ode}
\end{align}
Here, (i) comes from \eqref{eq:defn-I2-I3-I4-ode-I3}, 
(ii) arises from \eqref{eq:pt-qt-equiv-ODE-St-k}, 
whereas (iii) is a consequence of \eqref{eq:defn-I2-I3-I4-ode-I3}.

\paragraph{The sub-collection in $\mathcal{I}_{4}$.}

We now turn attention to $\mathcal{I}_4$ (cf.~\eqref{eq:defn-I2-I3-I4-ode-I4}), towards which we find it helpful to define
\begin{subequations}
	\label{eq:defn-J1t-J2t-J3t}
\begin{align}
\mathcal{J}_{1,t} & \coloneqq \Big\{ y_T :\xi_{t}\big(y_T\big)<c_{14}\Big\} \label{eq:defn-J1t-J2t-J3t-J1t}\\
\mathcal{J}_{2,t} & \coloneqq \bigg\{ y_T :\xi_{t}\big(y_T\big)\geq c_{14},\frac{q_{t-1}(y_{t-1})}{p_{t-1}(y_{t-1})}\leq\frac{8 q_{t}(y_t)}{p_{t}(y_t)}\bigg\} \label{eq:defn-J1t-J2t-J3t-J2t}\\
\mathcal{J}_{3,t} & \coloneqq \bigg\{ y_T :\xi_{t}\big(y_T\big)\geq c_{14},\frac{q_{t-1}(y_{t-1})}{p_{t-1}(y_{t-1})}>\frac{8 q_{t}(y_t)}{p_{t}(y_{t})}\bigg\}
\label{eq:defn-J1t-J2t-J3t-J3t}
\end{align}
\end{subequations}
for each $2\leq t\leq T$. 
%
Equipped with the above definitions, we first make the observation that 
\begin{align}
\mathop{\mathbb{E}}_{Y_{T}\sim p_{T}}\bigg[\frac{q_{1}(Y_{1})}{p_{1}(Y_{1})}\ind\left\{ Y_{1}\in\mathcal{E},Y_{T}\in\mathcal{I}_{4}\right\} \bigg] 
	& \leq 2 \mathop{\mathbb{E}}_{Y_{T}\sim p_{T}}\bigg[\frac{q_{\tau-1}(Y_{1})}{p_{\tau-1}(Y_{1})}\ind\left\{ Y_{1}\in\mathcal{E},Y_{T}\in\mathcal{I}_{4}\right\} \bigg]\notag\\
 	& = 2\sum_{t=2}^{T}\mathop{\mathbb{E}}_{Y_{T}\sim p_{T}}\bigg[\frac{q_{t-1}(Y_{t-1})}{p_{t-1}(Y_{t-1})}\ind\left\{ Y_{1}\in\mathcal{E},Y_{T}\in\mathcal{I}_{4}\right\}
	\ind\{\tau=t\} \bigg]\notag\\
	& \leq 2\sum_{t=2}^{T}\mathop{\mathbb{E}}_{Y_{T}\sim p_{T}}\bigg[\frac{q_{t-1}(Y_{t-1})}{p_{t-1}(Y_{t-1})}\ind\left\{Y_{1}\in\mathcal{E}, Y_{T}\in \mathcal{J}_{3,t}\right\} \bigg], 
	\label{eq:sum-I4-UB-13579}
\end{align}
where the first inequality follows from \eqref{eq:pt-qt-equiv-ODE-St-k}, 
and the last line comes from the definition of $\mathcal{I}_4$ (cf.~\eqref{eq:defn-I2-I3-I4-ode-I4}) and $\mathcal{J}_{3,t}$ (cf.~\eqref{eq:defn-J1t-J2t-J3t-J3t}).   
For notational simplicity, let us define, for $2\leq t\leq T$, 
\begin{align*}
h_{t} & \coloneqq\frac{q_{t}(Y_{t})}{p_{t}(Y_{t})}.
\end{align*}
In view of the second inequality in \eqref{eq:defn-J1t-J2t-J3t-J3t}, one has 
$h_{t-1} > 8h_{t}$ as long as $y_T \in \mathcal{J}_{3,t}$. 
Consequently, 
\begin{align*}
& \sum_{t=2}^{T}h_{t-1}\ind\left\{Y_T \in \mathcal{J}_{3,t}\right\} \\
& < 
\sum_{t=2}^{T}h_{t-1}\ind\left\{Y_T \in \mathcal{J}_{3,t}\right\} + \frac{1}{7}\sum_{t=2}^{T}h_{t-1}\ind\left\{Y_T \in \mathcal{J}_{3,t}\right\} -\frac{8}{7}\sum_{t=2}^{T}h_{t}\ind\left\{Y_T \in \mathcal{J}_{3,t}\right\}\\
& =\frac{8}{7}\sum_{t=2}^{T}
\bigg(
	\Big(h_{t-1} - h_{t-1}\ind\left\{Y_T \in \mathcal{J}_{1,t}\right\}  - h_{t-1}\ind\left\{Y_T \in \mathcal{J}_{2,t}\right\} \Big)
-
 \Big(h_{t} - h_{t}\ind\left\{Y_T \in \mathcal{J}_{1,t}\right\}  - h_{t}\ind\left\{Y_T \in \mathcal{J}_{2,t} \right\}\Big) 
\bigg)\\
&= \frac{8}{7}\sum_{t=2}^{T}\big(h_{t}-{h}_{t-1}\big)\ind\left\{Y_T \in \mathcal{J}_{1,t}\cup \mathcal{J}_{2,t}\right\}
+
	\frac{8}{7}\sum_{t=2}^{T} \big(h_{t-1} - h_{t}\big).
\end{align*}
Here, the second line holds true since, for all $t$, one has (i) $\mathcal{J}_{1,t}\cup \mathcal{J}_{2,t} \cup \mathcal{J}_{3,t}= \mathbb{R}^d$, and (ii) $\mathcal{J}_{1,t}$, 
$\mathcal{J}_{2,t}$ and $\mathcal{J}_{3,t}$ are disjoint. 
Substituting this into \eqref{eq:sum-I4-UB-13579}, we arrive at
\begin{align}
	&\mathop{\mathbb{E}}_{Y_{T}\sim p_{T}}\bigg[\frac{q_{1}(Y_{1})}{p_{1}(Y_{1})}
	\ind\left\{ Y_{1}\in\mathcal{E}, Y_T \in \mathcal{J}_{3,t}\right\} \bigg]  
	\leq 2\sum_{t=2}^{T}\mathop{\mathbb{E}}_{Y_{T}\sim p_{T}}\big[h_{t-1}\ind\left\{Y_T \in \mathcal{J}_{3,t}\right\}\big] \notag\\
 & \quad \leq\frac{8}{7}\sum_{t=2}^{T}\Big(\mathop{\mathbb{E}}_{Y_{T}\sim p_{T}}\big[h_{t}\ind\left\{Y_T \in \mathcal{J}_{1,t}\cup \mathcal{J}_{2,t}\right\}\big] - \mathop{\mathbb{E}}_{Y_{T}\sim p_{T}}\big[h_{t-1}\ind\left\{Y_T \in \mathcal{J}_{1,t}\cup \mathcal{J}_{2,t}\right\}\big]\Big) \notag\\
	&\qquad\qquad + \frac{8}{7}\sum_{t=2}^{T}\Big(\mathop{\mathbb{E}}_{Y_{T}\sim p_{T}}\big[h_{t-1}\big]-\mathop{\mathbb{E}}_{Y_{T}\sim p_{T}}\big[h_{t}\big]\Big) . 
	\label{eq:sum-I4-UB-7924}
\end{align}

In order to further bound \eqref{eq:sum-I4-UB-7924}, 
we make note of a few basic facts. Firstly, the identity below holds: 
\begin{align*}
\mathop{\mathbb{E}}_{Y_{T}\sim p_{T}}\big[h_{t}\big] 
& =\mathop{\mathbb{E}}_{Y_{T}\sim p_{T}}\bigg[\frac{q_{t}(Y_{t})}{p_{t}(Y_{t})} \bigg]
	=\mathop{\mathbb{E}}_{Y_{t}\sim p_{t}}\bigg[\frac{q_{t}(Y_{t})}{p_{t}(Y_{t})} \bigg]
	= 1, \qquad  2\leq t\leq T.
\end{align*}
Secondly, by defining the set
\begin{align}
	\mathcal{E}_t &\coloneqq \Big\{y : q_{t}(y) >  \exp\big(- c_{6} d\log T \big)  \Big\}, \qquad 2\leq t\leq T, 
\end{align}
we can show that
\begin{align*}
	&\sum_{t=2}^{T}\mathop{\mathbb{E}}_{Y_{T}\sim p_{T}}\big[h_{t}\ind\left\{ Y_{t}\notin\mathcal{E}_{t}, Y_T\in\mathcal{J}_{1,t}\right\} \big] 
 \le\sum_{t=2}^{T}\mathop{\mathbb{E}}_{Y_{T}\sim p_{T}}\bigg[\frac{q_{t}(Y_{t})}{p_{t}(Y_{t})}\ind\left\{ Y_{t}\notin\mathcal{E}_{t}\right\} \bigg]
=\sum_{t=2}^{T}\mathop{\mathbb{E}}_{Y_{t}\sim p_{t}}\bigg[\frac{q_{t}(Y_{t})}{p_{t}(Y_{t})}\ind\left\{ Y_{t}\notin\mathcal{E}_{t}\right\} \bigg]\\
 & \qquad =\sum_{t=2}^{T}\mathbb{P}_{Y_{t}\sim q_{t}}\left\{ Y_{t}\notin\mathcal{E}_{t} \right\} =\sum_{t=2}^{T}\mathbb{P}_{X_{t}\sim q_{t}}\left\{ X_{t}\notin\mathcal{E}_{t}\right\} \\
 & \qquad\leq\sum_{t=2}^{T}\mathbb{P}_{X_{t}\sim q_{t}}\left\{ X_{t}\notin\mathcal{E}_{t} \text{ and }\|X_{t}\|_{2}\leq T^{2c_{R}+2}\right\} +\sum_{t=2}^{T}\mathbb{P}_{X_{t}\sim q_{t}}\left\{ \|X_{t}\|_{2}>T^{2c_{R}+2}\right\} \\
 & \qquad\leq\sum_{t=2}^{T}{\displaystyle \int}_{x_{t}:q_{t}(x_{t})\leq\exp(-c_{6}d\log T),\|x_{t}\|_{2}\leq T^{2c_{R}+2}}q_{t}(x_{t})\mathrm{d}x_{t}+T\exp\big(-c_{6}d\log T\big)\\
 & \qquad\leq T\big(2T^{2c_{R}+2}\big)^{d}\exp(-c_{6}d\log T)+T\exp\big(-c_{6}d\log T\big)\le\exp\Big(-\frac{c_{6}}{2}d\log T\Big),
\end{align*}
where the penultimate line comes from \eqref{eq:Xt-2range-ODE}, and
the last inequality holds true as long as $c_{6}$ is large enough. 
Plugging the preceding two results into \eqref{eq:sum-I4-UB-7924}, we reach
\begin{align}
	&\mathop{\mathbb{E}}_{Y_{T}\sim p_{T}}\bigg[\frac{q_{1}(Y_{1})}{p_{1}(Y_{1})}\ind\left\{ Y_{1}\in\mathcal{E},Y_{T}\in\mathcal{I}_{4}\right\} \bigg]  
	 \le \frac{8}{7}\sum_{t=2}^{T}
	\mathop{\mathbb{E}}_{Y_{T}\sim p_{T}}\big[ (h_{t}-h_{t-1}) \ind\left\{y_{t}\in\mathcal{E}_t, Y_T \in \mathcal{J}_{1,t}\right\}\big] 
	\notag\\
	&\quad\quad\qquad\qquad + \frac{8}{7}\sum_{t=2}^{T}\mathop{\mathbb{E}}_{Y_{T}\sim p_{T}}\big[h_{t}\ind\left\{Y_T \in \mathcal{J}_{2,t}\right\}\big] + \exp\Big(-\frac{c_{6}}{2}d\log T\Big).
	\label{eq:sum-I4-UB-7935}
\end{align}
%



As it turns out, the sum w.r.t.~the set $\mathcal{J}_{1,t}$ and the sum w.r.t.~the set $\mathcal{J}_{2,t}$ 
in \eqref{eq:sum-I4-UB-7935}  can be controlled respectively using the same arguments as for $\mathcal{I}_1$ and $\mathcal{I}_3$ to derive
\begin{align*}
\sum_{t=2}^{T}\mathop{\mathbb{E}}_{Y_{T}\sim p_{T}}\big[(h_{t}-h_{t-1})\ind\left\{ y_{t}\in\mathcal{E}_{t},Y_{T}\in\mathcal{J}_{1,t}\right\} \big] & \lesssim\frac{d\log^{4}T}{T}+\sqrt{d\log^{3}T}\varepsilon_{\score}+(d\log T)\varepsilon_{\Jacobi},\\
\sum_{t=2}^{T}\mathop{\mathbb{E}}_{Y_{T}\sim p_{T}}\big[h_{t}\ind\left\{ Y_{T}\in\mathcal{J}_{2,t}\right\} \big] & \lesssim\sqrt{d\log^{3}T}\varepsilon_{\score}+(d\log T)\varepsilon_{\Jacobi};
\end{align*}
we omit the arguments here for the sake of brevity.  
Therefore, we have proven that
\begin{align}
\mathop{\mathbb{E}}_{Y_{T}\sim p_{T}}\bigg[\frac{q_{1}(Y_{1})}{p_{1}(Y_{1})}\ind\left\{ Y_{1}\in\mathcal{E},Y_{T}\in\mathcal{I}_{4}\right\} \bigg] & \lesssim\frac{d\log^{4}T}{T}+\sqrt{d\log^{3}T}\varepsilon_{\score}+(d\log T)\varepsilon_{\Jacobi}.  
	\label{eq:UB-I4-ode}
\end{align}

\paragraph{Putting all this together.} 
Taking \eqref{eq:UB-I2-ode}, \eqref{eq:UB-I3-ode} and \eqref{eq:UB-I4-ode} together, we establish the advertised result. 

\bibliographystyle{apalike}
\bibliography{reference-diffusion}


\end{document}